%% file: main.tex
\documentclass[letterpaper,journal]{IEEEtran}

\usepackage[american]{babel}
\usepackage{amsmath,amsfonts,amssymb}
\usepackage{algorithm}
\usepackage{algpseudocode}
\usepackage{array}
\usepackage{booktabs}
\usepackage{multirow}
\usepackage{tabularx}
\usepackage{makecell}
\usepackage{colortbl}
\usepackage[dvipsnames, svgnames, x11names]{xcolor}
\usepackage[caption=false,font=normalsize,labelfont=sf,textfont=sf]{subfig}
\usepackage{graphicx}
\graphicspath{{figure/}}
\usepackage{textcomp}
\usepackage{stfloats}
\usepackage{url}
\usepackage{cite}
\usepackage[hyphens]{xurl}
\usepackage{microtype}
\usepackage{bm}
\usepackage{spverbatim}
\usepackage{etoc}
\usepackage{tikz}
\usetikzlibrary{arrows.meta, positioning, calc, decorations.markings}

\usepackage{amsthm}
\theoremstyle{plain}
\newtheorem{theorem}{Theorem}
\newtheorem{proposition}[theorem]{Proposition}
\newtheorem{lemma}[theorem]{Lemma}

\theoremstyle{definition}

\newtheorem{assumption}[theorem]{Assumption}
\theoremstyle{remark}
\newtheorem{remark}[theorem]{Remark}

\usepackage[pagebackref,breaklinks,colorlinks]{hyperref}
\hypersetup{
    citecolor=green!80!black
}
\usepackage[capitalize,noabbrev]{cleveref}

\newcolumntype{C}{>{\centering\arraybackslash}X}
\newcommand{\romark}[1]{\uppercase\expandafter{\romannumeral #1\relax}}
\AtBeginEnvironment{spverbatim}{\small}
\newcommand{\norm}[1]{\left\lVert #1 \right\rVert}
\newcommand{\inp}[2]{\left\langle #1, #2 \right\rangle}
\newcommand{\ie}{\emph{i.e.}}
\newcommand{\eg}{\emph{e.g.}}
\DeclareRobustCommand{\swudi}{\mbox{\normalfont\scshape SWUDI}}
\DeclareRobustCommand{\aswudi}{\mbox{\normalfont\scshape SWUDI-A}}
\DeclareRobustCommand{\optmerge}{\mbox{\normalfont OptMerge}}
\DeclareRobustCommand{\wudi}{\mbox{\normalfont\scshape WUDI}}

\makeatletter
\providecommand{\part}{\@ifstar{\@gobble}{\@gobble}}

\providecommand{\@IEEEtablestring}{table}
\providecommand{\@IEEEfigurecaptionsepspace}{\vskip\abovecaptionskip}
\providecommand{\@IEEEtablecaptionsepspace}{\vskip\abovecaptionskip}
\renewcommand{\tablename}{Table}
\def\fnum@table{\tablename\nobreakspace\thetable}
\long\def\@makecaption#1#2{%
  \ifx\@captype\@IEEEtablestring
    {\normalfont\footnotesize\centering #1. #2\par}%
    \@IEEEtablecaptionsepspace
  \else
    \@IEEEfigurecaptionsepspace
    \setbox\@tempboxa\hbox{\normalfont\footnotesize {#1.}~#2}%
    \ifdim \wd\@tempboxa >\hsize
      \normalfont\footnotesize {#1.}~#2\par
    \else
      \hbox to\hsize{\hfil\box\@tempboxa\hfil}%
    \fi
  \fi}
\makeatother

\begin{document}

\title{Closed-Form Spectral Regularization for Multi-Task Model Merging}

\author{Yongxian~Wei, Runxi~Cheng, Xingxuan~Zhang, Li~Shen, Chun~Yuan, Peng~Cui,~\IEEEmembership{Fellow,~IEEE}, Dacheng~Tao,~\IEEEmembership{Fellow,~IEEE}
\thanks{Yongxian~Wei, Runxi~Cheng, and Chun~Yuan are with Shenzhen International Graduate School, Tsinghua University, Shenzhen 518071, China (email: \{weiyx23, crx23\}@mails.tsinghua.edu.cn, yuanc@sz.tsinghua.edu.cn).

Xingxuan~Zhang and Peng~Cui are with the Department of Computer Science and Technology, Tsinghua University, Beijing 100084, China (email: xingxuanzhang@hotmail.com, cuip@tsinghua.edu.cn).

Li~Shen is with Sun Yat-sen University, Shenzhen 510275, China (email: mathshenli@gmail.com).

Dacheng~Tao is with Nanyang Technological University, Singapore 639798 (email: dacheng.tao@gmail.com).
}%
}


\IEEEtitleabstractindextext{%
\begin{abstract}
Model merging combines several independently fine-tuned experts into a single multi-task model without any training data, reducing the storage, serving, and decentralized-development costs of large foundation models.
State-of-the-art merging methods formulate merging as a layer-wise quadratic interference minimization problem. Although this problem admits an exact closed-form pseudoinverse solution, that solution underperforms hundreds of iterations of gradient descent in practice. The iterative loop dominates the cost of the pipeline (\eg, $85$ minutes and $42$ GB of GPU memory), yet its effectiveness has remained unexplained.
We revisit this regime and show that the iterative solver does not primarily act as an optimizer; rather, it serves as an \emph{implicit spectral regularizer} for an ill-posed normal equation, where small-eigenvalue directions of the per-layer interference operator amplify proxy noise.
Building on this finding, we formalize multi-task model merging as a noisy linear inverse problem, and propose a spectral filtering estimator parameterized by a per-direction filter $h_k$. We instantiate this estimator with \swudi{}, a closed-form method that combines a soft exponential filter, which matches the gradient-flow trajectory of iterative descent, with a hard top-$K$ truncation that suppresses noise-amplifying small-eigenvalue directions. Furthermore, we propose \aswudi{}, an adaptive variant that replaces the global rank hyperparameter with per-layer rank rules, further improving robustness across architectures. Both variants share a single symmetric eigendecomposition per linear layer and require no training data or optimizer state.
Across four general benchmarks (vision/language) and a multimodal merging benchmark spanning VQA, Geometry, Chart, OCR, Grounding, and modality merging, our proposed spectral solvers match or outperform state-of-the-art merging methods. Crucially, they reduce wall-clock time by $28$--$72\times$ and peak GPU memory by up to $50\%$. Code and the extended benchmark are available at \url{https://github.com/WalkerWorldPeace/MLLMerging}.
\end{abstract}

}

\maketitle
\IEEEdisplaynontitleabstractindextext
\IEEEpeerreviewmaketitle

\input{sections/01_introduction}
\input{sections/02_related_work}
\input{sections/03_preliminaries}
\input{sections/04_methods}
\input{sections/05_benchmarks}
\input{sections/06_analysis}
\input{sections/07_conclusion}

\bibliographystyle{IEEEtran}
\bibliography{references}

\clearpage
\onecolumn
\appendices

\part*{}
\addcontentsline{toc}{part}{Appendix}

\begin{center}
    \LARGE \bfseries {Supplementary Material of Closed-Form\\Spectral Regularization for Multi-Task Model Merging}
\end{center}

\vspace{0.5em}

\setcounter{tocdepth}{3}
\etocsetnexttocdepth{subsubsection}
\localtableofcontents

\input{appendices/A_notation}

\input{appendices/B_derivations}
\input{appendices/C_additional_results}
\input{appendices/D_reproducibility}

\end{document}

%% file: sections/01_introduction.tex
\section{Introduction}
\label{sec:intro}

\IEEEPARstart{U}{pdating} foundation models is costly: full pre-training or large-scale continued training requires substantial compute and data access. At the same time, domain-specialized, task-specific fine-tuned checkpoints are continually released on open-source platforms such as Hugging Face~\cite{wolf2019huggingface}. Model merging~\cite{yadav2024matters,ilharcoediting,yang2024model} aims to combine $N$ experts that share the same backbone into a single multi-task model \emph{without} any training data, dramatically reducing the storage, serving, and decentralized-development costs of large foundation models. State-of-the-art methods \wudi{}~\cite{cheng2025whoever} and \optmerge{}~\cite{wei2026optmerge} formulate model merging as a layer-wise quadratic interference minimization problem, yielding consistently strong performance across diverse tasks and models.

These iterative methods share a common computational core: they minimize the interference proxy through hundreds of iterations of a gradient-based optimizer with carefully tuned learning rates, momenta, and initializations (the precise choice, Adam for full fine-tuning versus SGD for LoRA, is detailed in Sec.~\ref{sec:prelim}). Empirical evidence reveals two key observations. \emph{First}, the proxy admits an exact closed-form minimum (a normal-equation pseudoinverse), but plugging this closed form into the merged model yields markedly worse downstream accuracy than running the iterative solver to early stopping (\eg, a $2.3$-point drop on CLIP-ViT-B/32). \emph{Second}, the iterative solver dominates the cost of the entire merging pipeline: on a 3B-parameter LLM, $300$ Adam steps take $85$ minutes and $42$\,GB of GPU memory. Why iterative descent outperforms the exact minimizer of the same objective, especially in a setting without training data, has remained unexplained.

We revisit iterative descent and show that it performs \emph{implicit spectral regularization} for an ill-posed normal equation. The per-layer interference loss
\[
\mathcal{L}(\tau)=\sum_{i=1}^N \tfrac{1}{\|\tau_i\|_F^2}\bigl\|(\tau-\tau_i)\tau_i^\top\bigr\|_F^2
\]
has a unique minimum-norm closed form $\tau^{\rm cf}=DC^\dagger$, where $A_i=\tau_i^\top \tau_i/\|\tau_i\|_F^2$, $C=\sum_i A_i$, and $D=\sum_i \tau_i A_i$. The eigenstructure $C=Q\Lambda Q^\top$ exhibits a long tail of small $\lambda_k$ that correspond to directions weakly supported by any task vector. In those directions, the proxy reduces to $y_k=\lambda_k\,\tau_k^\circ+\xi_k$, where $\xi_k$ collects the proxy-induced noise (\ie, the discrepancy between the task-vector proxy $\tau_i^\top\tau_i$ and the unobserved activation covariance it stands in for), and $\tau^{\rm cf}$ divides by $\lambda_k$, amplifying $\xi_k$. Iterative descent, by contrast, acts as an early-stopping spectral filter that down-weights small-$\lambda_k$ directions. This mechanism explains why $300$-step optimization can outperform the exact closed-form solution.

\begin{figure*}[!t]
\centering
\includegraphics[width=\textwidth]{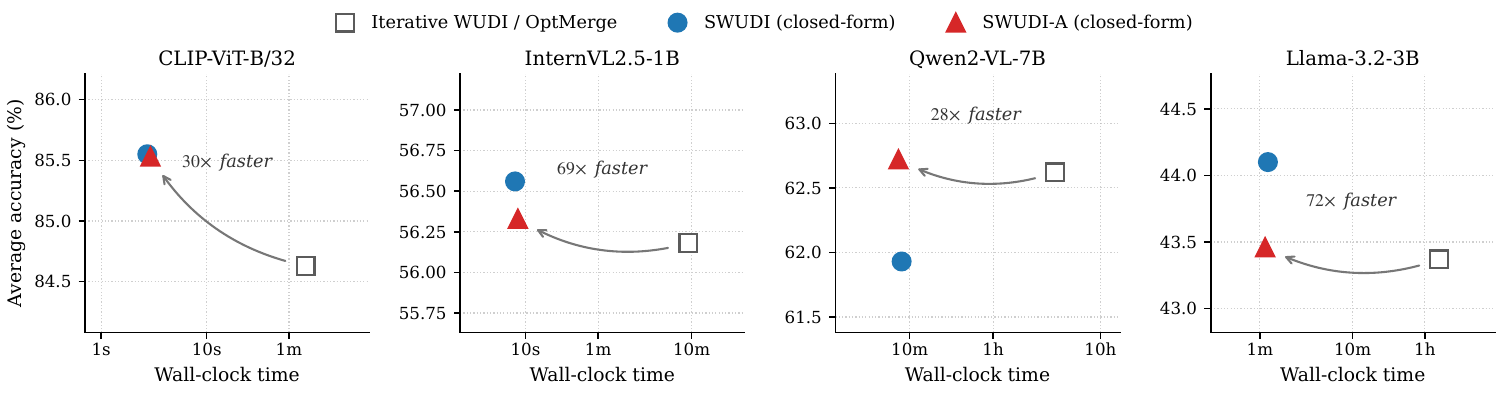}
\caption{\textbf{Accuracy--cost Pareto frontier across representative settings.} Each panel plots average accuracy against wall-clock time for the iterative \wudi{}/\optmerge{} baseline (\textcolor{black!60}{\(\square\)}) and our proposed closed-form solver \swudi{} (\textcolor[HTML]{1F77B4}{\(\bullet\)}) together with its adaptive variant \aswudi{} (\textcolor[HTML]{D62728}{\(\blacktriangle\)}). The closed-form solvers move merging toward the upper-left, achieving higher or comparable accuracy under a smaller merging budget, so the iterative baseline is Pareto-dominated.}
\label{fig:pareto}
\end{figure*}

Building on this finding, we formalize model merging as a noisy linear inverse problem $\tau C=D$, and propose a \emph{unified spectral filtering estimator} parameterized by a per-direction filter $h_k\in[0,1]$ applied to every eigendirection of $C$. The estimator subsumes the closed-form pseudoinverse, gradient flow, and rank truncation as filter choices, and any instantiation requires only a single symmetric eigendecomposition per linear layer. We materialize this framework in two stages, illustrated alongside prior merging families in Fig.~\ref{fig:method-overview}.
\textbf{(i)} \swudi{} (Spectrally Regularized \wudi{}): a tunable spectral variant of the unified estimator that couples a soft exponential filter $s_t(\lambda_k)=1-e^{-t\lambda_k}$, which exactly matches the gradient-flow trajectory of iterative descent, with a hard top-$K$ truncation $m_k=\mathbf{1}[k\le K]$, $K=\lceil r\,d_i\rceil$. The soft factor inherits the regularization that early-stopped descent already provides; the hard mask removes noise-amplifying tail directions before the soft factor can assign them non-negligible residual weight.
\textbf{(ii)} \aswudi{} (Adaptive \swudi{}): we further upgrade \swudi{} into an adaptive, parameter-free form by replacing the global rank ratio $r$ with per-layer rank rules driven by the eigenspectrum itself. For heavy-tailed spectra, we use $K_\ell=\bigl\lceil(\sum_k\sqrt{\lambda_k})^2/\sum_k\lambda_k\bigr\rceil$, an effective-rank estimator~\cite{roy2007effective} that returns exactly the oracle active rank when the spectrum is flat-and-truncated. For spiked-noise spectra, we use $K_\ell=\bigl|\{k:\sqrt{\lambda_k}>\omega(\beta)\,\mathrm{median}_j\sqrt{\lambda_j}\}\bigr|$, an asymptotically optimal singular-value threshold~\cite{marchenko1967distribution,gavish2014optimal} that retains only directions above the asymptotic random-noise floor under a spiked-noise model. These two variants are not separate algorithms but successive refinements of the same spectral framework: \swudi{} establishes the filter shape, while \aswudi{} derives its only remaining hyperparameter directly from the spectrum. Both are data- and training-free, require neither Adam states nor learning-rate schedules, and reduce the per-layer cost from hundreds of matrix multiplications to a single eigendecomposition.

We evaluate on four general benchmarks and a comprehensive multimodal benchmark (covering VQA, Geometry, Chart, OCR, Grounding, and modality merging). Spanning vision, language, multimodal, LoRA, and full-parameter settings, our solvers establish a new state-of-the-art. On CLIP-ViT, they achieve 
85.55\%, 89.57\%, and 92.51\% accuracy across the B/32, B/16, and L/14 backbones. Applying AdaMerging~\cite{yang2024adamerging} to our closed-form delta further lifts B/32 accuracy to 86.08\%, proving spectral merging is complementary to test-time adaptation. On Flan-T5 GLUE, \aswudi{} reaches +1.15\% over TSV-Merging. Furthermore, merging task-specialized MLLMs boosts general capabilities: the merged model achieves 70.58\% on integrated multimodal QA, far outperforming individual experts. Crucially, these accuracy gains are delivered with $28$--$72\times$ wall-clock speedup and up to $50\%$ peak GPU memory reduction relative to the iterative baselines.

\begin{figure*}[!t]
\centering
\includegraphics[width=0.7\textwidth]{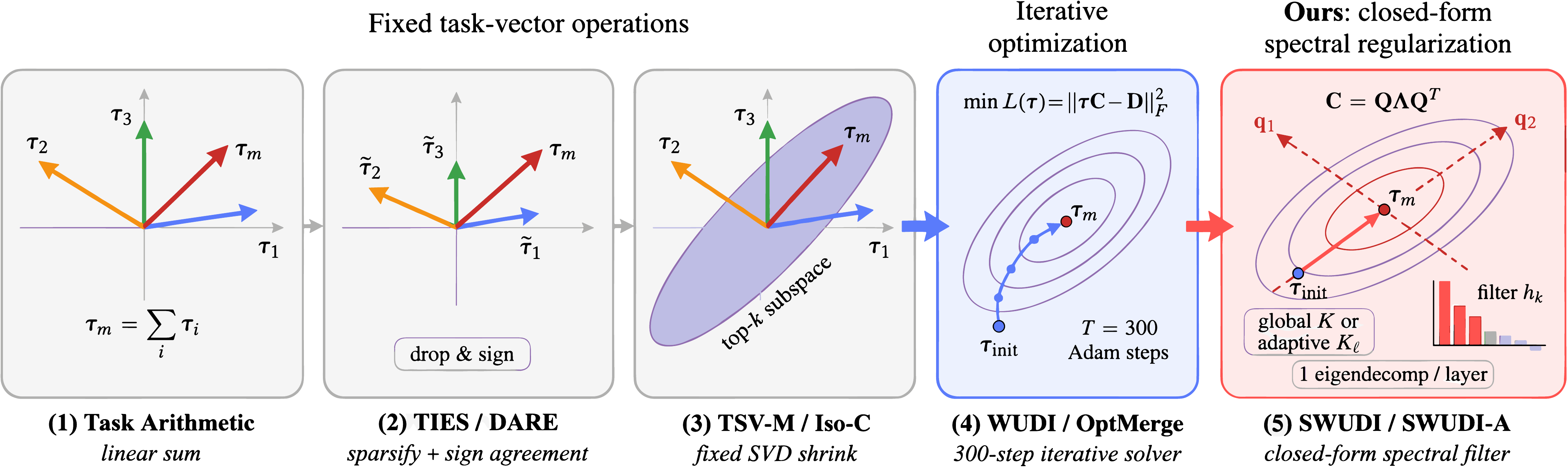}
\caption{\textbf{Illustration of different model merging methods.} Panels~1--3 apply
fixed operations to per-layer task vectors. Panel~4 shows
\wudi{}/\optmerge{} reaching the merged solution by hundreds of Adam steps
on a quadratic proxy loss. Panel~5 (Ours): \swudi{}/\aswudi{} replace this
loop with a single per-layer eigendecomposition followed by a spectral
filter that down-weights noise-amplifying small-eigenvalue directions.}
\label{fig:method-overview}
\end{figure*}

Our contributions are summarized as follows:
\begin{itemize}
\item \textbf{Theory.} We show that merging via the layer-wise quadratic interference loss is a noisy linear inverse problem whose closed-form pseudoinverse amplifies proxy noise on small-$\lambda_k$ directions, and that finite-step iterative descent acts as an implicit spectral regularizer (Propositions~\ref{prop:cf} and~\ref{prop:flow}).
\item \textbf{Methodology.} We propose a spectral filtering estimator: \swudi{}, a closed-form spectral variant that pairs an exponential gradient-flow filter with a hard rank truncation; and \aswudi{}, the adaptive form of \swudi{} that removes its hyperparameter via per-layer rank rules driven by the eigenspectrum (Sec.~\ref{sec:aswudi}). Meanwhile, our solvers reduce wall-clock time by $28$--$72\times$ across every setting and reduce peak GPU memory by up to $50\%$.
\item \textbf{Benchmark.} We introduce the first model merging benchmark that provides a fine-grained categorization of MLLM capabilities and evaluates how merging integrates multiple modalities. We train expert models for each task and publicly release their weights and code. This benchmark is designed to help the model merging community better evaluate the generalizability of their methods.
\end{itemize}

%% file: sections/02_related_work.tex
\section{Related Work}
\label{sec:related}

\subsection{Data-Free Model Merging}
Data-free merging produces a single multi-task model from $N$ fine-tuned experts that share a common base, without using any training or unlabeled test data. Existing methods can be broadly grouped into four families.

\textbf{Linear interpolation:}
Weight Averaging~\cite{wortsman2022model} averages all expert weights and works surprisingly well in narrow scenarios in which experts share a basin in parameter space. Task Arithmetic~\cite{ilharcoediting} introduces global task vectors $\bm{\tau}_i:=\Theta_i-\Theta_0$ and combines them additively as $\bm{\tau}_m=s\sum_i\bm{\tau}_i$ with a global coefficient $s$.
\textbf{Sparsification-based:}
TIES-Merging~\cite{yadav2023ties} trims, signs, and disjointly sums task vectors to suppress conflicting components. DARE~\cite{yu2024language} randomly drops and rescales task-vector entries to mitigate parameter interference.
\textbf{SVD-based:}
TSV-Merging (TSV-M)~\cite{gargiulo2024task} measures task-specific singular interference and decorrelates the dominant singular components. Iso-C~\cite{marczak2025no} flattens the singular spectrum so that no task dominates the merged model. Both methods can be interpreted as fixed spectral manipulations of the stacked task-vector geometry.
\textbf{Optimization-based:}
DOGE~\cite{wei2025modeling} frames model merging as a constrained optimization problem and solves it via adaptive projective gradient descent.
\wudi{}~\cite{cheng2025whoever} proves that, under the linear-subspace approximation, fine-tuning data are not needed: the task vectors themselves serve as a proxy for hidden activations. \optmerge{}~\cite{wei2026optmerge} augments \wudi{} with low-rank denoising of the task-vector matrix and a stable initialization; it tunes the optimizer separately for full and LoRA fine-tuning.

\subsection{Test-Time Adaptation and Dynamic Merging}
Test-time adaptation methods~\cite{yang2024adamerging,yang2024representation,daheim2024model} use unlabeled test data to learn merging coefficients. AdaMerging~\cite{yang2024adamerging} is representative: it learns per-layer scales from test inputs, whereas our solvers determine which spectral directions should be inverted in a data-free manner. These two axes are complementary in principle; we can combine AdaMerging-style scaling on top of our closed-form solutions. Dynamic (MoE-style) merging~\cite{tang2024merging,huang2024emr,lu2024twin,shen2025efficient} loads task-specific modules at inference time, which requires router training and increases storage.

\subsection{Model Merging for Multimodal LLMs}
VL-merging~\cite{sung2023empirical} merges modality-specific encoders before fine-tuning. VisionFuse~\cite{chen2024enhancing} concatenates visual features and applies task arithmetic on the LLM. UnIVAL~\cite{shukor2023unival} interpolates between multimodal-task experts. DAMC~\cite{chen2024model} composes vision/audio/video MLLMs through parameter decoupling and online activation merging. AdaMMS~\cite{du2025adamms} performs unsupervised hyperparameter search but only merges two MLLMs at a time. UQ-Merge~\cite{qu2024textttuqmerge} uses uncertainty quantification on unlabeled inputs to determine the merging order, but it treats every fine-tuning subset as a separate task without capability-level categorization. Our prior conference version~\cite{wei2026optmerge} introduced the first MLLM merging benchmark with a clean separation of training data and evaluation suites for VQA, Geometry, Chart, OCR, and Grounding, and additionally studied modality merging across vision, audio, and video.

%% file: sections/03_preliminaries.tex
\section{Rethinking Optimization-Based Merging}
\label{sec:prelim}
\label{sec:theory}

This section first introduces the task-vector merging notation and the objective of \wudi{}/\optmerge{}, and then rethinks the same objective as a noisy linear inverse problem. This organization makes the transition from existing iterative merging to the closed-form solvers in Sec.~\ref{sec:methods} explicit.

\subsection{Preliminaries}
\label{sec:preliminaries}

\subsubsection{Notation and Per-Layer Operators}

\noindent\textbf{Models.} $\Theta_0\in\mathbb{R}^d$ denotes the parameters of a shared base model, and $\Theta_1,\ldots,\Theta_N$ denote the parameters of $N$ \emph{experts} obtained by fine-tuning $\Theta_0$ on task-specific data. We restrict merging to two-dimensional weight tensors, \ie, linear and projection layers: for layer $\ell$, $W_0^{(\ell)},W_i^{(\ell)}\in\mathbb{R}^{d_o\times d_i}$. Non-two-dimensional parameters, including normalization parameters, embeddings, biases, and position indices, are merged by parameter averaging.

\noindent\textbf{Task vectors.} The global task vector of expert $i$ is $\bm{\tau}_i:=\Theta_i-\Theta_0\in\mathbb{R}^d$. Restricted to layer $\ell$, the corresponding task-vector matrix is
\begin{equation}
\tau_i^{(\ell)} \;:=\; \Pi_\ell(\bm{\tau}_i) \;=\; W_i^{(\ell)} - W_0^{(\ell)} \;\in\; \mathbb{R}^{d_o\times d_i},
\label{eq:task-vector}
\end{equation}
where $\Pi_\ell$ extracts the $\ell$-th weight block. When the layer is fixed, we omit $\ell$ and write $\tau_i$ for clarity.

\noindent\textbf{Merged delta and initial point.} The \emph{merged delta} of a method is $\bm{\tau}_m$, and the merged model is $\Theta_m=\Theta_0+\bm{\tau}_m$. We use $\tau:=\bm{\tau}_m^{(\ell)}\in\mathbb{R}^{d_o\times d_i}$ to denote the per-layer 2-D variable optimized by \wudi{} or returned by our closed-form solvers. The \emph{initial merged delta} of an iterative or closed-form solver is denoted $\tau_{\rm init}$ to distinguish it from the base model $\Theta_0$.

\subsubsection{The \wudi{} Loss}
\wudi{} Merging~\cite{cheng2025whoever} notes that, based on a linear subspace assumption, the hidden-activation interference $(\tau-\tau_i)x$ for a linear layer with input activations $x\in\mathbb{R}^{d_i\times n_s}$ can be effectively approximated by replacing $x$ with $\tau_i^\top$. This formulation yields a completely data-free per-layer loss:

\begin{equation}
\min_{\tau}\;\mathcal{L}\bigl(\tau\bigr)=\sum_{i=1}^N \frac{1}{\|\tau_i\|_F^2}\,\bigl\|(\tau-\tau_i)\,\tau_i^\top\bigr\|_F^2.
\label{eq:wudi}
\end{equation}
To minimize this loss, \wudi{} applies the Adam optimizer for $T=300$ steps.

\subsubsection{\optmerge{} Improvements}
\optmerge{}~\cite{wei2026optmerge} extends \wudi{} with three algorithmic refinements.
\emph{(i)} On full fine-tuned models, the task-vector matrix is centered as $\tilde \tau_i=\tau_i-\bar\tau$ and projected onto its top-$k$ singular components $U_{1:k}\Sigma_{1:k}V_{1:k}^\top$, where the per-layer mean is $\bar\tau:=\tfrac1N\sum_{j=1}^N\tau_j$. The projected matrix replaces $\tau_i^\top$ in Eq.~\eqref{eq:wudi}, denoising the proxy.
\emph{(ii)} On LoRA fine-tuned models, where $\tau_i$ is rank-deficient and the merged vector tends to take ``shortcuts'' by inflating its Frobenius norm, the optimizer is replaced by SGD with implicit regularization, and a low-rank truncation is applied directly to $\tau_i$ without centering.
\emph{(iii)} The per-layer variable is initialized to $\tau_{\rm init}=\bar\tau$, which stabilizes the training trajectory.

\optmerge{} produces the strongest results on the MLLM merging benchmark~\cite{wei2026optmerge}. Despite these gains, \optmerge{} retains the $300$-iteration optimization loop. The rethinking below argues that this loop is not necessary: it performs an implicit spectral regularization that we can carry out in closed form.

\subsection{Rethinking as a Noisy Linear Inverse Problem}
\label{sec:rethinking}
The two unexplained facts about iterative \wudi{}/\optmerge{} (that the exact closed-form minimum is worse than $300$-step iterative descent, and that this descent dominates the merging wall-clock time) are explained by a single change of viewpoint: \wudi{} is a noisy linear inverse problem, and iterative descent on it acts as an \emph{implicit spectral regularizer} of an ill-posed normal equation. We show that the per-layer \wudi{} objective in Eq.~\eqref{eq:wudi} is a quadratic in $\tau$ with a closed-form minimum-norm solution (Proposition~\ref{prop:cf}), explain why that closed form is suboptimal in the presence of proxy noise, and prove that gradient flow induces an exact exponential spectral filter on the closed-form pseudoinverse (Proposition~\ref{prop:flow}).

\subsubsection{Closed-Form Normal Equation}
\label{sec:cf}

For a linear layer with task vectors $\tau_i\in\mathbb{R}^{d_o\times d_i}$, define the symmetric operator
\begin{equation}
A_i \;:=\; \frac{\tau_i^\top \tau_i}{\|\tau_i\|_F^2}\in\mathbb{R}^{d_i\times d_i},\quad
C \;:=\; \sum_{i=1}^N A_i,\quad
D \;:=\; \sum_{i=1}^N \tau_i\,A_i.
\label{eq:ACD}
\end{equation}
$C$ is symmetric positive semidefinite. We let $C=Q\Lambda Q^\top$ be the eigendecomposition with eigenvalues $\lambda_1\ge\lambda_2\ge\cdots\ge\lambda_{d_i}\ge 0$ \emph{sorted in descending order} and corresponding eigenvectors $q_k$, the columns of $Q$.

\begin{proposition}[Closed-form WUDI normal equation]
\label{prop:cf}
The \wudi{} objective in Eq.~\eqref{eq:wudi} is the quadratic
\begin{equation}
\mathcal{L}(\tau) \;=\; \mathrm{tr}\bigl(\tau\,C\,\tau^\top\bigr) - 2\,\mathrm{tr}\bigl(\tau\,D^\top\bigr) + \mathrm{const},
\label{eq:wudi-quad}
\end{equation}
with gradient $\nabla_\tau\mathcal{L}(\tau)=2(\tau C-D)$. Any stationary point therefore satisfies the normal equation
\begin{equation}
\tau\,C \;=\; D.
\label{eq:normal}
\end{equation}
The set of stationary points is non-empty: each row of $D$ lies in $\mathrm{Range}(C)$, equivalently $D = D\,C^\dagger C$. Among all stationary points the unique minimum-Frobenius-norm element is
\begin{equation}
\tau^{\rm cf} \;=\; D\,C^\dagger \;=\; D\,Q\,\Lambda^\dagger\,Q^\top,
\label{eq:closed-form}
\end{equation}
where $\Lambda^\dagger$ inverts the strictly positive eigenvalues and sets the remaining entries to zero.
\end{proposition}

\begin{table}[!t]
\centering
\caption{\wudi{} iteration-count sweep on CLIP-B/32 TA8, including the closed-form \wudi{} solution.}
\label{tab:wudi-iters}
\setlength{\tabcolsep}{3.5pt}
\renewcommand{\arraystretch}{1.08}
\resizebox{\columnwidth}{!}{%
\begin{tabular}{l|cccccc|c}
\toprule
\textbf{Solver} & \multicolumn{6}{c|}{\textbf{Iterative \wudi{} ($T$ steps)}} & \textbf{Closed form} \\
\cmidrule(lr){2-7}\cmidrule(lr){8-8}
 & 100 & 200 & 300 & 500 & 700 & 1000 & $DC^\dagger$ \\
\midrule
Avg.~Acc.~(\%) & 80.08 & 83.82 & 84.63 & \textbf{84.82} & 84.72 & 84.52 & 82.33 \\
\bottomrule
\end{tabular}}
\end{table}

\begin{proof}[Sketch]
By the definition of $A_i$, each normalized term in Eq.~\eqref{eq:wudi} is a trace quadratic in $\tau-\tau_i$; summing and collecting terms gives Eq.~\eqref{eq:wudi-quad}, and differentiating gives $\nabla_\tau\mathcal{L}=2(\tau C-D)$, hence the normal equation $\tau C=D$. Consistency follows because any $z\in\mathrm{Null}(C)$ is annihilated by every task vector: $\tau_i z=0$ for all $i$, which implies $A_i z=0$ and $Dz=0$. Thus $\mathrm{Null}(C)\subseteq\mathrm{Null}(D)$, equivalently the rows of $D$ lie in $\mathrm{Range}(C)$, or $D=DC^\dagger C$. The solutions are therefore $\tau=DC^\dagger+Z(I-CC^\dagger)$ with arbitrary $Z$. The second term lies in the null space of $C$ and is orthogonal to $DC^\dagger$, so the unique minimum-Frobenius-norm solution sets it to zero, giving Eq.~\eqref{eq:closed-form}. The eigenform follows from $C^\dagger=Q\Lambda^\dagger Q^\top$.
\end{proof}

\subsubsection{Why Exact Closed Form Is Suboptimal}
\label{sec:noise}

The closed-form solution is empirically \emph{not} optimal for downstream performance. Table~\ref{tab:wudi-iters} makes this gap concrete on CLIP-ViT-B/32: iterative \wudi{} improves at early steps, peaks at a finite iteration count, and then degrades as the trajectory approaches the exact pseudoinverse. The closed-form \wudi{} solution $DC^\dagger$ reaches only $82.33\%$, below both the $300$-step result ($84.63\%$) and the best early-stopped result ($84.82\%$). This unimodal pattern is the empirical signature that early stopping regularizes the inverse problem, whereas excessive optimization recovers noise-amplifying tail directions.

\begin{figure}[!t]
\centering
\includegraphics[width=0.78\columnwidth]{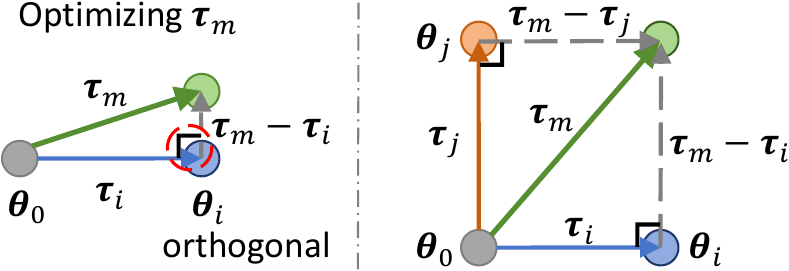}
\caption{\textbf{Norm-shortcut failure mode of unregularized iterative optimization.} When optimizing Eq.~\eqref{eq:wudi}, $\tau_m$ tends to take a shortcut by inflating its magnitude to make $(\tau_m-\tau_i)\tau_i^\top$ approximately orthogonal to each task vector, rather than aligning with the signal subspace.}
\label{fig:norm-shortcut}
\end{figure}

The reason is a noise-amplification mechanism familiar from inverse problems.
Decompose \(D=\tau^\circ C+E\), where \(\tau^\circ\) is an unobserved ideal merged delta and \(E\) denotes proxy-induced noise from replacing the true input activations \(x_i\) with the task-vector proxy \(\tau_i^\top\).
For each eigendirection \(q_k\), let \(\tau_k^\circ:=\tau^\circ q_k\in\mathbb{R}^{d_o}\).
Projecting onto \(q_k\) gives
\begin{equation}
y_k\,:=\,Dq_k \;=\; \lambda_k\,\tau_k^\circ + \xi_k,\qquad \xi_k:=Eq_k.
\label{eq:projection}
\end{equation}
For $\lambda_k>0$, the closed-form pseudoinverse gives
\begin{equation}
\tau^{\rm cf}_k = \tau^\circ_k + \xi_k/\lambda_k.
\label{eq:noise-amp}
\end{equation}
Small eigenvalues $\lambda_k$ correspond to row-space directions weakly supported by any task vector, exactly where the proxy-induced noise $\xi_k$ is large in magnitude relative to $\lambda_k\tau^\circ_k$. The pseudoinverse \emph{amplifies} this noise. Fig.~\ref{fig:theory-chain} empirically supports this view: the leading eigendirections explain nearly all proxy reduction, whereas the full pseudoinverse overfits the proxy and yields higher real interference. A regularized solver therefore replaces the inversion $1/\lambda_k$ by $h_k/\lambda_k$ for a filter $h_k\in[0,1]$ that vanishes (or shrinks) for small $\lambda_k$.

A related parameter-level instability surfaces as unconstrained inflation of $\|\tau\|_F$: when ill-conditioned directions are inverted, an iterative solver of Eq.~\eqref{eq:wudi-quad} can drive $\|\tau\|_F$ upward to make $(\tau-\tau_i)\tau_i^\top$ approximately orthogonal to each $\tau_i$, rather than recovering the underlying signal. The two phenomena are linked (both stem from poorly damped small-$\lambda_k$ directions). Fig.~\ref{fig:norm-shortcut} illustrates the geometry: when task vectors lie in a narrow subspace, the unique low-loss direction lies far from the origin, so unregularized descent on Eq.~\eqref{eq:wudi-quad} keeps inflating the merged-vector norm. The link to theory is made formal in Proposition~4 (Appendix~B): the WUDI proxy bounds the real per-layer interference up to a Frobenius slack term proportional to $\|\tau-\tau_i\|_F^2$, so once $\|\tau\|_F$ inflates, the slack term dominates and the proxy ceases to control the real interference. Suppressing small-$\lambda_k$ directions therefore plays a dual role: it removes the noise-amplifying inversion of Eq.~\eqref{eq:noise-amp} and keeps $\|\tau\|_F$ controlled, restoring tightness of the proxy bound.

\begin{figure*}[!t]
\centering
\includegraphics[width=0.7\textwidth]{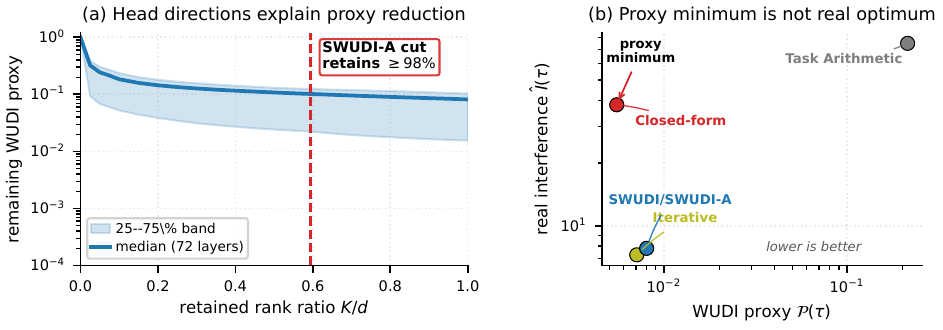}
\vspace{-0.5em}
\caption{\textbf{The exact pseudoinverse is insufficient.}
\textbf{(a)} Most WUDI proxy reduction is achieved by the leading eigendirections: the \aswudi{} cut retains at least $98\%$ of the median proxy reduction, indicating that the discarded spectral tail contributes little to the proxy objective.
\textbf{(b)} The closed-form pseudoinverse $DC^\dagger$ minimizes the WUDI proxy $\mathcal{P}(\tau)$ but yields higher real interference $\hat I(\tau)$ than the regularized alternatives in the lower-left region. This motivates regularized rather than full pseudoinversion.}\label{fig:theory-chain}
\end{figure*}

\subsubsection{Why Iterative Descent Works? Implicit Spectral Filtering}
\label{sec:flow}

\begin{proposition}[Gradient flow induces an exponential spectral filter]
\label{prop:flow}
Consider the gradient flow $\dot\tau(t)=-\tfrac{1}{2}\nabla_\tau\mathcal{L}(\tau(t))$ for the loss in Eq.~\eqref{eq:wudi-quad}, started at $\tau(0)=\tau_{\rm init}$. The flow is the linear ODE $\dot\tau(t)=D-\tau(t)\,C$, with closed-form solution
\begin{equation}
\tau(t)\;=\;\tau_{\rm init} \;+\; \bigl(D C^{\dagger} - \tau_{\rm init}\,C C^\dagger\bigr)\,Q\,\mathrm{diag}\bigl(h_k(t)\bigr)\,Q^\top,
\label{eq:flow}
\end{equation}
where the spectral filter is
\begin{equation}
h_k(t)\;=\;1-e^{-\lambda_k t},\qquad k=1,\ldots,d_i.
\label{eq:flow-filter}
\end{equation}
Thus, $h_k(t)\to 0$ as $\lambda_k\to 0^+$, while $h_k(t)\to 1$ as $\lambda_k t\to\infty$. The closed-form pseudoinverse $\tau^{\rm cf}=DC^\dagger$ is the $t\to\infty$ limit on the column space of $C$ (\ie, on the directions where $\lambda_k>0$), and the trajectory remains at $\tau_{\rm init}$ on the null space.
\end{proposition}

\begin{proof}
Let $\tilde\tau(t):=\tau(t)Q$ and $\tilde D:=DQ$. The flow decouples into $d_i$ vector ODEs $d\tilde\tau_k/dt=\tilde D_k-\lambda_k\tilde\tau_k$ (one per eigendirection, each $\tilde\tau_k,\tilde D_k\in\mathbb{R}^{d_o}$). For $\lambda_k>0$ the unique solution is $\tilde\tau_k(t)=\tilde\tau_{{\rm init},k}\,e^{-\lambda_k t}+(\tilde D_k/\lambda_k)(1-e^{-\lambda_k t})$. For $\lambda_k=0$, $\tilde\tau_k(t)\equiv\tilde\tau_{{\rm init},k}$. Reassembling and identifying the filter completes the proof.
\end{proof}

\begin{figure}[!t]
\centering
\includegraphics[width=0.8\columnwidth]{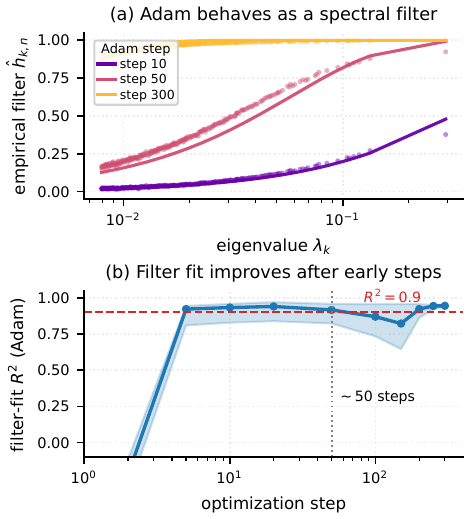}
\vspace{-0.5em}
\caption{\textbf{Iterative merging is implicit spectral filtering.} \textbf{(a)} Adam's empirical update filter is well approximated by an exponential spectral filter at three checkpoints: large-eigenvalue directions are fitted earlier than small-eigenvalue directions. \textbf{(b)} Across layers, the median fit quality exceeds $R^2=0.9$ after about $50$ steps. This supports replacing the iterative loop with the closed-form filter. The exact SGD/Landweber identity is reported in Appendix~C.}
\label{fig:optimizer}
\end{figure}

The discrete Landweber iteration $\tau_{n+1}=\tau_n+\eta(D-\tau_n C)$ admits the corresponding filter $h_k^{\rm LW}(n)=1-(1-\eta\lambda_k)^n$, stable for $0<\eta<2/\lambda_{\max}$. For the optimizer used by \wudi{}/\optmerge{}, the Adam trajectory empirically matches an exponential filter $1-e^{-t_{\rm eff}\lambda_k}$ up to a non-trivial $R^2\in[0.92,0.97]$ on per-layer fits (see Fig.~\ref{fig:optimizer}), with Spearman correlation approaching $0.99$ for step counts $\ge 50$. We therefore treat Adam as an empirically early-stopped spectral regularizer.

\subsubsection{Implication for Solver Design}
\label{sec:implication}
The analysis above gives a direct design principle: instead of minimizing the proxy objective for hundreds of optimizer steps, compute a spectrally regularized solution of the normal equation $\tau C=D$ in closed form. Concretely, the full pseudoinverse factor $1/\lambda_k$ in $DC^\dagger$ should be replaced by $h_k/\lambda_k$, where $h_k\in[0,1]$ attenuates eigendirections that are weakly supported and therefore prone to noise amplification.

From this, two complementary filter components naturally arise: (1) \textbf{Soft filtering} ($h_k=1-e^{-t\lambda_k}$), where a continuous time parameter $t$ matches the gradient-flow stopping time; and (2) \textbf{Hard truncation} ($h_k=\mathbf{1}[k\le K]$), where a rank cutoff $K$ removes the poorly conditioned spectral tail and can be tuned globally or adapted per layer. The following section integrates these components into a unified spectral filtering estimator (Eq.~\eqref{eq:unified}), instantiating this approach as a hybrid solver and an adaptive per-layer truncation rule.

\noindent\emph{Connection to existing methods.} This filter view places several data-free merging methods in a common language. Tikhonov regularization corresponds to the classical filter $h_k=\lambda_k/(\lambda_k+\alpha)$. Iso-C~\cite{marczak2025no} and TSV-Merging~\cite{gargiulo2024task} can be re-read as fixed shrinkage or truncation rules on the spectrum of the stacked task-vector matrix, equivalently on the square-root spectrum of $C$. Finite-step gradient descent on the \wudi{}~\cite{cheng2025whoever} quadratic gives the Landweber filter, while the Adam optimizer used in \wudi{}/\optmerge{} empirically behaves as an early-stopped spectral regularizer rather than an exact Landweber iteration.

More broadly, this comparison clarifies two levels at which a data-free merging method can intervene: it can denoise the proxy normal equation itself, thereby modifying the estimated operator/right-hand-side pair $(C,D)$, or it can keep the proxy equation fixed and regularize the inversion of its ill-conditioned operator. Our \optmerge{} mainly belongs to the first category and additionally stabilizes the iterative optimization trajectory. The next section pursues the second route by replacing the full pseudoinverse $C^\dagger$ with closed-form spectral filters that suppress noise-amplifying eigendirections.

%% file: sections/04_methods.tex
\section{Methodology}
\label{sec:methods}

We now turn the spectral view of Sec.~\ref{sec:theory} into closed-form, data-free merging algorithms. For each layer, we reuse the normal-equation quantities $C$ and $D$ from Eq.~\eqref{eq:ACD}, with eigendecomposition $C=Q\Lambda Q^\top$ and eigenvalues $\lambda_1\ge\cdots\ge\lambda_{d_i}$.

\subsection{\swudi{}: Spectrally Regularized \wudi{}}
\label{sec:swudi}

We first cast all closed-form spectral solvers of $\tau C=D$ into a single family parameterized by a per-direction filter $h_k\in[0,1]$, and then specialize the family to obtain \swudi{}. For spectral filter coefficients $\{h_k\}_{k=1}^{d_i}$, define $C_h^\dagger:=Q\operatorname{diag}(h_k/\lambda_k)Q^\top$, with the Moore--Penrose convention that the diagonal entry is set to $0$ on directions with $\lambda_k=0$. The unified spectral filtering estimator is
\begin{equation}
\widehat{\tau}_h \;=\; \tau_{\rm init} \;+\; (D-\tau_{\rm init}\,C)\,C_h^\dagger.
\label{eq:unified}
\end{equation}
The core operation is the per-direction filter $h_k$, which controls how strongly each eigendirection of $C$ is inverted; $\tau_{\rm init}$ is an optional base point at which the filter is applied, and the $\tau_{\rm init}=0$ specialization $\widehat{\tau}_h=D\,C_h^\dagger$ recovers the direct filtered inverse and remains close in accuracy in our experiments. Eq.~\eqref{eq:unified} makes clear that the choice of $h_k$ determines the regularization. Two filter behaviors are essential for \swudi{}: the exponential filter $h_k=1-e^{-t\lambda_k}$ exactly recovers the gradient-flow solution stopped at time $t$ from Proposition~\ref{prop:flow}, transferring the early-stopping effect of iterative \wudi{} into closed form; the hard filter $h_k=\mathbf{1}[k\le K]$ yields a rank-$K$ truncated spectral inverse, removing weakly supported tail directions. By contrast, the unregularized choice $h_k\equiv 1$ recovers the pseudoinverse limit $DC^\dagger$ and serves as the unstable reference case.

This filter perspective dictates the behavior of $h_k$. The gradient-flow filter $s_t(\lambda)=1-e^{-t\lambda}$ captures the early-stopping regularization provided by finite-step iterative descent across the bulk of the spectrum. However, its small-$\lambda$ behavior must be evaluated in terms of the quantity that actually enters $C_h^\dagger$: although $s_t(\lambda)\to 0$ as $\lambda\to 0$, the \emph{effective inverse gain} $g_t(\lambda):=s_t(\lambda)/\lambda$ converges to $t$. Consequently, pure exponential filtering still allows tail noise to propagate with finite gain across many weakly supported directions. Conversely, hard truncation $\mathbf{1}[k\le K]$ completely eliminates the tail by setting $g_t$ to zero for truncated indices, but it fails to reproduce the gradient-flow regularization on the retained directions.
\swudi{} combines the two by multiplying them into a single two-factor filter that we plug into Eq.~\eqref{eq:unified}:
\begin{equation}
\begin{aligned}
h_k &= m_k \cdot s_t(\lambda_k),\quad s_t(\lambda_k)=1-e^{-t\lambda_k},\\
m_k &= \mathbf{1}[k\le K],\quad K=\lceil r\,d_i\rceil.
\end{aligned}
\label{eq:swudi}
\end{equation}
Therefore, \swudi{} improves merging quality not through exact proxy minimization, but by preventing the proxy inverse from overfitting to noise-amplifying tail directions. The retained head directions capture most of the transferable task signal while keeping the merged delta norm $\|\tau\|_F$ controlled (Sec.~\ref{sec:noise}). Two hyperparameters control this regularizer: a continuous exponential time $t\ge 0$, which corresponds to the gradient-flow stopping time of \wudi{}, and a rank ratio $r\in(0,1]$. The soft factor $s_t$ applies early-stopping regularization to the retained directions, while the hard mask $m_k$ zeroes out the effective inverse gain $g_t$ on the long tail of small-$\lambda_k$ directions before they can introduce noise-amplifying weights into $C_h^\dagger$. Ultimately, the merged delta $\widehat{\tau}_h$ is computed using Eq.~\eqref{eq:unified}, with the filter $h_k$ defined in Eq.~\eqref{eq:swudi}.

\begin{algorithm}[!t]
\caption{Unified closed-form spectral merging}
\label{alg:swudi}
\small
\begin{algorithmic}[1]
\Require expert deltas $\{\tau_i^{(\ell)}\}_{i=1}^N$ for every linear layer; optional \swudi{} parameters $(t,r)$
\Ensure merged delta $\bm{\tau}_m$
\For{each linear layer $\ell$ with $\tau_i\in\mathbb{R}^{d_o\times d_i}$}
  \State $A_i\gets \tau_i^\top\tau_i/\|\tau_i\|_F^2$,\quad $C\gets\sum_i A_i$,\quad $D\gets\sum_i\tau_i A_i$
  \State $C=Q\operatorname{diag}(\lambda_1,\ldots,\lambda_{d_i})Q^\top$,\quad $\lambda_1\ge\cdots\ge\lambda_{d_i}$
  \Statex \hspace{1.2em}\Comment{$C$ is the spectral operator; $D$ is the right-hand side.}
  \State $\tau_{\rm init}\gets\sum_i\tau_i$
  \State $K_{\rm S}\gets\lceil r d_i\rceil$
  \State $K_{\rm A}^{(\ell)}\gets\left\lceil(\sum_k\sqrt{\lambda_k})^2/\sum_k\lambda_k\right\rceil$
  \Statex \hspace{1.2em}$h_k=\begin{cases}
  \mathbf{1}[k\le K_{\rm S}](1-e^{-t\lambda_k}), & \swudi{},\\
  \mathbf{1}[k\le K_{\rm A}^{(\ell)}], & \aswudi{}.
  \end{cases}$
  \Statex \hspace{1.2em}\Comment{$h_k$ regularizes each eigendirection.}
  \State $C_h^\dagger\gets Q\operatorname{diag}(h_k/\lambda_k)Q^\top$
  \State $\bm{\tau}_m^{(\ell)}\gets \tau_{\rm init}+(D-\tau_{\rm init}C)\,C_h^\dagger$
\EndFor
\State Average non-2-D parameters and return $\bm{\tau}_m$.
\end{algorithmic}
\end{algorithm}

\subsection{\aswudi{}: Adaptive Variant}
\label{sec:aswudi}

The rank ratio $r$ in \swudi{} is global. However, spectra differ significantly across layers (\eg, attention $q/k/v/o$, MLP, embedding) and architectures (\eg, CLIP-ViT, Flan-T5, Llama, MLLMs). The adaptive variant, \aswudi{}, addresses this by choosing $K_\ell$ per layer using a closed-form rank rule based on the eigenvalues $\{\lambda_k^{(\ell)}\}$, thereby eliminating the need for a global rank hyperparameter. Within the unified spectral estimator (Eq.~\eqref{eq:unified}), \aswudi{} acts as a hard-truncation specialization: it sets the soft factor to the identity for retained directions and replaces the global rank ratio with a layer-wise spectral rank rule $K_\ell$, allowing the spectrum itself to dictate the cutoff.

\textbf{Layer-wise rank selection:}
We provide two parameter-free layer-wise rank rules, each corresponding to a specific spectral regime and computable from the existing eigendecomposition. Both operate on the singular values of the stacked task-vector matrix $M:=[\tau_1/\|\tau_1\|_F;\ldots;\tau_N/\|\tau_N\|_F]\in\mathbb{R}^{Nd_o\times d_i}$. Because $M^\top M=\sum_i A_i=C$, the singular values of $M$ are simply $\sigma_k=\sqrt{\lambda_k}$, allowing these rank rules to be evaluated directly from the eigenspectrum of $C$.

\emph{(i) Participation-square-root rule.}
\begin{equation}
K_\ell^{\rm psqrt}\;=\;\left\lceil \frac{\bigl(\sum_k\sigma_k\bigr)^2}{\sum_k\sigma_k^2}\right\rceil
=\left\lceil\frac{\bigl(\sum_k\sqrt{\lambda_k^{(\ell)}}\bigr)^2}{\sum_k\lambda_k^{(\ell)}}\right\rceil.
\label{eq:psqrt}
\end{equation}
This applies a participation-ratio effective-rank estimator~\cite{roy2007effective} to $\sigma_k$, thus measuring the effective column rank of $M$ rather than the squared-energy rank of $C$. This prevents undue concentration on the largest eigenvalues in heavy-tailed spectra. We utilize this as the default rule when the spectrum decays smoothly without a distinct noise floor.

\emph{(ii) Marchenko--Pastur Gavish--Donoho rule.}
\begin{equation}
K_\ell^{\rm Gavish}\;=\;\bigl|\bigl\{k:\sigma_k>\omega_{\rm GD}(\beta)\,\widehat\sigma_{\rm med}\bigr\}\bigr|,
\label{eq:gd}
\end{equation}
where $\beta=\min(Nd_o,d_i)/\max(Nd_o,d_i)$, $\omega_{\rm GD}(\beta)$ is the Gavish--Donoho ratio~\cite{gavish2014optimal}, and $\widehat\sigma_{\rm med}=\mathrm{median}_k(\sigma_k)$ robustly estimates the noise scale. Under a Marchenko--Pastur spiked model $M=M^\circ+\Xi$ with low-rank $M^\circ$ and i.i.d.\ noise $\Xi$, this rule recovers the spike rank with high probability~\cite{marchenko1967distribution}. We apply it to spectra with a clear noise bulk and isolated spikes.

In summary, \texttt{psqrt} provides a smooth participation count that consistently returns a positive rank and tolerates heavy tails, whereas \texttt{Gavish-Donoho} acts as a strict noise-floor test that may return a rank of $0$ if no singular value is significant. Consequently, the appropriate rule can be selected based on the spectrum and fine-tuning regime prior to downstream evaluation. Fig.~\ref{fig:spectral-rank} visualizes these regimes, with detailed per-architecture statistics provided in Appendix~C. Algorithm~\ref{alg:swudi} summarizes the unified closed-form procedure.

\subsection{Computational Complexity}
\label{sec:complexity}
For a single linear layer, the dominant cost is the symmetric eigendecomposition of $C\in\mathbb{R}^{d_i\times d_i}$, which is $O(d_i^3)$ time and $O(d_i^2)$ memory. Forming $C$ and $D$ costs $O(N d_o d_i^2)$ FLOPs. Iterative \wudi{}/\optmerge{} performs $T$ matrix multiplications of similar shapes per layer plus Adam first/second moments, so the wall-clock speedup is roughly $T/c$, where $c$ captures implementation-dependent constants. Empirically, we observe $28$--$72\times$ speedups (Sec.~\ref{sec:analysis}, Table~\ref{tab:efficiency}). Because no Adam state is needed, peak GPU memory is reduced by approximately the size of the optimizer state.

%% file: sections/05_benchmarks.tex
\section{Benchmarks and Experimental Results}
\label{sec:bench}
\label{sec:exp_results}

This section details our experimental setup and results. We evaluate five merging scenarios spanning vision, language, and multimodal foundation models, utilizing both LoRA and full fine-tuning settings. Sec.~\ref{sec:bench-mllm} focuses on our proposed MLLM merging benchmark, while Sec.~\ref{sec:bench-general} extends the evaluation to four widely adopted model merging benchmarks. Finally, we provide a comprehensive discussion and analysis.

\subsection{MLLM Merging Benchmark}
\label{sec:bench-mllm}

We evaluate our approach on our MLLM merging benchmark, briefly summarizing the setup here while deferring comprehensive details to Appendix~D. Fig.~\ref{fig:bench-setting} illustrates the benchmark's two settings, capability merging and modality merging.

\begin{figure}[!t]
\centering
\includegraphics[width=0.92\columnwidth]{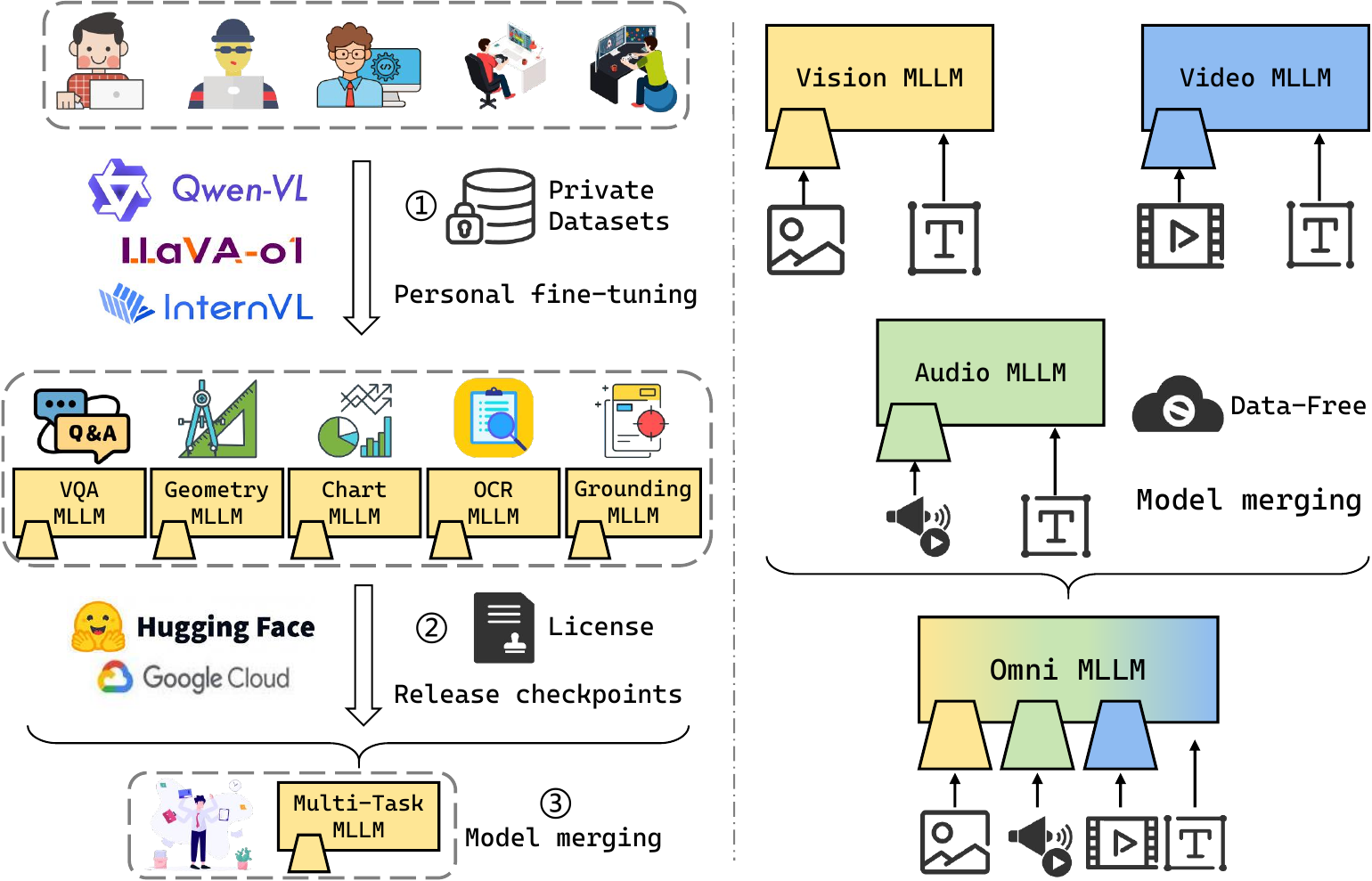}
\caption{\textbf{Two settings of the MLLM merging benchmark.} \emph{Capability merging} (left) combines task-specialized experts that share the same MLLM backbone into a single multi-task model covering VQA, Geometry, Chart, OCR, and Grounding. \emph{Modality merging} (right) composes vision-, audio-, and video-language experts that share an LLM backbone but use modality-specific encoders and connectors. Both settings are data-free, enabling the merged model to retain expert capabilities without requiring joint training data.}
\label{fig:bench-setting}
\end{figure}

\noindent\textbf{Backbones.} We consider two MLLMs that cover both fine-tuning regimes: \emph{InternVL2.5-1B-Instruct}~\cite{chen2024expanding} (full fine-tuning) and \emph{Qwen2-VL-7B-Base}~\cite{wang2024qwen2} (LoRA fine-tuning). For modality merging, we follow~\cite{chen2024model} and use Vicuna-7B-v1.5~\cite{zheng2023judging} paired with CLIP-ViT-L-336px for vision, BEATs-Iter3+ with a Q-Former for audio, and LanguageBind for video.

\noindent\textbf{Tasks and data.} We consider five capabilities (VQA, Geometry, Chart, OCR, Grounding), each with at least 100K training samples. The dataset table is detailed in Appendix~D.

\noindent\textbf{Evaluation.} We use VLMEvalKit~\cite{duan2024vlmevalkit} and lmms-eval~\cite{zhang2024lmms} under matched settings. For capability evaluation, we report results on VizWiz~\cite{gurari2018vizwiz}, GQA~\cite{hudson2019gqa}, MathVista~\cite{lu2024mathvista}, MATH-Vision~\cite{wang2024measuring}, ChartQA~\cite{masry2022chartqa}, TextVQA~\cite{singh2019towards}, OCRVQA~\cite{mishra2019ocr}, and RefCOCO/+/g~\cite{kazemzadeh2014referitgame}. For integrated multimodal QA, we report results on MMMU~\cite{yue2024mmmu}, DocVQA~\cite{mathew2021docvqa}, ScienceQA~\cite{lu2022learn}, AI2D~\cite{kembhavi2016diagram}, and InfographicVQA~\cite{mathew2022infographicvqa}. For modality merging, we report results on MUSIC-AVQA~\cite{li2022learning} and AVQA~\cite{yang2022avqa}.

\subsection{General Model Merging Benchmarks}
\label{sec:bench-general}

\noindent\textbf{Benchmarks.}
\textbf{(i) CLIP-ViT TA8.}
The standard 8-task vision benchmark from FusionBench~\cite{tang2024fusionbench} (SUN397, Cars, RESISC45, EuroSAT, SVHN, GTSRB, MNIST, DTD). We evaluate three CLIP-pretrained backbones~\cite{radford2021learning}: ViT-B/32, ViT-B/16, and ViT-L/14, reporting the mean per-task accuracy.
\textbf{(ii) CLIP-ViT-B/32 TALL20.}
A 20-task extension used to evaluate scalability as the number of merged tasks increases.
\textbf{(iii) Flan-T5-base on GLUE.}
Eight GLUE tasks~\cite{wang2018glue} fine-tuned with rank-16 LoRA on Flan-T5-base~\cite{chung2024scaling}. This evaluates our method on rank-deficient task-vector matrices in the NLP domain.
\textbf{(iv) Llama-3.2-3B.}
Following MergeBench~\cite{he2026mergebench}, five domain experts (math, code, instruction following, safety, multilingual) are merged into a single Llama-3.2-3B model~\cite{dubey2024llama}. Evaluation uses \texttt{lm\_eval} across GSM8K~\cite{cobbe2021training}, HumanEval/MBPP+~\cite{chen2021evaluating,liu2023is}, IFEval~\cite{zhou2023instruction}, TruthfulQA~\cite{lin2022truthfulqa}, MMLU~\cite{hendrycks2021measuring}, ARC~\cite{clark2018think}, and HellaSwag~\cite{zellers2019hellaswag}. This setting employs full-parameter deltas with 3B trainable parameters per expert.

\noindent\textbf{Methods.}
We compare our proposed \swudi{} together with its adaptive variant \aswudi{}, against several model merging methods: Weight Average~\cite{wortsman2022model}, Task Arithmetic~\cite{ilharcoediting}, TIES~\cite{yadav2023ties}, DARE-TA and DARE-TIES~\cite{yu2024language}, TSV-Merging~\cite{gargiulo2024task}, Iso-C~\cite{marczak2025no}, \wudi{} Merging~\cite{cheng2025whoever}, and our \optmerge{}~\cite{wei2026optmerge}.

\noindent\textbf{Hyperparameters.}
Following common practice, we search over the global scaling coefficient $s\in\{0.1,0.2,0.3,0.5,1.0\}$ for all merging methods. For \swudi{}, we additionally tune $r\in\{0.55,0.60,0.65,0.70,0.75,0.85\}$ and $t\in\{300,500,700,1000,1300,1800\}$. In contrast, \aswudi{} requires no continuous hyperparameters beyond the global $s$. Specifically, \aswudi{} applies the participation-square-root rule in Eq.~\eqref{eq:psqrt} on CLIP-ViT, Flan-T5, and the MLLMerging benchmark, while employing the Gavish--Donoho rule in Eq.~\eqref{eq:gd} on the Llama-3.2-3B MergeBench (a full-parameter LLM with spiked-noise spectra).

\begin{table*}[!t]
\centering
\caption{Capability merging results on InternVL2.5 (full fine-tuning) across multiple tasks. Best scores are bolded and second-best scores are underlined.}
\label{tab:intern}
\setlength{\tabcolsep}{2pt}
\renewcommand{\arraystretch}{1.05}
\resizebox{0.7\textwidth}{!}{%
\begin{tabular}{l|cc|cc|c|cc|ccc|c}
\toprule
\multirow{2}{*}{\textbf{Method}} & \multicolumn{2}{c|}{\textbf{VQA}} & \multicolumn{2}{c|}{\textbf{Geometry}} & \textbf{Chart} & \multicolumn{2}{c|}{\textbf{OCR}} & \multicolumn{3}{c|}{\textbf{Grounding}} & \multirow{2}{*}{\textbf{Avg.}}\\
\cmidrule{2-11}
 & VizWiz & GQA & MathVista & MATH-Vision & ChartQA & TextVQA & OCRVQA & RefCOCO & RefCOCO+ & RefCOCOg & \\
\midrule
InternVL2.5-Instruct & 29.15 & 54.62 & 45.40 & 18.09 & 69.48 & 72.51 & 41.08 & 71.69 & 65.41 & 67.40 & 53.48 \\
\midrule
Weight Average        & 29.96 & 54.89 & 42.30 & 17.76 & 71.64 & 74.54 & 41.86 & 52.62 & 45.29 & 52.39 & 48.33 \\
Task Arithmetic       & 30.67 & 56.34 & 40.70 & 17.43 & \underline{72.88} & 76.26 & 43.39 & 74.90 & 68.15 & 72.75 & 55.35 \\
TIES Merging          & 30.63 & 56.48 & 44.10 & 16.78 & 72.28 & \underline{76.29} & 44.01 & 76.01 & 68.45 & 73.65 & 55.87 \\
TA w/ DARE            & 30.61 & 56.48 & 40.40 & 15.79 & \textbf{73.08} & \textbf{76.30} & 43.03 & 74.94 & 68.07 & 73.02 & 55.17 \\
TIES w/ DARE          & 30.65 & 56.11 & 44.30 & \underline{18.09} & 72.72 & 76.19 & 43.33 & 75.10 & 68.48 & 73.55 & 55.85 \\
TSV Merging           & \underline{31.15} & 56.67 & 44.90 & 17.76 & 70.56 & 75.66 & 45.38 & 65.19 & 58.51 & 59.17 & 52.50 \\
Iso-C                 & 28.21 & 55.36 & 42.10 & \underline{18.09} & 70.56 & 69.34 & \textbf{46.51} & 72.72 & 66.56 & 68.50 & 53.80 \\
\rowcolor{gray!10}WUDI Merging & 31.02 & 56.96 & 44.80 & 15.31 & 69.19 & 75.95 & 46.12 & 76.06 & \underline{70.14} & \textbf{74.48} & 56.00 \\
\rowcolor{gray!10}OptMerge & 30.85 & \textbf{57.05} & \textbf{46.90} & 15.79 & 68.80 & 75.98 & \underline{46.35} & 76.09 & 69.82 & \underline{74.18} & 56.18 \\
\midrule
\rowcolor{gray!10}\swudi{} & 31.11 & \underline{57.04} & \underline{46.60} & \textbf{18.42} & 69.76 & 76.04 & 46.06 & \textbf{76.24} & \textbf{70.18} & 74.12 & \textbf{56.56} \\
\rowcolor{gray!10}\aswudi{}        & \textbf{31.25} & 56.85 & 46.10 & 16.45 & 70.44 & 76.00 & 45.90 & \underline{76.20} & 69.99 & 74.08 & \underline{56.33} \\
\midrule
Mixture Training      & 29.79 & 61.33 & 45.00 & 17.11 & 70.32 & 72.96 & 60.25 & 72.06 & 65.93 & 67.46 & 56.22 \\
\bottomrule
\end{tabular}}
\end{table*}

\begin{table*}[!t]
\centering
\caption{Capability merging results on Qwen2-VL (LoRA fine-tuning) across multiple tasks. Best scores are bolded and second-best scores are underlined.}
\label{tab:qwen}
\setlength{\tabcolsep}{2pt}
\renewcommand{\arraystretch}{1.05}
\resizebox{0.7\textwidth}{!}{%
\begin{tabular}{l|cc|cc|c|cc|ccc|c}
\toprule
\multirow{2}{*}{\textbf{Method}} & \multicolumn{2}{c|}{\textbf{VQA}} & \multicolumn{2}{c|}{\textbf{Geometry}} & \textbf{Chart} & \multicolumn{2}{c|}{\textbf{OCR}} & \multicolumn{3}{c|}{\textbf{Grounding}} & \multirow{2}{*}{\textbf{Avg.}}\\
\cmidrule{2-11}
 & VizWiz & GQA & MathVista & MATH-Vision & ChartQA & TextVQA & OCRVQA & RefCOCO & RefCOCO+ & RefCOCOg & \\
\midrule
Qwen2-VL-Base         & 5.52 & 5.39 & 54.00 & 21.05 & 0.36 & 20.22 & 1.07 & 45.32 & 37.55 & 31.26 & 22.17 \\
\midrule
Weight Average        & 41.47 & 57.33 & 57.90 & \underline{25.66} & 59.56 & 81.09 & 57.85 & \underline{80.72} & 65.37 & 77.68 & 60.46 \\
Task Arithmetic       & 40.52 & \underline{62.31} & \textbf{58.40} & 23.68 & \underline{79.67} & 81.09 & 59.50 & 75.96 & 61.33 & 75.85 & 61.83 \\
TIES Merging          & 41.38 & 59.08 & 52.60 & 19.41 & 67.24 & 81.42 & 58.53 & 80.63 & 65.36 & 77.65 & 60.33 \\
TA w/ DARE            & 40.64 & \textbf{62.38} & \underline{58.10} & 23.68 & \textbf{79.76} & 81.04 & 59.34 & 75.83 & 61.41 & 75.80 & 61.80 \\
TIES w/ DARE          & \textbf{41.63} & 59.96 & 54.50 & 23.03 & 70.68 & 81.53 & 59.63 & \textbf{80.73} & \underline{65.65} & \underline{77.77} & 61.51 \\
TSV Merging           & 41.43 & 57.31 & 54.30 & 23.68 & 59.44 & 81.25 & 57.81 & 80.71 & 65.34 & 77.76 & 59.90 \\
Iso-C                 & 12.31 & 13.44 & 49.70 & 20.07 & 2.80 & 30.05 & 6.12 & 53.68 & 38.96 & 41.90 & 26.90 \\
\rowcolor{gray!10}WUDI Merging & 37.19 & 56.45 & 54.70 & \underline{25.66} & 67.84 & 79.92 & \textbf{65.56} & 76.25 & 60.72 & 71.99 & 59.63 \\
\rowcolor{gray!10}OptMerge & \underline{41.54} & 61.21 & \textbf{58.40} & \textbf{25.99} & 74.24 & 81.48 & 60.03 & 80.45 & \textbf{65.96} & 76.92 & \underline{62.62} \\
\midrule
\rowcolor{gray!10}\swudi{}  & 40.31 & 60.21 & 57.60 & 23.03 & 70.96 & \underline{81.60} & 63.96 & 80.12 & 65.45 & 76.07 & 61.93 \\
\rowcolor{gray!10}\aswudi{}     & 40.63 & 60.80 & 57.00 & 23.36 & 75.32 & \textbf{81.63} & \underline{64.23} & 80.22 & 65.62 & \textbf{78.42} & \textbf{62.72} \\
\midrule
Qwen2-VL-Instruct    & 44.09 & 62.18 & 57.20 & 17.43 & 70.04 & 78.38 & 65.42 & 82.89 & 77.87 & 75.63 & 63.11 \\
\bottomrule
\end{tabular}}
\end{table*}

\subsection{Multimodal Model Merging}
\label{sec:exp_modality}

We first evaluate the proposed solvers on our multimodal merging benchmark, covering full-parameter InternVL2.5, LoRA-fine-tuned Qwen2-VL, integrated multimodal QA, and modality merging across vision, audio, and video experts. Tables~\ref{tab:intern} and~\ref{tab:qwen} show that the closed-form spectral solvers consistently match or exceed the strongest iterative \wudi{}/\optmerge{} baselines on capability merging. On InternVL2.5-1B, \swudi{} obtains the best average accuracy ($56.56$), while \aswudi{} remains close behind ($56.33$) without method-specific rank tuning. On Qwen2-VL-7B, \aswudi{} reaches the best average ($62.72$), slightly above iterative \optmerge{} ($62.62$), which demonstrates that adaptive spectral truncation is especially useful when LoRA deltas are intrinsically low-rank.

The two backbones expose complementary benefits. In the full-parameter InternVL2.5 setting, spectral filtering preserves shared multimodal capabilities while improving the average over both optimization-based and spectrum-based baselines. In the Qwen2-VL LoRA setting, the adaptive rank rule prevents the norm-inflation failure mode of iterative data-free objectives and retains the compact directions that carry most of the LoRA signal. The multimodal results therefore support the central claim from two regimes: explicit spectral regularization is not only faster than iterative optimization, but also more stable when the task-vector geometry is low-rank or noisy.

The capability-level wins extend to comparisons against the corresponding mixture-trained model, the natural data-rich baseline. \swudi{} on InternVL2.5-1B slightly exceeds mixture training on average ($56.56$ vs.\ $56.22$), and \aswudi{} on Qwen2-VL-7B is within $0.4$ points of mixture training ($62.72$ vs.\ $63.11$). Reaching this accuracy regime without any joint training data, using only the experts' parameter deltas and a single eigendecomposition per layer, is the practical case that capability merging makes for production multimodal systems.

Next, we examine modality merging, where vision-, audio-, and video-language Vicuna-7B experts are integrated into an Omni-language model~\cite{chen2024model} (in Table~\ref{tab:omni}). \aswudi{} outperforms both \optmerge{} and all offline merging baselines. It even surpasses online composition methods that require modality-specific, inference-time composition. These results demonstrate that the spectral regularization principle generalizes effectively from capability merging to cross-modal composition, a setting where preserving complementary modality information is more critical than optimizing for any single expert.

\begin{table*}[t]
  \renewcommand{\arraystretch}{1}
  \caption{\textbf{Modality merging results on zero-shot image-audio-video question answering tasks} by merging vision-language, audio-language, and video-language models. The ``Individual Modalities'' columns show baseline performance for each single-modality model.}
  \label{tab:omni}
  \vspace{-0.4em}
  \centering
  \setlength{\tabcolsep}{2pt}
  \resizebox{0.8\textwidth}{!}{
\begin{tabular}{l|ccc|cccccccc|cc}
  \toprule
  & \multicolumn{3}{c|}{\textbf{Individual Modalities}} & \multicolumn{8}{c|}{\textbf{Merging Methods}} & \multicolumn{2}{c}{\textbf{Online Composing}} \\
  \cmidrule{2-14}
  \textbf{Datasets} & Vision & Audio & Video & \begin{tabular}[c]{@{}c@{}}Weight\\Average\end{tabular} & \begin{tabular}[c]{@{}c@{}}Task\\Arithmetic\end{tabular} & \begin{tabular}[c]{@{}c@{}}TIES\\Merging\end{tabular} & \begin{tabular}[c]{@{}c@{}}TSV\\Merging\end{tabular} & \begin{tabular}[c]{@{}c@{}}Iso-C\end{tabular} & \begin{tabular}[c]{@{}c@{}}WUDI\\Merging\end{tabular} & \begin{tabular}[c]{@{}c@{}}OptMerge\end{tabular} & \begin{tabular}[c]{@{}c@{}}SWUDI-A\end{tabular} & \begin{tabular}[c]{@{}c@{}}NaiveMC\end{tabular} & \begin{tabular}[c]{@{}c@{}}DAMC\end{tabular} \\
  \midrule
  \textbf{MUSIC-AVQA} & 50.77 & 27.93 & 49.02 & 47.75 & 52.14 & 50.35 & \underline{53.78} & 52.77 & 52.43 & 53.17 & \textbf{53.91} & 53.50 & 52.80 \\
  \textbf{AVQA} & 75.55 & 47.57 & 79.20 & 69.39 & 78.62 & 75.84 & \underline{80.90} & 77.51 & 76.86 & 80.82 & \textbf{81.26} & 80.26 & 80.78 \\
  \hline
  \textbf{Avg.} & 63.16 & 37.75 & 64.11 & 58.57 & 65.38 & 63.10 & \underline{67.34} & 65.14 & 64.65 & 67.00 & \textbf{67.59} & 66.88 & 66.79 \\
  \bottomrule
\end{tabular}
}
\end{table*}

\begin{table}[t]
\centering
\caption{Evaluation on general multimodal QA benchmarks.}
\label{tab:integrated}
\setlength{\tabcolsep}{4pt}
\renewcommand{\arraystretch}{1.05}
\resizebox{\columnwidth}{!}{%
\begin{tabular}{l|ccccc|c}
\toprule
\textbf{Method} & MMMU & DocVQA & SciQA & AI2D & InfoVQA & Avg. \\
\midrule
Individual VQA       & 26.00 & 62.93 & 50.83 & 44.59 & 39.07 & 44.68 \\
Individual Chart     & 30.33 & 57.13 & 40.01 & 29.86 & 26.02 & 36.67 \\
Individual Geometry  & 33.67 & 64.29 & 73.25 & 62.27 & 29.79 & 52.65 \\
Individual Grounding & 34.22 & 65.64 & 76.54 & 63.24 & 33.82 & 54.69 \\
Individual OCR       & 38.00 & 77.67 & 63.66 & 54.39 & 41.97 & 55.14 \\
\midrule
\rowcolor{gray!10}OptMerge & \textbf{39.33} & \textbf{84.18} & \underline{91.89} & \underline{79.44} & \textbf{56.84} & \underline{70.34} \\
\rowcolor{gray!10}\aswudi{}   & \textbf{39.33} & \underline{84.14} & \textbf{93.41} & \textbf{79.47} & \underline{56.57} & \textbf{70.58} \\
\bottomrule
\end{tabular}}
\end{table}

Finally, we assess whether the merged model preserves composite abilities rather than only isolated capabilities. Table~\ref{tab:integrated} evaluates the InternVL2.5-1B merge on integrated multimodal QA benchmarks. \aswudi{} attains the highest average ($70.58$), improving over the \optmerge{} result ($70.34$) and giving the largest gain on ScienceQA. This pattern is consistent with the noise-amplification analysis in Sec.~\ref{sec:noise}: integrated tasks are sensitive to spurious low-eigenvalue updates, so explicitly filtering those directions improves robustness beyond the per-capability averages.

Across capability, integrated-QA, and modality-merging evaluations, the closed-form solvers match or exceed iterative \wudi{}/\optmerge{} in nearly all average multimodal metrics, while reducing the merging cost by over an order of magnitude. The results demonstrate that merged multimodal experts can surpass individual or mixture-trained models when their complementary skills are optimally combined, further indicating that these benefits arise from explicit spectral regularization rather than a costly iterative optimizer.

\subsection{Vision Model Merging}
\label{sec:exp_general}
\label{sec:exp_clip}
\label{sec:exp_tall20}

\begin{table}[!t]
\centering
\caption{Cross-backbone average accuracy ($\%$) on CLIP-ViT TA8 (B/32, B/16, L/14) and the 20-task extension TALL20 (B/32). Detailed per-task results for all three TA8 backbones are provided in Appendix~C. The best and second-best results in each column are \textbf{bolded} and \underline{underlined}, respectively. The last two rows further apply AdaMerging~\cite{yang2024adamerging} to the closed-form merged delta using unlabeled test data.}
\label{tab:clip-summary}
\setlength{\tabcolsep}{6pt}
\renewcommand{\arraystretch}{1.05}
\begin{tabular}{l|ccc|c}
\toprule
\multirow{2}{*}{\textbf{Method}} & \multicolumn{3}{c|}{\textbf{TA8}} & \textbf{TALL20} \\
\cmidrule(lr){2-4}\cmidrule(lr){5-5}
 & \textbf{B/32} & \textbf{B/16} & \textbf{L/14} & \textbf{B/32} \\
\midrule
Weight Average    & 66.32 & 72.33 & 79.87 & 61.10 \\
Task Arithmetic   & 67.55 & 77.14 & 80.47 & 60.62 \\
TIES Merging      & 71.90 & 77.60 & 83.83 & 62.76 \\
TA w/ DARE        & 67.46 & 77.15 & 80.49 & 60.55 \\
TIES w/ DARE      & 60.96 & 74.30 & 74.33 & 62.22 \\
TSV Merging       & 83.07 & 87.10 & 90.57 & 73.22 \\
Iso-C             & 80.39 & 85.07 & 90.65 & 70.35 \\
$\tau^{\rm cf}=DC^\dagger$ & 82.33 & 88.04 & 91.69 & 72.54 \\
\rowcolor{gray!10}WUDI Merging      & 84.63 & 89.17 & 92.16 & 61.06 \\
\rowcolor{gray!10}OptMerge          & 84.53 & 89.49 & 92.38 & 61.71 \\
\midrule
\rowcolor{gray!10}\swudi{}    & 85.55 & 89.57 & 92.51 & 75.60 \\
\rowcolor{gray!10}\aswudi{}   & 85.53 & 89.49 & 92.52 & 75.62 \\
\midrule
\rowcolor{gray!10}\swudi{}\,$\to$\,AdaMerging  & \textbf{86.08} & \underline{89.78} & \underline{92.72} & \textbf{78.12} \\
\rowcolor{gray!10}\aswudi{}\,$\to$\,AdaMerging & \underline{86.05} & \textbf{89.81} & \textbf{92.75} & \underline{78.03} \\
\bottomrule
\end{tabular}
\end{table}

We next evaluate whether the spectral solvers generalize from multimodal models to vision models. Table~\ref{tab:clip-summary} summarizes the results for CLIP-ViT on the TA8 benchmark across three backbones, as well as the larger TALL20 setting. On TA8, \swudi{} and \aswudi{} achieve the best or second-best data-free averages across all backbones, with \aswudi{} matching the tuned \swudi{} while eliminating the need for method-specific rank tuning.

To understand this performance, we include the unregularized closed-form pseudoinverse $\tau^{\rm cf}=DC^\dagger$ as a direct ablation. Its performance gap to \aswudi{} shrinks monotonically as the backbone size increases ($-3.20$ pt on B/32, $-1.45$ pt on B/16, $-0.83$ pt on L/14). This aligns with the noise-amplification analysis in Sec.~\ref{sec:noise}: larger backbones yield better-conditioned task-vector spectra, meaning division by small $\lambda_k$ causes less degradation. However, on TALL20, this gap widens again to $-3.08$ pt, reflecting the re-emergence of long-tail noise as the number of tasks increases.

The TALL20 setting further demonstrates the robustness of our solvers in a more densely populated task space. Here, \aswudi{} outperforms TSV-Merging and significantly exceeds the iterative \wudi{} and \optmerge{}. Furthermore, the performance drop from TA8 to TALL20 is substantially smaller for the spectral solvers than for iterative methods, confirming that suppressing noise-amplifying tail directions becomes increasingly critical as more task vectors interact. Finally, when unlabeled test data are available, applying AdaMerging~\cite{yang2024adamerging} on top of our spectral anchors yields further improvements: $+0.2$ to $+0.5$ pt on TA8, and a substantial $+2.4$ to $+2.5$ pt on TALL20. This demonstrates that closed-form spectral merging provides a robust data-free initialization that remains highly complementary to test-time adaptation.

\subsection{Language Model Merging}
\label{sec:exp_flant5}
\label{sec:exp_llama}

We evaluate language-model merging in two contrasting regimes: LoRA fine-tuning on Flan-T5 GLUE~\cite{chung2024scaling} and full-parameter fine-tuning on Llama-3.2-3B MergeBench~\cite{he2026mergebench}. Table~\ref{tab:flant5} shows that the Flan-T5 LoRA setting tightens the gap between iterative and closed-form data-free objectives: \wudi{} and \optmerge{} reach $79.36\%$ and $81.03\%$, respectively, sitting close to TSV Merging ($81.83\%$) but still trailing the proposed solvers, with \aswudi{} achieving the best average ($82.98\%$) and \swudi{} the second-best ($82.70\%$). The benefit comes from matching the solver to the low-rank structure of LoRA deltas: adaptive truncation preserves the informative subspace and avoids the interference that the iterative quadratic loss leaves under rank-deficient $C$, where many small eigendirections couple weakly to the proxy gradient and slow convergence.

\begin{table*}[!t]
\centering
\begin{minipage}[t]{0.48\textwidth}
\centering
\caption{Multi-task performance when merging Flan-T5-base (LoRA fine-tuned) models on all eight tasks. The metric is accuracy except for STSB (Spearman $\rho$).}
\label{tab:flant5}
\setlength{\tabcolsep}{4pt}
\renewcommand{\arraystretch}{1.05}
\resizebox{\linewidth}{!}{%
\begin{tabular}{l|cccccccc|c}
\toprule
\textbf{Method} & \textbf{CoLA} & \textbf{MNLI} & \textbf{MRPC} & \textbf{QNLI} & \textbf{QQP} & \textbf{RTE} & \textbf{SST2} & \textbf{STSB} & \textbf{Avg.} \\
\midrule
Weight Average                   & \textbf{69.70} & 59.66 & 78.92 & \textbf{90.08} & \textbf{83.79} & \underline{80.51} & 91.17 & 72.00 & 78.23 \\
Task Arithmetic                  & 68.84 & 55.18 & 78.68 & 89.79 & 83.67 & 79.06 & 91.51 & 72.38 & 77.39 \\
TIES Merging                     & 68.17 & 48.96 & 78.92 & 89.31 & 83.43 & 79.78 & 91.51 & 74.22 & 76.79 \\
TA w/ DARE                       & 68.94 & 55.10 & 78.92 & 89.73 & \underline{83.71} & 79.06 & 91.51 & 72.58 & 77.44 \\
TIES w/ DARE                     & 31.16 & 0.43  & 79.90 & 84.44 & 82.23 & 76.90 & 89.56 & 75.94 & 65.07 \\
TSV Merging                      & \underline{69.32} & 77.09 & \underline{80.39} & \underline{90.04} & 83.62 & 79.06 & \underline{92.55} & 82.55 & 81.83 \\
Iso-C                            & 69.13 & 57.35 & 76.72 & 88.63 & 82.66 & 80.14 & 91.28 & 63.32 & 76.15 \\
\rowcolor{gray!10}WUDI Merging   & 68.65 & 72.18 & 78.43 & 84.64 & 82.70 & 71.48 & \textbf{93.00} & 83.82 & 79.36 \\
\rowcolor{gray!10}OptMerge       & 68.36 & 70.29 & \underline{80.39} & 89.58 & 83.28 & 79.06 & \textbf{93.00} & 84.27 & 81.03 \\
\midrule
\rowcolor{gray!10}\swudi{}     & 69.22 & \textbf{82.00} & 77.94 & 89.80 & 83.46 & \textbf{80.87} & \textbf{93.00} & \textbf{85.33} & \underline{82.70} \\
\rowcolor{gray!10}\aswudi{}    & 68.94 & \underline{80.00} & \textbf{83.33} & 89.77 & 83.26 & \textbf{80.87} & \underline{92.55} & \underline{85.10} & \textbf{82.98} \\
\bottomrule
\end{tabular}}
\end{minipage}\hfill
\begin{minipage}[t]{0.5\textwidth}
\centering
\caption{Llama-3.2-3B MergeBench with five experts. We report per-task and average accuracy ($\%$) on eight tasks. Multilingual subtasks follow the fr-only legacy protocol. Best results are shown in bold, and the second-best results are underlined.}
\label{tab:llama}
\setlength{\tabcolsep}{4pt}
\renewcommand{\arraystretch}{1.05}
\resizebox{\linewidth}{!}{%
\begin{tabular}{l|cccccccc|c}
\toprule
\textbf{Method} & \textbf{GSM8K} & \textbf{HE+} & \textbf{MBPP+} & \textbf{IFEval} & \textbf{TQA} & \textbf{MMLU$_\text{fr}$} & \textbf{ARC$_\text{fr}$} & \textbf{HSwag$_\text{fr}$} & \textbf{Avg.} \\
\midrule
Weight Average                  & 42.76 & 31.71 & \underline{59.52} & 9.24  & 46.02 & 46.35 & 35.76 & 44.53 & 39.49 \\
Task Arithmetic                 & 44.73 & 33.54 & \textbf{59.79} & 14.42 & 47.38 & 46.34 & 36.27 & 44.88 & 40.92 \\
TIES Merging                    & 42.61 & 30.49 & 57.14 & 7.58  & 44.91 & \textbf{48.25} & 37.13 & 44.75 & 39.11 \\
TA w/ DARE                      & 46.70 & 33.54 & 59.26 & 18.67 & \textbf{47.80} & 48.12 & 36.44 & \textbf{45.56} & 42.01 \\
TIES w/ DARE                    & 52.99 & 35.37 & 57.41 & \textbf{25.51} & \underline{47.42} & 47.55 & 37.04 & \underline{45.19} & \underline{43.56} \\
TSV Merging                     & 55.72 & 36.59 & 56.88 & 20.15 & 46.28 & \underline{48.17} & 37.13 & 45.01 & 43.24 \\
Iso-C                           & 48.22 & 35.37 & 55.56 & 8.13  & 44.83 & 47.49 & 37.04 & 44.29 & 40.12 \\
\rowcolor{gray!10}WUDI Merging  & 52.54 & \underline{37.20} & 57.14 & 17.93 & 46.56 & 44.86 & 37.21 & 44.17 & 42.20 \\
\rowcolor{gray!10}OptMerge      & 53.53 & 34.76 & 58.99 & \textbf{25.51} & 45.23 & 46.99 & \underline{37.30} & 44.66 & 43.37 \\
\midrule
\rowcolor{gray!10}\swudi{}       & \textbf{58.45} & \textbf{37.80} & 57.67 & \underline{24.58} & 45.02 & 47.01 & \textbf{37.55} & 44.69 & \textbf{44.10} \\
\rowcolor{gray!10}\aswudi{}    & \underline{58.15} & 36.59 & 56.35 & 22.18 & 45.42 & 46.85 & \underline{37.30} & 44.83 & 43.46 \\
\bottomrule
\end{tabular}}
\end{minipage}
\end{table*}

The Llama-3.2-3B benchmark complements this LoRA case with full-parameter experts spanning math, code, instruction following, safety, and multilingual tasks. As reported in Table~\ref{tab:llama}, \swudi{} obtains the best average among the merging methods, and \aswudi{} (with the Gavish--Donoho rank rule appropriate for the spiked-noise spectra of full-parameter LLM deltas) remains the strongest tuning-free alternative. The contrast with Flan-T5 illustrates why a single rank rule is not sufficient: LoRA deltas favor a heavy-tailed low-rank prior, whereas full-parameter LLM deltas are better described by a spiked-noise spectrum. In both regimes, the same benefit emerges: spectral filtering converts an unstable proxy inversion into a controlled merge that improves accuracy while avoiding the cost of iterative optimization.

Evaluating $5$ to $20$ experts across multimodal, vision, and language foundation models, we find that \swudi{} and \aswudi{} define the high-accuracy end of the closed-form Pareto frontier. Consequently, both solvers transform the implicit regularization, previously obtained through hundreds of optimizer steps, into an explicit closed-form computation. This yields the accuracy benefits of spectral filtering with substantially lower wall-clock time and memory costs.

%% file: sections/06_analysis.tex
\section{Efficiency and Diagnostic Analysis}
\label{sec:analysis}

This section distills the empirical analysis into two high-level messages. First, replacing iterative optimization with closed-form spectral filtering substantially reduces wall-clock time and memory. Second, spectral diagnostics explain why adaptive truncation is needed across architectures.

\subsection{Efficiency and Accuracy--Cost Trade-off}
\label{sec:efficiency}
\label{sec:pareto}

Table~\ref{tab:efficiency} reports the raw wall-clock time and peak GPU memory of the merging step on representative settings. The closed-form solvers consistently reduce both quantities because they eliminate optimizer state and per-iteration workspaces and replace hundreds of matrix multiplications with one symmetric eigendecomposition per layer. The magnitude of the memory saving depends on the regime: it is modest when resident model parameters dominate the footprint, as in full-parameter Llama merging, but substantial when optimizer states dominate, as in Qwen2-VL LoRA merging.

\begin{table}[!t]
\centering
\caption{Wall-clock time, peak GPU memory, and speedup for the merging step. Speedup is measured against the iterative data-free baseline in each setting (\wudi{} for CLIP/Llama and \optmerge{} for MLLMs). Mixture-training rows report the cost of jointly fine-tuning a single multi-task model.}
\label{tab:efficiency}
\setlength{\tabcolsep}{3pt}
\renewcommand{\arraystretch}{1.08}
\resizebox{\columnwidth}{!}{%
\begin{tabular}{ll|rrc}
\toprule
\textbf{Setting} & \textbf{Method} & \textbf{Time} & \textbf{Peak Mem.} & \textbf{Speedup} \\
\midrule
\multirow{2}{*}{CLIP-B/32}
 & WUDI Merging & $86.3$\,s & $4.82$\,GB & $1.0\times$ \\
 & \swudi{} / \aswudi{} & $2.8$\,s & ${3.64}$\,GB & $30.8\times$ \\
\midrule
\multirow{2}{*}{Llama-3.2-3B}
 & WUDI Merging & $5126.1$\,s & $42.03$\,GB & $1.0\times$ \\
 & \swudi{} / \aswudi{} & $70.7$\,s & ${39.78}$\,GB & $72.5\times$ \\
\midrule
\multirow{3}{*}{Qwen2-VL-7B}
 & OptMerge & $13608.9$\,s & $21.97$\,GB & $1.0\times$ \\
 & \swudi{} / \aswudi{} & $487.6$\,s & $10.81$\,GB & $27.9\times$ \\
 & Mixture training & $24.56$\,h & $256$\,GB & --- \\
\midrule
\multirow{3}{*}{InternVL2.5-1B}
 & OptMerge & $552.3$\,s & $7.03$\,GB & $1.0\times$ \\
 & \swudi{} / \aswudi{} & $8.0$\,s & $4.16$\,GB & $69.0\times$ \\
 & Mixture training & $25.38$\,h & $240$\,GB & --- \\
\bottomrule
\end{tabular}}\\[2pt]
\end{table}

These efficiency gains are most compelling when viewed alongside accuracy. Combining Table~\ref{tab:efficiency} with the benchmark results demonstrates that the proposed solvers shift the merging process toward the upper left of the accuracy-cost plane. They achieve comparable or superior accuracy under a substantially reduced computational budget. Fig.~\ref{fig:pareto} illustrates this trade-off across four representative settings. Here, \aswudi{} acts as the ideal low-cost, tuning-free solution, while \swudi{} provides a high-accuracy alternative if hyperparameter tuning is permitted. Both approaches successfully replace the iterative optimizer with explicit spectral regularization.

The mixture-training reference rows provide further context for these computational savings. Joint multi-task fine-tuning of Qwen2-VL-7B and InternVL2.5-1B requires approximately $24$ to $25$ hours and $240$ to $256$\,GB of aggregate GPU memory. This represents roughly $180\times$ the wall-clock time and $20\times$ the peak memory required by \aswudi{} on the same backbones, even under the strict assumption that all task-specific training data are centrally co-located. Therefore, closed-form spectral merging not only reduces the cost of iterative \optmerge{} by an order of magnitude but also offers a data-free alternative to the natural baseline (joint training) at a mere fraction of its computational and data-governance costs.

\subsection{Spectral Diagnostics}
\label{sec:diag}

The spectral diagnostics in Fig.~\ref{fig:spectral-rank} explain why explicit spectral regularization is needed. Small-eigenvalue directions are most vulnerable to pseudoinverse instability: dividing by a small $\lambda_k$ amplifies proxy noise, so the spectral tail should not be inverted without regularization. This supports the hard truncation used by \swudi{} and \aswudi{}.

The same diagnostics also show why the cutoff should be adaptive. Spectra vary substantially across layers and architectures: vision and LoRA merges often exhibit a head-and-tail structure, where a participation-style rule~\cite{roy2007effective} preserves the useful subspace, whereas full-parameter LLM merges more often resemble a bulk-plus-spike regime, favoring a more conservative Gavish--Donoho cutoff~\cite{marchenko1967distribution,gavish2014optimal}. The retained ranks reflect the underlying fine-tuning geometry: for example, \aswudi{}-\texttt{psqrt} keeps a mean rank ratio of $0.149$ on Qwen2-VL-7B LoRA, versus $0.615$ on fully fine-tuned InternVL2.5-1B. These diagnostics are explanatory rather than tuning criteria; they show why layer-wise rank adaptation is preferable to a single global cutoff. Full per-architecture statistics are reported in Table~XVII (Appendix~C).

\begin{figure}[!t]
\centering
\includegraphics[width=0.8\columnwidth]{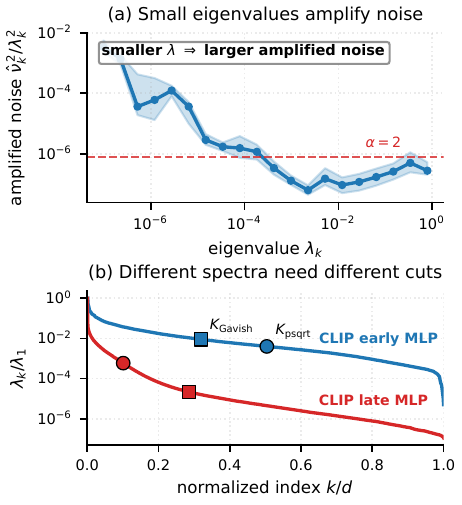}
\caption{\textbf{Noise amplification motivates adaptive rank truncation.} \textbf{(a)} Small-eigenvalue directions are associated with larger amplified noise $\hat\nu_k^{\,2}/\lambda_k^2$ under the closed-form pseudoinverse (binned median + $25$--$75\%$ band, pooled over CLIP-ViT-B/32 TA8 layers). \textbf{(b)} Different layer spectra lead to different \aswudi{} rank-rule cuts: the $\rm psqrt$ rule retains more directions in heavy-tailed spectra, whereas the Gavish--Donoho rule is more conservative for concentrated spectra. This motivates layer-wise adaptive rank selection.}
\label{fig:spectral-rank}
\end{figure}

Together, these results show that \wudi{}/\optmerge{} works primarily through implicit spectral regularization. Our closed-form solvers make this regularization explicit, preserving accuracy across diverse settings while avoiding hundreds of optimizer steps.

%% file: sections/07_conclusion.tex
\section{Conclusion}
\label{sec:conclusion}

We revisit data-free model merging as a noisy linear inverse problem. While \wudi{} and \optmerge{} optimize a quadratic objective over hundreds of steps, this objective has a closed-form pseudoinverse whose small-eigenvalue directions amplify proxy noise. Iterative descent succeeds largely because it implicitly filters these unstable directions. This insight motivates a spectral-filtering estimator and closed-form solvers: \swudi{}, which combines an exponential filter with hard rank truncation, and \aswudi{}, which uses layer-wise spectral rules to eliminate rank hyperparameters. Both require only one eigendecomposition per layer, with no training data or optimizer state. We also introduce a multimodal benchmark for capability and modality merging. Across vision, language, LoRA, full-parameter LLM, and MLLM settings, our solvers match or exceed iterative baselines while running $28\times$ to $72\times$ faster and using up to $50\%$ less peak memory. These findings suggest that effective merging need not rely on long optimization trajectories: once the inverse problem is formulated, the central design choice is the spectral filter to impose.

%% file: appendices/A_notation.tex
\section{Notations}
\label{app:notation}

This appendix lists the symbols that are used repeatedly across the theory, method, and diagnostic sections. Unless explicitly written, the layer index $\ell$ is suppressed for per-layer matrices; indices $i$ and $k$ denote experts and eigendirections, respectively. Hats indicate constructed estimators, and the superscript ${\rm cf}$ denotes a closed-form solution.

\subsection{Models and Task Vectors}
This group fixes the model-level and per-layer objects used throughout the paper. Table~\ref{tab:notation-models} distinguishes full-model parameters from 2-D layer blocks and records the main task-vector quantities used by the closed-form estimator.

\begin{table}[!t]
\centering
\caption{Notation: models and task vectors.}
\label{tab:notation-models}
\setlength{\tabcolsep}{6pt}
\renewcommand{\arraystretch}{1.15}
\begin{tabularx}{\textwidth}{lX}
\toprule
\textbf{Symbol} & \textbf{Meaning} \\
\midrule
$\Theta_0,\Theta_i,\Theta_m$ & base, fine-tuned expert, and merged model parameters \\
$W_0^{(\ell)},W_i^{(\ell)}\in\mathbb{R}^{d_o\times d_i}$ & $\ell$-th 2-D weight block of the base/expert model \\
$\tau_i$ or $\tau_i^{(\ell)}$ & per-layer task vector $W_i^{(\ell)}-W_0^{(\ell)}$ \\
$\bm{\tau}_m$ and $\tau$ & full-model merged delta and its per-layer variable, $\tau:=\bm{\tau}_m^{(\ell)}$ \\
$\tau_{\rm init}$ & optional initialization in the spectral filtering estimator; default $\sum_i\tau_i$, or $0$ for the direct filtered inverse \\
$\tau^{\rm cf}=DC^\dagger$ & minimum-norm closed-form solution of the WUDI normal equation \\
$\tau^\circ$ & unobserved ideal merged delta used only in the inverse-problem analysis \\
$\widehat{\tau}_h$ & spectral filtering estimator in Eq.~\eqref{eq:unified} \\
\bottomrule
\end{tabularx}
\end{table}

\subsection{Spectral and Filter Quantities}
The spectral notation in Table~\ref{tab:notation-spectral} is used to express the WUDI objective as a normal equation. The matrix $C$ defines the eigendirections to be inverted, $D$ is the right-hand side, and $h_k$ controls how strongly each direction is retained.

\begin{table}[!t]
\centering
\caption{Notation: spectral and filter quantities.}
\label{tab:notation-spectral}
\setlength{\tabcolsep}{6pt}
\renewcommand{\arraystretch}{1.15}
\begin{tabularx}{\textwidth}{lX}
\toprule
\textbf{Symbol} & \textbf{Meaning} \\
\midrule
$A_i=\tau_i^\top\tau_i/\|\tau_i\|_F^2$ & single-task input-side covariance proxy \\
$C=\sum_i A_i$ & aggregated input-side covariance proxy \\
$D=\sum_i\tau_i A_i$ & right-hand side of the normal equation $\tau C=D$ \\
$M=[\tau_1/\|\tau_1\|_F;\ldots;\tau_N/\|\tau_N\|_F]$ & stacked normalized task vectors, satisfying $M^\top M=C$ \\
$C=Q\Lambda Q^\top$ & eigendecomposition of $C$ \\
$\lambda_k,q_k$ & $k$-th eigenvalue and eigenvector of $C$ \\
$\sigma_k=\sqrt{\lambda_k}$ & $k$-th singular value of $M$ \\
$C^\dagger$ & Moore--Penrose pseudoinverse of $C$ \\
$h_k$ & spectral filter coefficient applied to direction $q_k$ \\
$s_t(\lambda_k)=1-e^{-t\lambda_k}$ & soft exponential filter (gradient-flow stopping time $t$) \\
$m_k=\mathbf{1}[k\le K]$ & hard top-$K$ truncation mask \\
$h_k=m_k\,s_t(\lambda_k)$ & \swudi{} two-factor spectral filter, see Eq.~\eqref{eq:swudi} \\
$g_t(\lambda_k)=s_t(\lambda_k)/\lambda_k$ & effective inverse gain on direction $q_k$ \\
$C_h^\dagger=Q\,\mathrm{diag}(h_k/\lambda_k)\,Q^\top$ & filtered pseudoinverse used in Eq.~\eqref{eq:unified} \\
\bottomrule
\end{tabularx}
\end{table}

\subsection{Rank Rules and Hyperparameters}
Table~\ref{tab:notation-hp} collects the quantities that control truncation and final scaling. We keep only the symbols needed to describe \swudi{} and \aswudi{}. Other quantities are defined where they are used.

\begin{table}[!t]
\centering
\caption{Notation: rank rules and hyperparameters.}
\label{tab:notation-hp}
\setlength{\tabcolsep}{6pt}
\renewcommand{\arraystretch}{1.15}
\begin{tabularx}{\textwidth}{lX}
\toprule
\textbf{Symbol} & \textbf{Meaning} \\
\midrule
$K$ & retained-rank cutoff \\
$K_\ell$ & per-layer retained rank in \aswudi{} \\
$r\in(0,1]$ & \swudi{} global rank ratio, $K=\lceil r d_i\rceil$ \\
$t$ & \swudi{} exponential time parameter \\
$s$ & global scaling coefficient applied to the merged delta \\
$K_\ell^{\rm psqrt}$ & participation-square-root rank rule, $\left\lceil(\sum_k\sqrt{\lambda_k})^2/\sum_k\lambda_k\right\rceil$ \\
$K_\ell^{\rm Gavish}$ & Gavish--Donoho hard-threshold rank rule for spiked-noise spectra \\
\bottomrule
\end{tabularx}
\end{table}

\subsection{Diagnostics and Proxies}
The final notation block, summarized in Table~\ref{tab:notation-diag}, supports the diagnostic figures and the noise-amplification analysis. It separates the computable WUDI proxy from the unobserved signal/noise quantities used to explain why small-eigenvalue directions should be regularized.

\begin{table}[!t]
\centering
\caption{Notation: diagnostics, proxies, and per-direction noise model.}
\label{tab:notation-diag}
\setlength{\tabcolsep}{6pt}
\renewcommand{\arraystretch}{1.15}
\begin{tabularx}{\textwidth}{lX}
\toprule
\textbf{Symbol} & \textbf{Meaning} \\
\midrule
$\mathcal{P}(\tau)$ & WUDI proxy, $\sum_i\|(\tau-\tau_i)\tau_i^\top\|_F^2/\|\tau_i\|_F^2$ \\
$\hat I(\tau)$ & calibration-based estimate of real layer-wise interference \\
$E=D-\tau^\circ C$ & proxy-noise matrix in the inverse-problem view \\
$\xi_k=Eq_k$ & proxy noise projected onto eigendirection $q_k$ \\
\bottomrule
\end{tabularx}
\end{table}

%% file: appendices/B_derivations.tex
\section{Theoretical Proofs}
\label{app:proofs}

This appendix expands the derivations underlying Sec.~\ref{sec:theory} and Sec.~\ref{sec:methods}. We first derive the WUDI normal equation and justify the task-vector proxy for the input subspace; we then analyze the resulting inverse problem, spectral filters, adaptive rank rules, and parameter-drift bound. Throughout, we work per linear layer with task vectors $\tau_i\in\mathbb{R}^{d_o\times d_i}$, and we use the matrix calculus identity $\partial_X\,\mathrm{tr}(XAX^\top B)=BXA+B^\top X A^\top$.

\subsection{From the \wudi{} Loss to the Normal Equation}
\label{app:normal}
Recall that
\[
\mathcal{L}(\tau)=\sum_{i=1}^N\frac{1}{\|\tau_i\|_F^2}\bigl\|(\tau-\tau_i)\tau_i^\top\bigr\|_F^2.
\]
Expanding term $i$:
\[
\bigl\|(\tau-\tau_i)\tau_i^\top\bigr\|_F^2 = \mathrm{tr}\bigl((\tau-\tau_i)\,\tau_i^\top \tau_i\,(\tau-\tau_i)^\top\bigr).
\]
Let $A_i:=\tau_i^\top \tau_i/\|\tau_i\|_F^2$. Then
\[
\mathcal{L}(\tau)=\sum_i \mathrm{tr}\bigl((\tau-\tau_i)A_i(\tau-\tau_i)^\top\bigr).
\]
Expanding the quadratic and summing,
\[
\mathcal{L}(\tau)=\mathrm{tr}\bigl(\tau\,C\,\tau^\top\bigr)-2\,\mathrm{tr}\bigl(\tau\,D^\top\bigr)+\mathrm{const},
\]
where $C=\sum_i A_i$, $D=\sum_i \tau_i A_i$. Differentiating, $\nabla_\tau\mathcal{L}=2(\tau C-D)$. Stationary points satisfy $\tau C=D$. We first verify solvability: for any $z\in\mathrm{Null}(C)$, $z^\top C z=\sum_i\|\tau_i z\|_2^2/\|\tau_i\|_F^2=0$ forces $\tau_i z=0$ for all $i$, hence $A_i z=0$ and $D z=0$. Therefore $\mathrm{Null}(C)\subseteq\mathrm{Null}(D)$, equivalently each row of $D$ lies in $\mathrm{Range}(C)$. The full set of stationary points is parameterized by $\tau=DC^\dagger+Z(I-CC^\dagger)$ for an arbitrary free matrix $Z\in\mathbb{R}^{d_o\times d_i}$, which spans the null-space component. The minimum-Frobenius-norm choice is $Z=0$, yielding the closed form $\tau^{\rm cf}=D\,C^\dagger$, computed via the eigendecomposition $C=Q\Lambda Q^\top$.

\subsection{Task-Vector Proxy for the Input Subspace}
\label{app:tau-as-x}
\paragraph{Row-space justification}
The \wudi{} loss in Eq.~\eqref{eq:wudi} substitutes the transpose of the task vector $\bm{\tau}_i$ for the input subspace $\bm{x}_i$. We provide a self-contained justification. For a single linear layer with weight matrix $W_l$, the per-sample loss gradient has the standard outer-product form
\[
\nabla_{W_l}\mathcal{L}_{t,n}\;=\;g_{t,n}\,x_{t,n}^{\top},
\]
where $x_{t,n}$ is the layer's input activation at training step $t$ on sample $n$ and $g_{t,n}$ is the back-propagated output-side gradient. Summing the resulting GD updates across iterations gives
\[
\tau_{i,l}\;=\;-\eta\sum_{t=1}^{T}\sum_{n=1}^{B_{\rm s}}\,g_{t,n}\,x_{t,n}^{\top},
\]
where $B_{\rm s}$ is the per-step batch size.
Each row of $\tau_{i,l}$ is therefore a $g$-weighted superposition of input vectors visited during fine-tuning. Equivalently, the row space of $\tau_{i,l}$ is contained in (or, in finite-trajectory practice, biased toward) the activation subspace $\mathrm{span}\{x_{t,n}\}$. This containment is the formal content of the WUDI substitution: the proxy operator $\tau_i^\top\tau_i$ acts on the same subspace as the unobserved input Gram $\sum_{t,n}x_{t,n}x_{t,n}^\top$, up to the gradient-induced reweighting that we collect into the residual term $E$ of Sec.~\ref{sec:noise}. In particular, $\tau C=D$ in Sec.~\ref{sec:cf} can be read as a noisy linear inverse problem with respect to that activation subspace, which is the foundation on which the spectral filters in Sec.~\ref{sec:methods} are built.
\paragraph{Covariance-dominance bound}
\label{app:cov-ineq}
The discussion above establishes a row-space containment between $\tau_{i,l}$ and the input activation subspace. We now upgrade it to a quantitative inequality that justifies treating the WUDI proxy as a computable upper bound on the real per-layer interference
\[
\mathcal{I}_i(\tau)
\;:=\;\mathbb{E}_{x\sim\mathcal{D}_{i,l}}\!\bigl[\,\bigl\|(\tau-\tau_{i,l})\,x\bigr\|_2^2\,\bigr]
\;=\;\mathrm{tr}\!\bigl((\tau-\tau_{i,l})\,\Sigma_{i,l}\,(\tau-\tau_{i,l})^\top\bigr),
\]
where $\Sigma_{i,l}:=\mathbb{E}[x_{i,l}^{}x_{i,l}^\top]$ is the input second moment. Equivalently, given an empirical activation matrix $X_{i,l}\in\mathbb{R}^{d_i\times n_i}$ with $n_i$ samples, $\mathcal{I}_i(\tau)=\tfrac{1}{n_i}\|(\tau-\tau_{i,l})X_{i,l}\|_F^2$ and $\Sigma_{i,l}=\tfrac{1}{n_i}X_{i,l}X_{i,l}^\top$.

\begin{assumption}[Activation covariance dominated by task-vector Gram]\label{ass:cov}
For each task $i$ and each linear layer $l$, the empirical input second moment $\Sigma_{i,l}:=\mathbb E[x_{i,l}^{}x_{i,l}^\top]$ admits the decomposition
\begin{equation}
\Sigma_{i,l}\;\preceq\;a_{i,l}\,\frac{\tau_{i,l}^\top\tau_{i,l}}{\|\tau_{i,l}\|_F^2}\;+\;R_{i,l},
\label{eq:cov-decomp}
\end{equation}
with constants $a_{i,l}>0$ and a residual operator $R_{i,l}\succeq 0$ whose action on any layer-wise delta $\delta_{i,l}:=\tau-\tau_{i,l}$ is bounded by
\begin{equation}
\mathrm{tr}\bigl(\delta_{i,l}\,R_{i,l}\,\delta_{i,l}^\top\bigr)\;\le\;b_{i,l}\,\|\delta_{i,l}\|_F^2,
\label{eq:residual-bound}
\end{equation}
for some $b_{i,l}\ge 0$.
\end{assumption}

The decomposition in Eqs.~\eqref{eq:cov-decomp}--\eqref{eq:residual-bound} has two complementary readings. (i) The first term states that the directions on which $\Sigma_{i,l}$ has appreciable mass are precisely the directions on which $\tau_{i,l}^\top\tau_{i,l}$ has appreciable mass, scaled by the layer-specific constant $a_{i,l}$. This is the formal version of ``task vectors approximate the input subspace they were trained on.'' (ii) The residual $R_{i,l}$ collects every input direction that the gradient trajectory failed to cover (e.g., directions visited only at very early or very late iterations); $b_{i,l}$ controls how much $R_{i,l}$ can leak into the interference computation.

\begin{proposition}[Computable upper bound on real interference]
\label{prop:cov-bound}
Under Assumption~\ref{ass:cov}, the real per-layer interference satisfies
\[
\mathcal{I}_i(\tau)\;\le\;a_{i,l}\,\bigl\|(\tau-\tau_{i,l})\,\tau_{i,l}^\top\bigr\|_F^2\big/\|\tau_{i,l}\|_F^2\;+\;b_{i,l}\,\|\tau-\tau_{i,l}\|_F^2.
\]
\end{proposition}

\begin{proof}
Write $\delta:=\tau-\tau_{i,l}$. By definition, $\mathcal{I}_i(\tau)=\mathrm{tr}(\delta\,\Sigma_{i,l}\,\delta^\top)$. Substituting Eq.~\eqref{eq:cov-decomp} and using Eq.~\eqref{eq:residual-bound},
$\mathcal{I}_i(\tau)\le a_{i,l}\,\mathrm{tr}(\delta\,\tau_{i,l}^\top\tau_{i,l}\,\delta^\top)/\|\tau_{i,l}\|_F^2+b_{i,l}\,\|\delta\|_F^2$. The first term equals $a_{i,l}\,\|\delta\,\tau_{i,l}^\top\|_F^2/\|\tau_{i,l}\|_F^2$ since $\|\delta\,\tau_{i,l}^\top\|_F^2=\mathrm{tr}(\delta\,\tau_{i,l}^\top\tau_{i,l}\,\delta^\top)$.
\end{proof}

The first summand is exactly the WUDI proxy contribution from task $i$ (Eq.~\eqref{eq:wudi}); the second summand is a Frobenius-norm regularization on the merged delta. Minimizing the WUDI proxy therefore controls the real interference up to a $\|\delta\|_F$ slack term. Two consequences follow. First, the spectral filters in Sec.~\ref{sec:methods} that suppress small-$\lambda_k$ directions of $C$ are precisely those that make the proxy a tight bound: discarding low-eigenvalue directions reduces the proxy without inflating $\|\delta\|_F$. Second, the Frobenius-norm-inflation regime documented in Fig.~\ref{fig:norm-shortcut} is exactly the failure mode in which iterative \wudi{} drives the proxy down by inflating $\|\delta\|_F$, leaving the second summand large; closed-form spectral solvers avoid this regime by construction.

The empirical validity of Assumption~\ref{ass:cov} is supported by the capture-gap diagnostics in Fig.~\ref{fig:dropped-panels-1}(a): the task-vector subspaces capture input energy in early and middle layers (gap $+0.18$--$0.43$ vs.\ random subspaces). The last MLP layer is a documented exception (gap $\approx 0$); for that layer $b_{i,l}$ is comparable to $a_{i,l}$ and the proxy is correspondingly looser, in line with the observation that the input-subspace assumption is layer-conditional rather than global.

\subsection{\texorpdfstring{Noise Amplification and Spectral Risk}{Noise Amplification and Spectral Risk}}
\label{app:noise}
\paragraph{Inverse-problem view}
Decompose $D=\tau^\circ C+E$, where $\tau^\circ$ is an unobserved ideal merged delta (the model that would minimize the true downstream loss) and $E$ collects the proxy mismatch (replacing $x_i$ by $\tau_i^\top$, plus the linear-subspace approximation error). Project on $q_k$, the $k$-th eigenvector of $C$:
\[
y_k:=Dq_k=\lambda_k\tau_k^\circ+\xi_k,\quad \xi_k:=Eq_k.
\]
The closed-form solution gives, for $\lambda_k>0$, $\tau^{\rm cf}_k=y_k/\lambda_k=\tau_k^\circ+\xi_k/\lambda_k$. As $\lambda_k\to 0^+$, the noise term $\xi_k/\lambda_k$ dominates. Equivalently, defining the residual signal $R_k^\circ:=(\tau^\circ-\tau_{\rm init})q_k$ and the residual right-hand side $B:=D-\tau_{\rm init}\,C$, so that $B q_k=\lambda_k R_k^\circ+\xi_k$, any spectral filter $h_k$ with $h_k\to 0$ as $\lambda_k\to 0$ produces a regularized residual estimator $\widehat R_k=h_k(B q_k/\lambda_k)$ and merged update $\widehat\tau q_k=\tau_{\rm init} q_k+\widehat R_k$, in which the noise contribution scales as $h_k/\lambda_k$, controllable by the filter shape. This shrinks the residual rather than the absolute estimate, so discarded directions retain the initial point $\tau_{\rm init}$, matching the implementation of \aswudi{}.

\paragraph{Empirical noise-amplification fit}
\label{app:noise-fit}
The inverse-problem view above treats the residual $\xi_k=Eq_k$ as an arbitrary noise vector. Empirically, its squared norm is well described by a power law in $\lambda_k$:
\begin{equation}
\hat\nu_k^{\,2}\;\approx\;\sigma_0^2 + \sigma_1^2\,\lambda_k^{\alpha},
\label{eq:noise-fit}
\end{equation}
with three parameters $(\sigma_0^2,\sigma_1^2,\alpha)\ge 0$ fitted per layer in log--log space. Across the $72$ linear layers of CLIP-ViT-B/32 TA8, we obtain $\alpha$ with mean $1.871$ and median $1.855$; $68\%$ of layers fit exponents below $2$, and only a small minority fit exponents above. The empirical $\alpha\approx 2$ regime means $\hat\nu_k^{\,2}/\lambda_k^2$ behaves as $\sigma_0^2/\lambda_k^2+\sigma_1^2$ in the small-$\lambda_k$ tail, so the closed-form pseudoinverse risk diverges only through the offset $\sigma_0^2$, while the bulk of the spectrum (where $\lambda_k$ is large) sees a vanishing noise contribution. This is consistent with the binned-median plot in Fig.~\ref{fig:spectral-rank}(a) and motivates suppressing small-$\lambda_k$ directions with a spectral filter.

\subsection{Spectral Filters and the Unified Estimator}
\label{app:flow}
\paragraph{Gradient flow and Landweber filters}
The Frobenius gradient flow is
\[
\dot\tau(t)=D-\tau(t)C,\quad \tau(0)=\tau_{\rm init}.
\]
Multiplying both sides on the right by $Q$ and writing $\tilde\tau(t)=\tau(t)Q$, $\tilde D=DQ$, we obtain $d_i$ decoupled vector ODEs $\dot{\tilde\tau}_k(t)=\tilde D_k - \lambda_k\tilde\tau_k(t)$ (one per eigendirection, each $\tilde\tau_k,\tilde D_k\in\mathbb{R}^{d_o}$), each solved by
\[
\tilde\tau_k(t)=\tilde\tau_{{\rm init},k}+\frac{\tilde D_k-\lambda_k\tilde\tau_{{\rm init},k}}{\lambda_k}\bigl(1-e^{-\lambda_k t}\bigr).
\]
Reassembling and identifying the spectral filter,
\[
\tau(t)=\tau_{\rm init}+\bigl(D C^\dagger-\tau_{\rm init}\,C C^\dagger\bigr)Q\,\mathrm{diag}\bigl(1-e^{-\lambda_k t}\bigr)Q^\top.
\]
The discrete Landweber iteration $\tau_{n+1}=\tau_n+\eta(D-\tau_n C)$ has filter
$h_k^{\rm LW}(n)=1-(1-\eta\lambda_k)^n$, stable for $0<\eta<2/\lambda_{\max}$. As $n\to\infty$, both filters converge to $1$ on $\lambda_k>0$, recovering $\tau^{\rm cf}$. We refer to~\cite{engl1996regularization} for the classical theory of spectral filters as regularizers for ill-posed linear inverse problems.

\paragraph{Derivation of the unified spectral filtering estimator}
\label{app:unified}
The estimator $\widehat{\tau}_h$ in Eq.~\eqref{eq:unified} follows directly from the linear ODE solution. Let $R(t):=\tau(t)-\tau_{\rm init}$ be the deviation from the initial point and recall $B=D-\tau_{\rm init}\,C$ from Appendix~\ref{app:noise}. The corresponding ODE $\dot R(t)=B-R(t)\,C$ with $R(0)=0$ has solution $R(t)=B\,Q\,\mathrm{diag}\bigl((1-e^{-\lambda_k t})/\lambda_k\bigr)_{\lambda_k>0}\,Q^\top$. Replacing $1-e^{-\lambda_k t}$ by an arbitrary filter $h_k$ in the eigenbasis yields $R=B\,C_h^\dagger$, which is exactly Eq.~\eqref{eq:unified} after adding $\tau_{\rm init}$. Setting $\tau_{\rm init}=0$ gives the direct filtered inverse $\widehat{\tau}_h=D\,C_h^\dagger$. Fig.~\ref{fig:S1} compares the resulting filters---the full pseudoinverse, the Wiener filter, the tuned \swudi{} filter, and the adaptive \aswudi{} cutoff---together with their predicted residual risk and the boundary signal-to-noise ratio at the cutoff.

\begin{figure*}[!t]
\centering
\includegraphics[width=\textwidth]{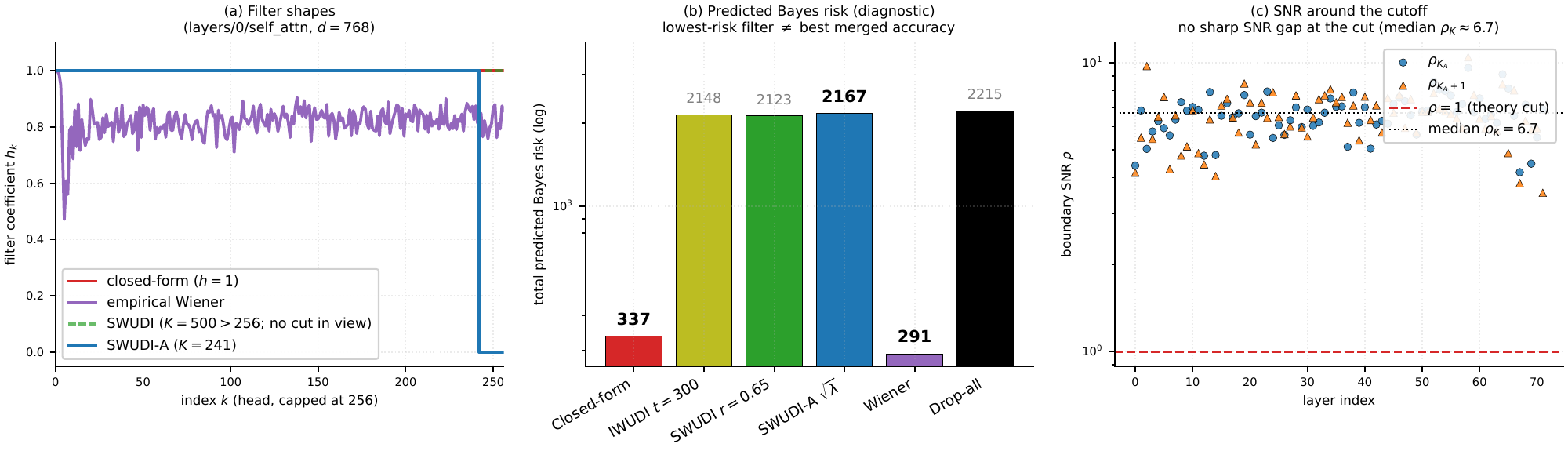}
\caption{\textbf{Spectral-filter diagnostics on CLIP-ViT-B/32 TA8.} This figure is diagnostic rather than prescriptive: it shows that simple Wiener/Bayes-risk or SNR-gap criteria do not by themselves pick the best merging filter, which is why \aswudi{} instead relies on the conservative rank rules of Appendix~\ref{app:adaptive-rank}. Panel (a) plots the filter value $h_k$ (vertical axis) against the eigen-direction index $k$ sorted by decreasing eigenvalue (horizontal axis), comparing the full pseudoinverse ($h_k\!=\!1$), the empirical Wiener filter, the tuned \swudi{} hard cutoff, and the adaptive \aswudi{} cutoff; the tuned \swudi{} cut ($K\!=\!500$) lies beyond the displayed head range, so within view it coincides with the pseudoinverse. Panel (b) compares the total predicted residual (Bayes) risk of these filters, where lower bars indicate smaller risk under the residual-noise model introduced in Sec.~\ref{sec:noise} and detailed in Appendix~\ref{app:noise-fit}; the Wiener and pseudoinverse filters attain the lowest predicted risk, yet they are \emph{not} the best on real merged accuracy, so this risk model is a negative diagnostic and not a selection criterion. Panel (c) plots the boundary signal-to-noise ratio $\rho_k$ at the eigen-direction $k$ (the ratio of retained signal energy to residual-noise energy at that boundary) evaluated at and just after the \aswudi{} cutoff $K_A$ across layers, where $K_A$ is the per-layer \aswudi{} retained rank. The median boundary SNR stays well above $1$ with no sharp drop across the cut, confirming that \aswudi{} is a conservative spectral-rank rule rather than a sharp SNR threshold: it removes the spectral tail while retaining the directions that dominate the proxy reduction.}
\label{fig:S1}
\end{figure*}

\subsection{Adaptive Rank Rules}
\label{app:adaptive-rank}
The goal of this section is to explain how \aswudi{} selects a retained rank $K_\ell$ for each layer without using a global rank ratio. We use two complementary rules: a participation-ratio rule for heavy-tailed spectra, and a Gavish--Donoho threshold for spectra with a clearer signal-plus-noise structure.

\subsubsection{Participation-Square-Root Rule}
\label{app:psqrt}
For singular values $\sigma_k=\sqrt{\lambda_k}$, the participation ratio
\[
R_{\rm part}:=\frac{(\sum_k\sigma_k)^2}{\sum_k\sigma_k^2}
\]
counts the effective number of comparable singular components. We set
\[
K_\ell^{\rm psqrt}=\left\lceil R_{\rm part}\right\rceil
=\left\lceil\frac{(\sum_k\sqrt{\lambda_k})^2}{\sum_k\lambda_k}\right\rceil .
\]
If the active singular spectrum is flat and supported on exactly $K$ directions, then $R_{\rm part}=K$, so the rule recovers the active rank exactly. More generally, suppose the active singular values are $\sigma_k=\mu(1+\delta_k)$ for $k\le K$, with empirical mean perturbation close to zero and empirical second moment $c^2:=K^{-1}\sum_{k\le K}\delta_k^2$. A first-order expansion gives
\[
\sum_{k\le K}\sigma_k\approx K\mu,
\qquad
\sum_{k\le K}\sigma_k^2\approx K\mu^2(1+c^2),
\]
and therefore
\[
R_{\rm part}^{(\sigma)}\approx \frac{K}{1+c^2}.
\]
Thus the rule contracts the ideal active rank only according to the relative spread of the active singular values, not their absolute scale. By contrast, computing the same participation ratio on eigenvalues $\lambda_k=\sigma_k^2$ gives $R_{\rm part}^{(\lambda)}\approx K(1-4c^2)$ for small $c$, which is more sensitive to spectral spread and tends to retain too few directions in practice.

\subsubsection{Marchenko--Pastur Gavish--Donoho Rule}
\label{app:gd}
Stack the normalized task-vector matrices vertically as $M\in\mathbb{R}^{Nd_o\times d_i}$. Since $M^\top M=C$ (Sec.~\ref{sec:aswudi}), the singular values of $M$ are $\sigma_k=\sqrt{\lambda_k(C)}$. When the spectrum resembles a low-rank signal plus random noise, the Marchenko--Pastur law and the Gavish--Donoho threshold give a conservative hard cutoff: retain singular values above
\[
\omega_{\rm GD}(\beta)\,\widehat\sigma_{\rm med},\qquad
\beta=\frac{\min(Nd_o,d_i)}{\max(Nd_o,d_i)},
\]
where $\widehat\sigma_{\rm med}$ is the empirical median singular value. In the unknown-noise setting, $\omega_{\rm GD}(\beta)=\lambda_*(\beta)/\sqrt{\mu_\beta}$, where $\mu_\beta$ is the median of the Marchenko--Pastur distribution and
\begin{equation}
\lambda_*(\beta)=\sqrt{\,2(\beta+1)+\dfrac{8\beta}{(\beta+1)+\sqrt{\beta^2+14\beta+1}}\,}
\label{eq:gd-known-noise}
\end{equation}
is the known-noise threshold of \cite[Eq.~6]{gavish2014optimal}. This gives the per-layer rank in Eq.~\eqref{eq:gd}. We use this rule as a conservative alternative when the spectrum has a visible noise bulk rather than a long heavy tail.

Fig.~\ref{fig:S2} supports the need for adaptive rank selection: different layer types prefer different retained ranks, while SGD/Adam trajectory diagnostics connect the rank-rule behavior back to the spectral-filtering view.

\begin{figure*}[!t]
\centering
\includegraphics[width=\textwidth]{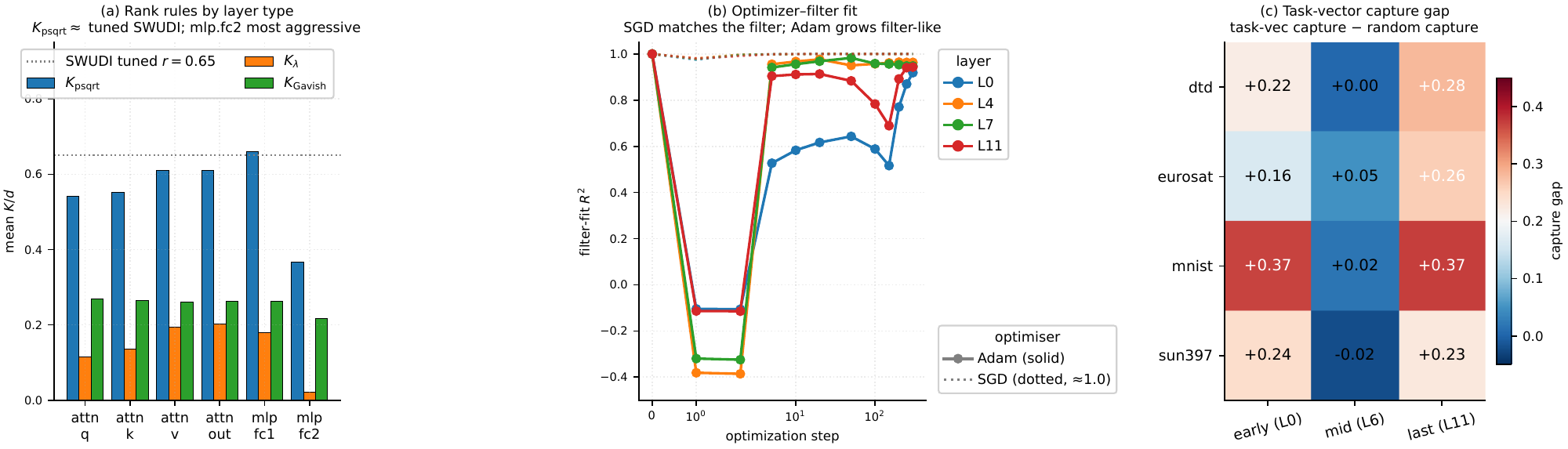}
\caption{\textbf{Rank-rule and optimizer-trajectory diagnostics on CLIP-ViT-B/32 TA8.} Panel (a) plots the mean retained-rank ratio $K/d_i$ (vertical axis) for different layer types (horizontal axis) under the participation-square-root rule $K_{\rm psqrt}$, the eigenvalue-participation rule $K_{\lambda}$, and the Gavish--Donoho rule $K_{\rm Gavish}$; the useful rank varies across layers (with \texttt{mlp.fc2} the most aggressively truncated), so it should be selected adaptively rather than fixed globally. Panel (b) plots how well iterative optimizer updates match a spectral filter: the horizontal axis is the optimization step (symmetric-log, so step $0$ is the initialization) and the vertical axis is the filter-fit $R^2$, with higher values indicating a closer spectral-filter interpretation; SGD (dotted) matches the filter throughout, while Adam (solid) starts unstable in the first few steps and becomes increasingly filter-like with training. Panel (c) reports the capture gap---task-vector capture minus random capture---by task (rows) and layer depth (columns); cell color and the printed value give the gap (warm red for a positive gap, cool blue for a near-zero or negative gap), so every cell carries a value and none are missing. Together, the diagnostics support layer-wise rank adaptation and the view that iterative merging behaves like spectral filtering.}
\label{fig:S2}
\end{figure*}

\subsection{Parameter-Drift Bound and Empirical Companion}
\label{app:conf-theorem}
This subsection reproduces the parameter-drift theorem of the conference version~\cite{wei2026optmerge}. The bound motivates the expert-construction protocol used in our merging experiments, namely controlling parameter drift during fine-tuning so that experts remain near a common basin around the base model.

\subsubsection{Notation and Setting}

\paragraph{Tasks and losses} For task $i$, let the loss $\mathcal{L}_i:\mathbb{R}^d\to\mathbb{R}$ be evaluated at parameters $\bm{\Theta}\in\mathbb{R}^d$.

\paragraph{Task vectors} After $T$ steps of (deterministic) gradient descent (GD) with fixed step size $\eta>0$ from a common initialization $\bm{\Theta}$, the task vector for task $i$ is
\[
\bm{\tau}_i := -\eta\sum_{t=0}^{T-1}\nabla\mathcal{L}_i(\bm{\Theta}^{(i)}_t).
\]

\paragraph{Merged update} Let $\bm{\tau}_m := \sum_{j=1}^N \alpha_j \bm{\tau}_j$ with nonnegative weights $\alpha_j\ge 0$. We study the loss of task $i$ at the merged point $\bm{\Theta}+\bm{\tau}_m$.

\paragraph{Norm and inner product} $\norm{\cdot}$ denotes the Euclidean norm and $\inp{\cdot}{\cdot}$ the Euclidean inner product. For nonzero vectors $u,v$, $\cos(u,v):=\inp{u}{v}/(\norm{u}\norm{v})$.

\subsubsection{Assumptions}

\begin{assumption}[$L$-smoothness]\label{ass:L-smooth}
Each $\mathcal{L}_i$ has $L$-Lipschitz continuous gradients: for all $\bm{\Theta},\bm{\Theta}'$,
\[
\norm{\nabla\mathcal{L}_i(\bm{\Theta})-\nabla\mathcal{L}_i(\bm{\Theta}')}\le L\norm{\bm{\Theta}-\bm{\Theta}'},
\]
equivalently, for any $\Delta$,
\[
\mathcal{L}_i(\bm{\Theta}+\Delta)\le\mathcal{L}_i(\bm{\Theta})+\inp{\nabla\mathcal{L}_i(\bm{\Theta})}{\Delta}+\tfrac{L}{2}\norm{\Delta}^2.
\]
\end{assumption}

\begin{assumption}[Polyak--\L{}ojasiewicz (PL) condition]\label{ass:PL}
Each $\mathcal{L}_i$ satisfies, for some $\mu>0$,
\[
\tfrac12\norm{\nabla\mathcal{L}_i(\bm{\Theta})}^2\ge\mu\big(\mathcal{L}_i(\bm{\Theta})-\mathcal{L}_i^*\big),
\]
where $\mathcal{L}_i^*:=\inf_{\bm{\Theta}}\mathcal{L}_i(\bm{\Theta})$.
\end{assumption}

\begin{assumption}[Directional similarity]\label{ass:similar}
For each $i$ and some $\kappa\in(0,1]$,
\[
\cos\big(-\nabla\mathcal{L}_i(\bm{\Theta}),\bm{\tau}_i\big)\ge\kappa,
\]
equivalently,
\[
\inp{\nabla\mathcal{L}_i(\bm{\Theta})}{\bm{\tau}_i}\le -\kappa\norm{\nabla\mathcal{L}_i(\bm{\Theta})}\norm{\bm{\tau}_i}.
\]
This ensures $\bm{\tau}_i$ is a descent direction for task $i$, with alignment quantified by $\kappa$.
\end{assumption}

\begin{assumption}[Approximate orthogonality]\label{ass:ortho}
For all $i\neq j$ and some $\varepsilon\in[0,1)$,
\[
\cos(\bm{\tau}_i,\bm{\tau}_j)\le\varepsilon.
\]
Prior works~\cite{ilharcoediting,ortiz-jimenez2023task} show that task vectors are nearly orthogonal in high-dimensional parameter space, which helps explain the success of model merging. A small $\varepsilon$ means that tasks are nearly orthogonal in update space, reducing negative transfer.
\end{assumption}

\begin{assumption}[Bounded gradients]\label{ass:gradbound}
There exists $G>0$ such that for all $i$ and all $\bm{\Theta}$ on the trajectory,
\[
\norm{\nabla\mathcal{L}_i(\bm{\Theta})}\le G.
\]
This boundedness condition is widely adopted in the optimization literature~\cite{gower2019sgd_assump,khaled2020better_assump}.
\end{assumption}

\subsubsection{Supporting Lemmas}

\begin{lemma}[Cross-task cosine leakage]\label{lem:cos}
Under Assumptions~\ref{ass:similar}--\ref{ass:ortho}, with $\nabla\mathcal{L}_i(\bm{\Theta})\neq\mathbf{0}$ and $\bm{\tau}_j\neq\mathbf{0}$, for $i\neq j$,
\[
\big|\cos(\nabla\mathcal{L}_i(\bm{\Theta}),\bm{\tau}_j)\big|\le\delta,\qquad \delta:=\kappa\varepsilon+\sqrt{1-\kappa^2}\sqrt{1-\varepsilon^2}.
\]
\end{lemma}

\begin{proof}[Proof sketch]
Normalize $u=-\nabla\mathcal{L}_i/\norm{\nabla\mathcal{L}_i}$, $v_i=\bm{\tau}_i/\norm{\bm{\tau}_i}$, $v_j=\bm{\tau}_j/\norm{\bm{\tau}_j}$. Assumption~\ref{ass:similar} gives $\inp{u}{v_i}\ge\kappa$ and Assumption~\ref{ass:ortho} gives $\inp{v_i}{v_j}\le\varepsilon$. Decomposing $u$ and $v_j$ along $v_i$ and its orthogonal complement and applying Cauchy--Schwarz yields the stated bound.
\end{proof}

\begin{lemma}[PL convergence under GD]\label{lem:PL}
Under Assumptions~\ref{ass:L-smooth}--\ref{ass:PL} and $\eta\in(0,1/L]$, the GD iterates for task $i$ satisfy
\[
\mathcal{L}_i(\bm{\Theta}_T)-\mathcal{L}_i^*\le(1-\eta\mu)^T\big(\mathcal{L}_i(\bm{\Theta}_0)-\mathcal{L}_i^*\big).
\]
\end{lemma}

\begin{proof}
For one GD step $\bm{\Theta}_{t+1}=\bm{\Theta}_t-\eta\nabla\mathcal{L}_i(\bm{\Theta}_t)$, the $L$-smooth upper bound (Assumption~\ref{ass:L-smooth}) gives
\[
\mathcal{L}_i(\bm{\Theta}_{t+1}) \le \mathcal{L}_i(\bm{\Theta}_t) - \eta\!\left(1-\tfrac{L\eta}{2}\right)\!\norm{\nabla\mathcal{L}_i(\bm{\Theta}_t)}^2.
\]
Since $\eta\le 1/L$, we have $1-L\eta/2\ge 1/2$, hence
\[
\mathcal{L}_i(\bm{\Theta}_{t+1}) \le \mathcal{L}_i(\bm{\Theta}_t) - \tfrac{\eta}{2}\norm{\nabla\mathcal{L}_i(\bm{\Theta}_t)}^2.
\]
Applying the PL inequality $\tfrac12\norm{\nabla\mathcal{L}_i(\bm{\Theta}_t)}^2\ge\mu(\mathcal{L}_i(\bm{\Theta}_t)-\mathcal{L}_i^*)$ yields
\[
\mathcal{L}_i(\bm{\Theta}_{t+1})-\mathcal{L}_i^* \le (1-\eta\mu)\big(\mathcal{L}_i(\bm{\Theta}_t)-\mathcal{L}_i^*\big).
\]
Unrolling this recursion over $t=0,\dots,T-1$ gives the claim.
\end{proof}

\begin{lemma}[Task-vector norm bound]\label{lem:tau-norm}
If $\bm{\tau}_j=-\eta\sum_{t=0}^{T-1}\nabla\mathcal{L}_j(\bm{\Theta}^{(j)}_t)$ and $\norm{\nabla\mathcal{L}_j(\bm{\Theta}^{(j)}_t)}\le G$ for all $t$, then $\norm{\bm{\tau}_j}\le\eta T G$.
\end{lemma}

\begin{proof}
By the triangle inequality,
\[
\norm{\bm{\tau}_j}\le\eta\sum_{t=0}^{T-1}\norm{\nabla\mathcal{L}_j(\bm{\Theta}^{(j)}_t)}\le\eta\sum_{t=0}^{T-1} G = \eta T G.\qedhere
\]
\end{proof}

\begin{lemma}[Inner-product upper bound]\label{lem:ip}
Under Assumptions~\ref{ass:L-smooth}--\ref{ass:PL} and $\eta\in(0,1/L]$,
\[
\inp{\nabla\mathcal{L}_i(\bm{\Theta})}{\bm{\tau}_i}\le -\big(1-(1-\eta\mu)^T\big)\big(\mathcal{L}_i(\bm{\Theta})-\mathcal{L}_i^*\big)+\tfrac{L}{2}\norm{\bm{\tau}_i}^2.
\]
\end{lemma}

\begin{proof}
$L$-Lipschitz gradients (Assumption~\ref{ass:L-smooth}) also imply the quadratic lower bound; applying it with $\Delta=\bm{\tau}_i$,
\[
\mathcal{L}_i(\bm{\Theta}+\bm{\tau}_i)\ge\mathcal{L}_i(\bm{\Theta})+\inp{\nabla\mathcal{L}_i(\bm{\Theta})}{\bm{\tau}_i}-\tfrac{L}{2}\norm{\bm{\tau}_i}^2,
\]
and rearrange to
\[
\inp{\nabla\mathcal{L}_i(\bm{\Theta})}{\bm{\tau}_i}\le\mathcal{L}_i(\bm{\Theta}+\bm{\tau}_i)-\mathcal{L}_i(\bm{\Theta})+\tfrac{L}{2}\norm{\bm{\tau}_i}^2.
\]
Since $\bm{\Theta}+\bm{\tau}_i=\bm{\Theta}_T$, Lemma~\ref{lem:PL} gives $\mathcal{L}_i(\bm{\Theta}_T)-\mathcal{L}_i^*\le(1-\eta\mu)^T(\mathcal{L}_i(\bm{\Theta})-\mathcal{L}_i^*)$, hence
\[
\mathcal{L}_i(\bm{\Theta}+\bm{\tau}_i)-\mathcal{L}_i(\bm{\Theta})\le -\big(1-(1-\eta\mu)^T\big)\big(\mathcal{L}_i(\bm{\Theta})-\mathcal{L}_i^*\big),
\]
which combined with the rearranged smoothness inequality yields the claim.
\end{proof}

\subsubsection{Main Theorems}

\begin{theorem}[Finite-step parameter-drift bound]\label{thm:fixed-eta}
Consider task $i$ trained for $T$ iterations of gradient descent with a fixed step size $\eta\in(0,1/L]$, and let $\gamma:=1-\eta\mu\in(0,1)$. Then the merged update $\bm{\tau}_m=\sum_{j=1}^N\alpha_j\bm{\tau}_j$ satisfies
\[
\mathcal{L}_i(\bm{\Theta}+\bm{\tau}_m)\le C_i+\mathcal{O}(\gamma^T)+\mathcal{O}(\delta\eta T)+\mathcal{O}(\eta^2 T^2),
\]
where $\mathcal{O}(\gamma^T)$ is the residual error from incomplete convergence on task~$i$, $\mathcal{O}(\delta\eta T)$ is the cross-task interference term, and $\mathcal{O}(\eta^2 T^2)$ is the curvature term from $L$-smoothness.
\end{theorem}

\begin{proof}
Define the $\eta,T$-independent constant
\[
C_i:=\mathcal{L}_i(\bm{\Theta})-\alpha_i\big(\mathcal{L}_i(\bm{\Theta})-\mathcal{L}_i^*\big).
\]
By $L$-smoothness,
\[
\mathcal{L}_i(\bm{\Theta}+\bm{\tau}_m)\le\mathcal{L}_i(\bm{\Theta})+\inp{\nabla\mathcal{L}_i(\bm{\Theta})}{\bm{\tau}_m}+\tfrac{L}{2}\norm{\bm{\tau}_m}^2.
\]
Decompose the inner product as
\[
\inp{\nabla\mathcal{L}_i}{\bm{\tau}_m}=\alpha_i\inp{\nabla\mathcal{L}_i}{\bm{\tau}_i}+\sum_{j\ne i}\alpha_j\inp{\nabla\mathcal{L}_i}{\bm{\tau}_j}.
\]
For the self term, Lemma~\ref{lem:ip} provides a constant part absorbed into $C_i$ and a residual term of order $\mathcal{O}(\gamma^T)$, plus a curvature correction $\mathcal{O}(\eta^2 T^2)$ via Lemma~\ref{lem:tau-norm}. For the cross terms, Lemma~\ref{lem:cos} together with Assumption~\ref{ass:gradbound} gives
\[
\big|\inp{\nabla\mathcal{L}_i}{\bm{\tau}_j}\big|\le\delta\eta T G^2,
\]
so the sum over $j\ne i$ is $\mathcal{O}(\delta\eta T)$. Finally, $\norm{\bm{\tau}_m}\le\eta T G\sum_j\alpha_j$ implies the smoothness term is $\mathcal{O}(\eta^2 T^2)$. Combining all contributions yields the stated bound.
\end{proof}

\begin{theorem}[Near-convergence regime]\label{thm:conv}
Suppose the residual PL error after $T$ steps is below a tolerance $\zeta>0$:
\[
(1-\eta\mu)^T\big(\mathcal{L}_i(\bm{\Theta})-\mathcal{L}_i^*\big)\le\zeta,
\]
equivalently,
\[
T\ge\frac{\ln\!\big((\mathcal{L}_i(\bm{\Theta})-\mathcal{L}_i^*)/\zeta\big)}{-\ln(1-\eta\mu)}.
\]
Then
\[
\mathcal{L}_i(\bm{\Theta}+\bm{\tau}_m)\le C_i+\mathcal{O}(\zeta)+\mathcal{O}(\delta\eta T)+\mathcal{O}(\eta^2 T^2),
\]
with the same $C_i$ as in Theorem~\ref{thm:fixed-eta}.
\end{theorem}

\begin{proof}
Starting from Theorem~\ref{thm:fixed-eta}, replace the residual term $\mathcal{O}(\gamma^T)$ by $\mathcal{O}(\zeta)$ using the near-convergence assumption. The cross-task and curvature terms are unchanged.
\end{proof}

\begin{remark}\label{rem:drift}
At a fixed learning rate, the improvement on the target task (captured by $1-\gamma^T$) typically outweighs the influence of other task vectors in the early training stage, especially when those vectors are close to orthogonal (small $\varepsilon$, hence small $\delta$). As training approaches convergence, the negative impact from cross-task interference grows linearly in $T$ as $\mathcal{O}(\delta\eta T)$, and curvature errors grow quadratically as $\mathcal{O}(\eta^2 T^2)$; even when individual single-task losses keep decreasing, the merged loss can worsen due to accumulated interference. Once $(1-\eta\mu)^T(\mathcal{L}_i-\mathcal{L}_i^*)\le\zeta$, the dominant residual terms are interference and curvature; reducing directional leakage (small $\delta$) and limiting $\eta T$ are therefore essential for high-quality merging. This motivates the conference benchmark's choice to fine-tune each MLLM expert for one epoch with a reduced learning rate, and is consistent with prior empirical observations that less intensive fine-tuning often yields stronger merging~\cite{yu2024language,li2025map} and that fine-tuned models tend to converge near the base model~\cite{merlin2023happens,chung2024scaling,wu2023pi}.
\end{remark}

\subsubsection{Empirical Fine-Tuning Step Sweep}
\label{app:rethinking}
To illustrate Theorem~\ref{thm:fixed-eta} empirically, we run the standard CLIP-ViT-B/32 merging benchmark following the FusionBench fine-tuning setup~\cite{tang2024fusionbench}. We train each task expert with Adam at learning rate $10^{-5}$ for $4{,}000$ steps with batch size $32$, and save checkpoints every $500$ steps. Across the eight TA8 tasks, single-task accuracy on the corresponding test split typically converges around $3{,}000$ steps (Fig.~\ref{fig:clip-steps-pertask}), whereas merged accuracy peaks earlier and then declines as fine-tuning proceeds (Fig.~\ref{fig:clip-merge-vs-steps}).

\begin{figure}[tbh]
\centering
\includegraphics[width=0.55\textwidth]{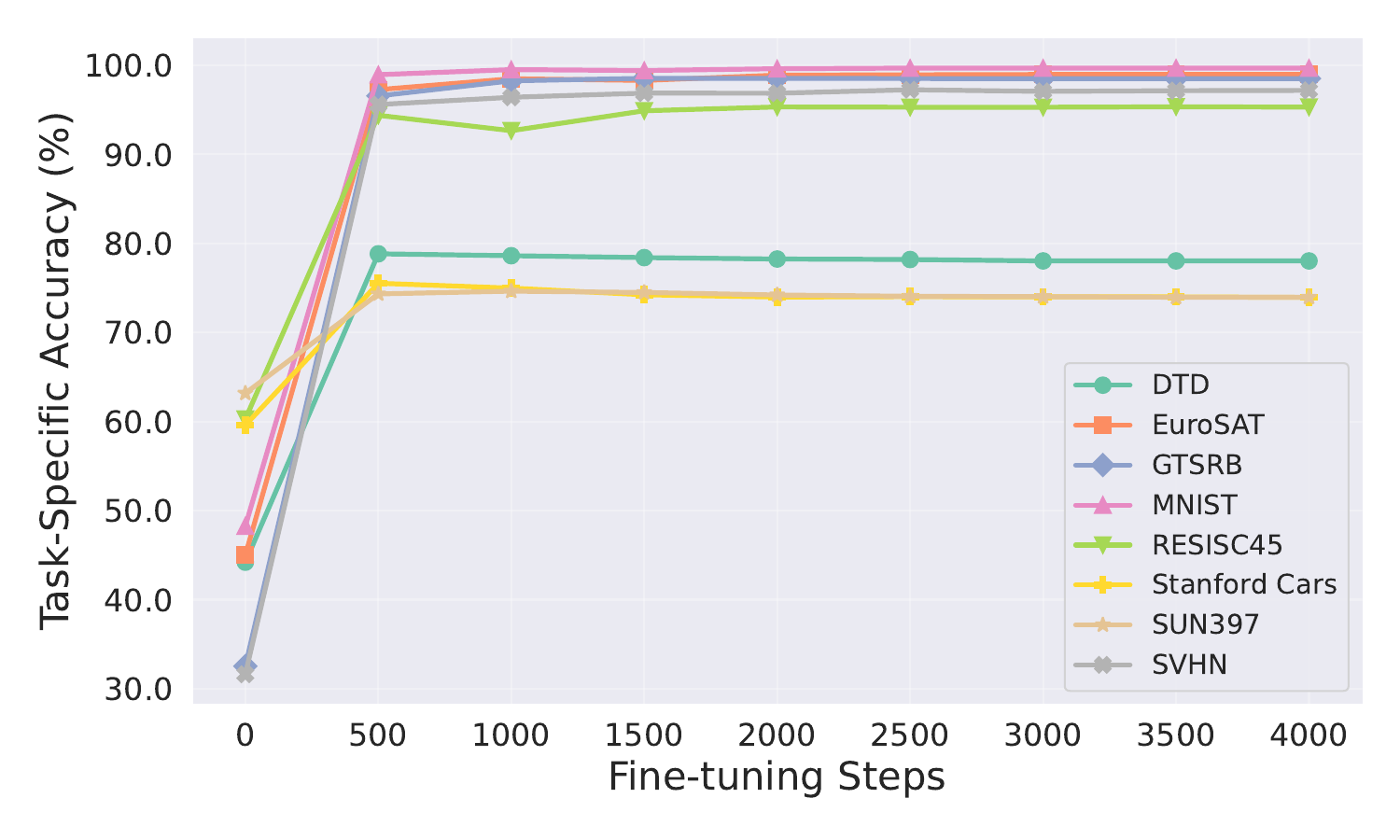}
\caption{\textbf{Single-task fine-tuning accuracy of CLIP-ViT-B/32 on the eight TA8 tasks as a function of fine-tuning steps.} Accuracy converges around $3{,}000$ steps on every task, providing the per-task ground truth against which merging accuracy is measured in Fig.~\ref{fig:clip-merge-vs-steps}.}
\label{fig:clip-steps-pertask}
\end{figure}

\begin{figure*}[tbh]
\centering
\resizebox{0.9\textwidth}{!}{%
\begin{minipage}{\textwidth}
\centering
\subfloat[Task Arithmetic]{\includegraphics[width=0.48\textwidth]{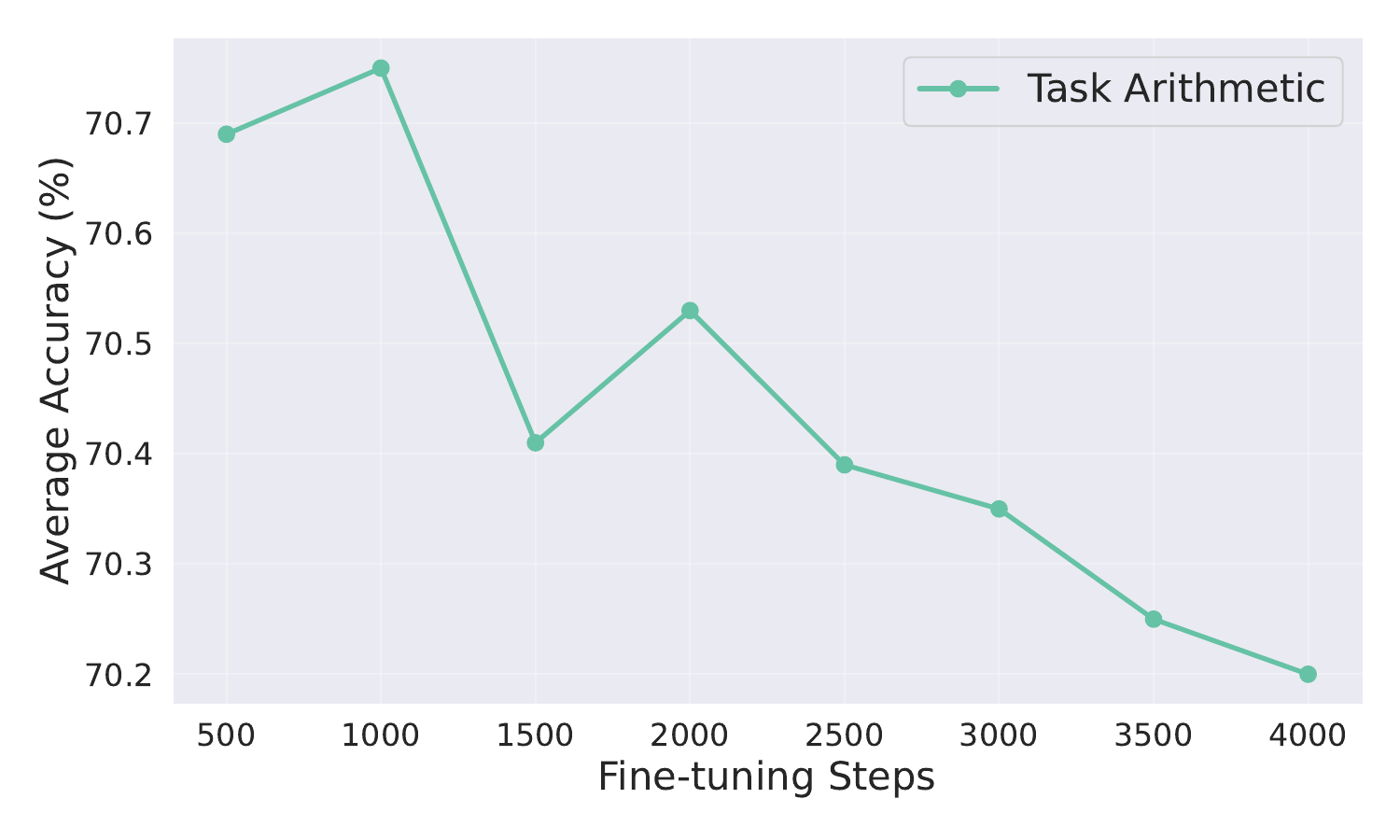}}\hfill
\subfloat[Weight Average]{\includegraphics[width=0.48\textwidth]{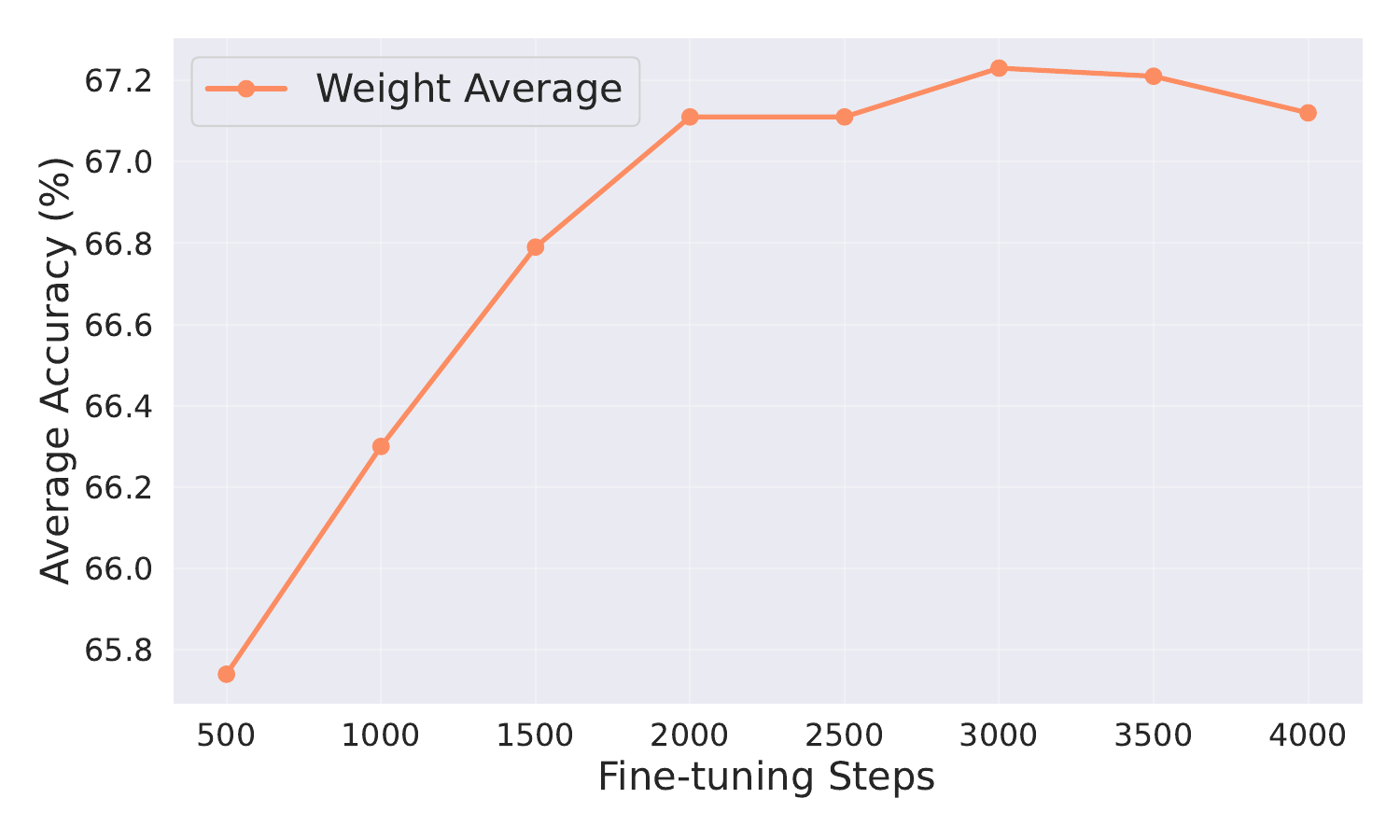}}\\[0.3em]
\subfloat[DARE]{\includegraphics[width=0.48\textwidth]{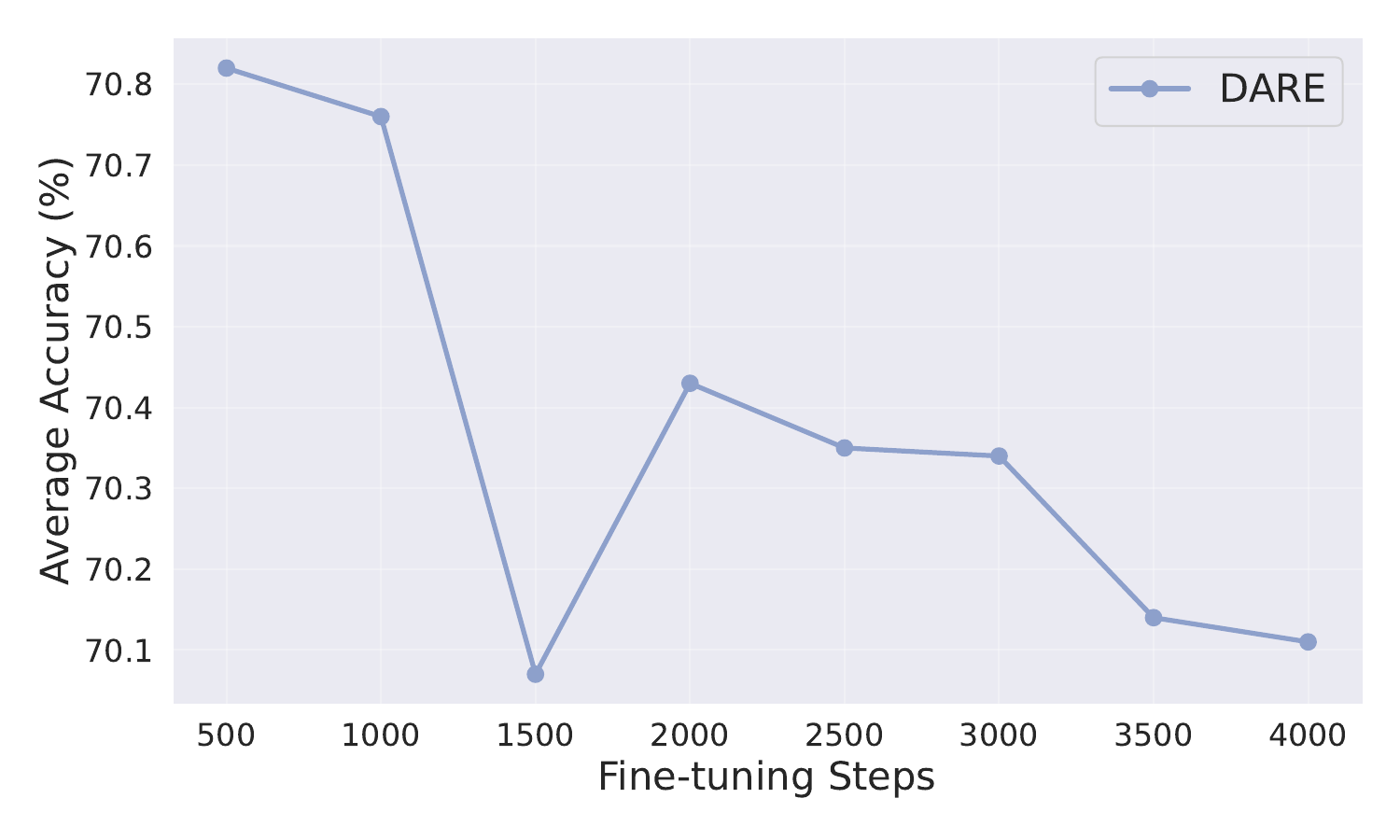}}\hfill
\subfloat[TSV Merging]{\includegraphics[width=0.48\textwidth]{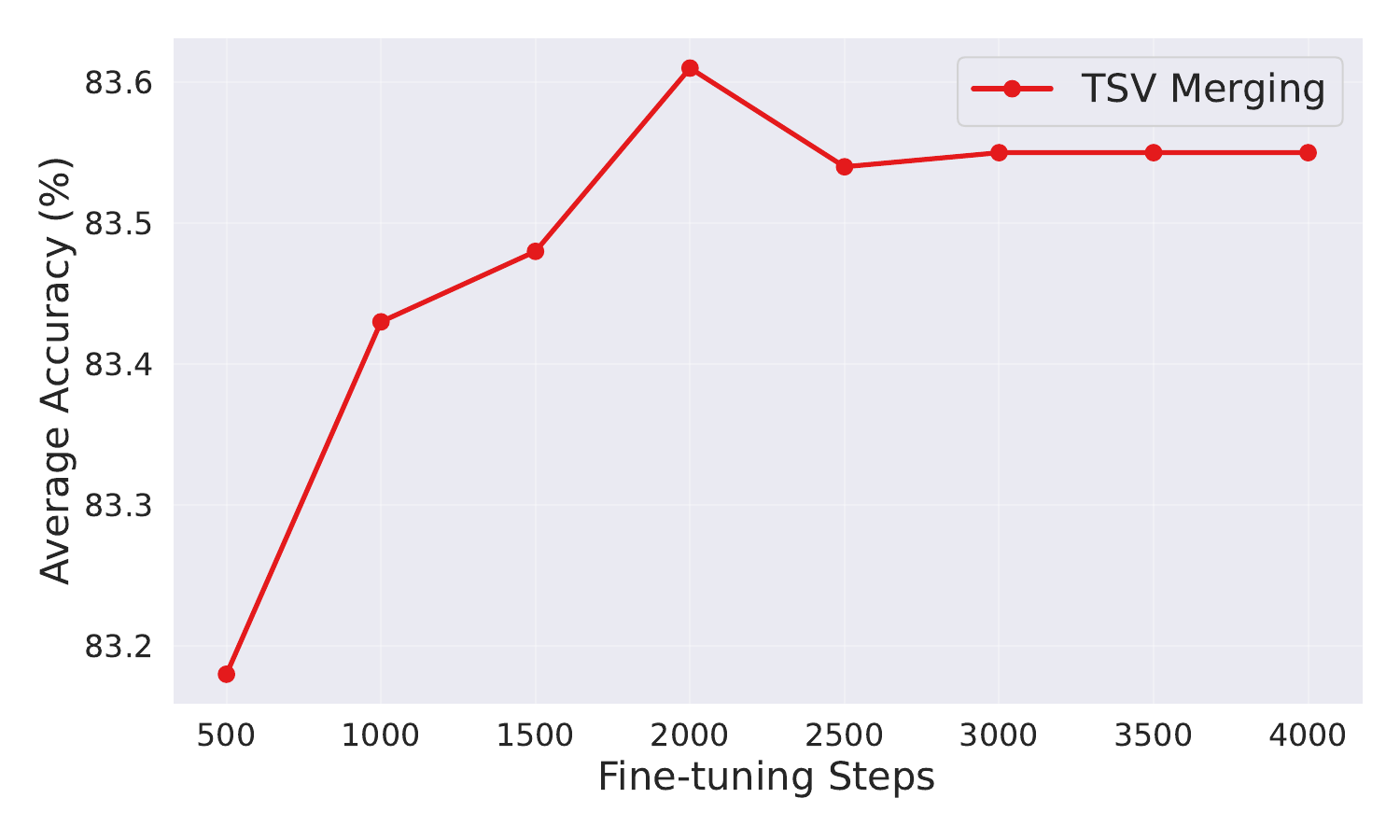}}
\end{minipage}}
\caption{\textbf{Average merging accuracy on CLIP-ViT-B/32 TA8 against the fine-tuning step at which each expert was checkpointed.} Across all four merging methods, accuracy first rises and then declines as fine-tuning progresses, with the peak occurring well before single-task convergence. This unimodal pattern is the empirical signature of Theorem~\ref{thm:fixed-eta}: in the early phase the target-task improvement $1-\gamma^T$ dominates; once the loss approaches its minimum, the cross-task interference $\mathcal{O}(\delta\eta T)$ and curvature $\mathcal{O}(\eta^2 T^2)$ terms grow large enough to outweigh the single-task gains, and merged accuracy decreases. MLLM training is organized in epochs rather than steps, so we fix the number of epochs to $1$ and reduce the learning rate, which keeps fine-tuned experts close to the base model in parameter space while still improving on the target task.}
\label{fig:clip-merge-vs-steps}
\end{figure*}

%% file: appendices/C_additional_results.tex
\section{Additional Analyses}
\label{app:tables}
\label{app:additional-results}

This appendix provides protocol details for the multimodal benchmark, per-task vision/language results, ablations of the closed-form solvers, spectral diagnostics, and additional MLLM scaling and checkpoint-merging studies.

\subsection{Per-Task Vision Results on CLIP-ViT}
This subsection expands the CLIP-ViT TA8 summary in Table~\ref{tab:clip-summary} with per-task results for the three evaluated backbones. Per-task accuracy on the eight vision tasks for CLIP-ViT-B/32, B/16, and L/14 is reported in Tables~\ref{tab:clip-b32},~\ref{tab:clip-b16}, and~\ref{tab:clip-l14}.

\begin{table*}[!t]
\centering
\caption{CLIP-ViT-B/32: per-task accuracy ($\%$) on the 8 vision tasks. Best per column in \textbf{bold}; second-best avg \underline{underlined}.}
\label{tab:clip-b32}
\setlength{\tabcolsep}{4pt}
\renewcommand{\arraystretch}{1.05}
\resizebox{0.8\textwidth}{!}{%
\begin{tabular}{l|cccccccc|c}
\toprule
\textbf{Method} & \textbf{SUN397} & \textbf{Cars} & \textbf{RESISC45} & \textbf{EuroSAT} & \textbf{SVHN} & \textbf{GTSRB} & \textbf{MNIST} & \textbf{DTD} & \textbf{Avg.} \\
\midrule
Weight Average                   & 65.44 & 62.43 & 70.63 & 75.74 & 64.51 & 54.96 & 86.28 & 50.59 & 66.32 \\
Task Arithmetic                  & 57.01 & 55.70 & 64.75 & 73.30 & 77.93 & 68.50 & 96.07 & 47.13 & 67.55 \\
TIES Merging                     & 67.01 & 64.15 & 74.30 & 74.52 & 77.74 & 69.38 & 94.13 & 53.99 & 71.90 \\
TA w/ DARE                       & 57.06 & 55.40 & 64.48 & 73.30 & 78.07 & 68.38 & 96.06 & 46.97 & 67.46 \\
TIES w/ DARE                     & 39.25 & 43.10 & 52.65 & 62.37 & 81.39 & 71.48 & 97.47 & 39.95 & 60.96 \\
TSV Merging                      & 67.62 & 71.65 & 84.70 & 93.44 & 91.90 & 92.53 & 98.86 & 63.83 & 83.07 \\
Iso-C                            & \textbf{71.66} & \textbf{73.44} & 84.76 & 88.04 & 78.69 & 84.62 & 96.69 & 65.21 & 80.39 \\
$\tau^{\rm cf}=DC^\dagger$        & 66.82 & 70.25 & 82.48 & 90.11 & 93.27 & 92.83 & 99.13 & 63.72 & 82.33 \\
\rowcolor{gray!10}WUDI Merging                     & 68.47 & 72.68 & 84.44 & 95.26 & 94.90 & 95.00 & 99.29 & 67.02 & 84.63 \\
\rowcolor{gray!10}OptMerge                         & 67.16 & 72.11 & 85.25 & 94.85 & \textbf{95.27} & \textbf{95.66} & \textbf{99.33} & 66.60 & 84.53 \\
\midrule
\rowcolor{gray!10}\swudi{}-soft (ablation) & 69.29 & 72.50 & 86.35 & 95.44 & 94.53 & 94.76 & 99.27 & 68.62 & 85.10 \\
\rowcolor{gray!10}\swudi{}    & 70.06 & 73.03 & \textbf{87.27} & \textbf{95.81} & 94.22 & 95.00 & 99.29 & 69.68 & \textbf{85.55} \\
\rowcolor{gray!10}\aswudi{}          & 69.99 & 72.81 & 87.22 & 95.70 & 94.36 & 95.00 & 99.30 & \textbf{69.84} & \underline{85.53} \\
\bottomrule
\end{tabular}}
\end{table*}

\begin{table*}[!t]
\centering
\caption{CLIP-ViT-B/16: per-task accuracy ($\%$) on the 8 vision tasks. Best per column in \textbf{bold}; second-best avg \underline{underlined}.}
\label{tab:clip-b16}
\setlength{\tabcolsep}{4pt}
\renewcommand{\arraystretch}{1.05}
\resizebox{0.8\textwidth}{!}{%
\begin{tabular}{l|cccccccc|c}
\toprule
\textbf{Method} & \textbf{SUN397} & \textbf{Cars} & \textbf{RESISC45} & \textbf{EuroSAT} & \textbf{SVHN} & \textbf{GTSRB} & \textbf{MNIST} & \textbf{DTD} & \textbf{Avg.} \\
\midrule
Weight Average                   & 68.74 & 69.05 & 75.06 & 83.30 & 74.98 & 62.57 & 93.75 & 51.17 & 72.33 \\
Task Arithmetic                  & 65.91 & 68.31 & 75.49 & 84.52 & 88.87 & 81.96 & 98.08 & 53.99 & 77.14 \\
TIES Merging                     & 70.65 & 71.23 & 79.89 & 87.52 & 83.29 & 76.29 & 96.43 & 55.48 & 77.60 \\
TA w/ DARE                       & 65.88 & 68.28 & 75.54 & 84.26 & 88.96 & 82.16 & 98.10 & 53.99 & 77.15 \\
TIES w/ DARE                     & 56.10 & 61.14 & 70.68 & 77.59 & 92.22 & 85.96 & 98.76 & 51.91 & 74.30 \\
TSV Merging                      & 73.12 & 80.74 & 89.75 & 96.19 & 94.15 & 94.10 & 99.08 & 69.68 & 87.10 \\
Iso-C                            & 75.13 & 81.06 & 90.35 & 94.70 & 86.21 & 89.13 & 97.68 & 66.28 & 85.07 \\
$\tau^{\rm cf}=DC^\dagger$        & 73.61 & 79.37 & 91.65 & 96.93 & 94.28 & 96.41 & 99.32 & 72.77 & 88.04 \\
\rowcolor{gray!10}WUDI Merging                     & 75.07 & 82.10 & 92.13 & 97.85 & 95.92 & 96.65 & 99.37 & 74.26 & 89.17 \\
\rowcolor{gray!10}OptMerge                         & 74.91 & \textbf{82.51} & \textbf{92.92} & 97.81 & \textbf{96.15} & \textbf{97.40} & \textbf{99.40} & 74.84 & 89.49 \\
\midrule
\rowcolor{gray!10}\swudi{}-soft (ablation) & 75.79 & 81.95 & 92.54 & \textbf{97.89} & 95.86 & 96.84 & \textbf{99.40} & \textbf{75.96} & \underline{89.53} \\
\rowcolor{gray!10}\swudi{}     & \textbf{76.01} & 82.18 & 92.73 & \textbf{97.89} & 95.74 & 97.01 & 99.37 & 75.64 & \textbf{89.57} \\
\rowcolor{gray!10}\aswudi{}       & 75.91 & 81.97 & 92.81 & 97.78 & 95.84 & 96.86 & 99.34 & 75.43 & 89.49 \\
\bottomrule
\end{tabular}}
\end{table*}

\begin{table*}[!t]
\centering
\caption{CLIP-ViT-L/14: per-task accuracy ($\%$) on the 8 vision tasks. Best per column in \textbf{bold}; second-best avg \underline{underlined}.}
\label{tab:clip-l14}
\setlength{\tabcolsep}{4pt}
\renewcommand{\arraystretch}{1.05}
\resizebox{0.8\textwidth}{!}{%
\begin{tabular}{l|cccccccc|c}
\toprule
\textbf{Method} & \textbf{SUN397} & \textbf{Cars} & \textbf{RESISC45} & \textbf{EuroSAT} & \textbf{SVHN} & \textbf{GTSRB} & \textbf{MNIST} & \textbf{DTD} & \textbf{Avg.} \\
\midrule
Weight Average                   & 72.53 & 81.54 & 82.32 & 88.52 & 81.63 & 74.02 & 96.62 & 61.76 & 79.87 \\
Task Arithmetic                  & 72.02 & 79.00 & 80.57 & 84.63 & 87.49 & 83.48 & 98.05 & 58.51 & 80.47 \\
TIES Merging                     & 74.77 & 83.16 & 86.51 & 89.70 & 89.67 & 85.19 & 97.75 & 63.88 & 83.83 \\
TA w/ DARE                       & 72.09 & 78.85 & 80.46 & 84.48 & 87.60 & 83.57 & 98.03 & 58.83 & 80.49 \\
TIES w/ DARE                     & 65.78 & 69.59 & 69.29 & 73.30 & 87.39 & 80.69 & 97.84 & 50.74 & 74.33 \\
TSV Merging                      & 78.18 & 89.79 & 93.52 & 96.70 & 95.58 & 96.48 & 99.08 & 75.27 & 90.57 \\
Iso-C                            & 79.78 & 90.76 & 94.37 & 96.48 & 92.92 & 95.38 & 98.77 & 76.76 & 90.65 \\
$\tau^{\rm cf}=DC^\dagger$        & 79.52 & 90.13 & 93.70 & 97.41 & 96.35 & 97.84 & 99.31 & 79.26 & 91.69 \\
\rowcolor{gray!10}WUDI Merging                     & 80.03 & 90.71 & 93.89 & 98.33 & 96.93 & 98.01 & 99.30 & 80.11 & 92.16 \\
\rowcolor{gray!10}OptMerge                         & 80.15 & 90.86 & 94.29 & 98.33 & \textbf{97.10} & \textbf{98.44} & 99.32 & 80.59 & 92.38 \\
\midrule
\rowcolor{gray!10}\swudi{}-soft (ablation) & 80.55 & 91.00 & 94.17 & \textbf{98.52} & 97.03 & 98.08 & 99.33 & 80.74 & 92.43 \\
\rowcolor{gray!10}\swudi{}            & 80.56 & \textbf{91.10} & \textbf{94.46} & 98.37 & 96.83 & 98.19 & 99.38 & \textbf{81.17} & \underline{92.51} \\
\rowcolor{gray!10}\aswudi{}                       & \textbf{80.59} & 91.06 & 94.40 & 98.48 & 97.07 & 98.25 & \textbf{99.39} & 80.90 & \textbf{92.52} \\
\bottomrule
\end{tabular}}
\end{table*}

\subsection{Spectral Diagnostics}
This subsection gathers the spectrum-level evidence used to motivate layer-wise adaptive truncation. We first compare architectures and fine-tuning regimes, then include additional diagnostic panels.

\subsubsection{Per-Architecture Spectral Statistics}
\label{app:tab-spectra}
Per-architecture spectral statistics on the per-layer interference operators $C^{(\ell)}=\sum_i \tau_i^\top \tau_i / \|\tau_i\|_F^2$, referenced from Sec.~\ref{sec:diag}, are summarized in Table~\ref{tab:spectra}. Figs.~\ref{fig:mllm-magnitude} and~\ref{fig:mllm-norm} provide the empirical hook for the LoRA-vs-full-FT contrast that motivates the dual rank-rule design.

\begin{figure}[!t]
\centering
\resizebox{0.9\textwidth}{!}{
\subfloat[InternVL2.5-1B (full FT)]{\includegraphics[width=0.42\textwidth]{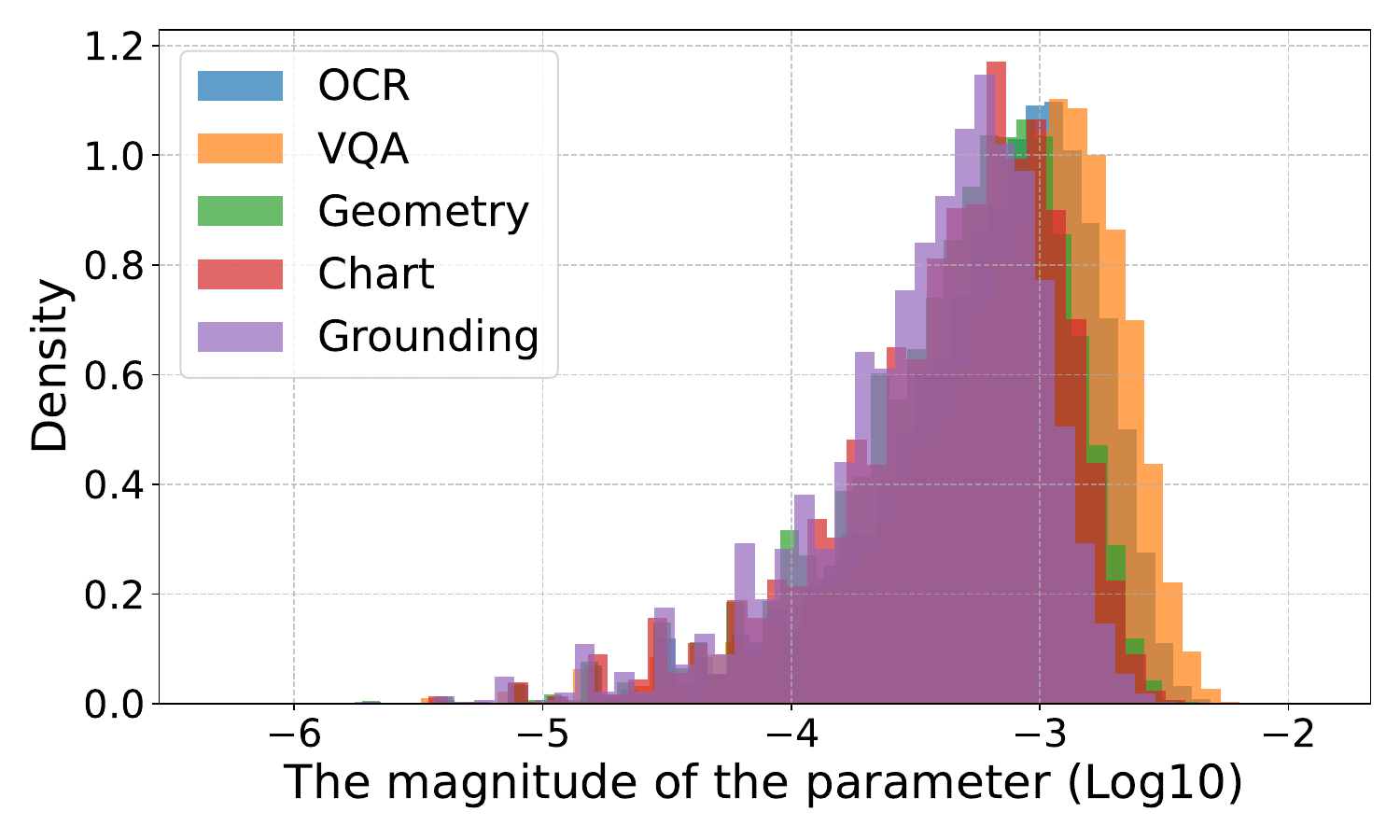}}\hfill
\subfloat[Qwen2-VL-7B (LoRA $r{=}8$)]{\includegraphics[width=0.42\textwidth]{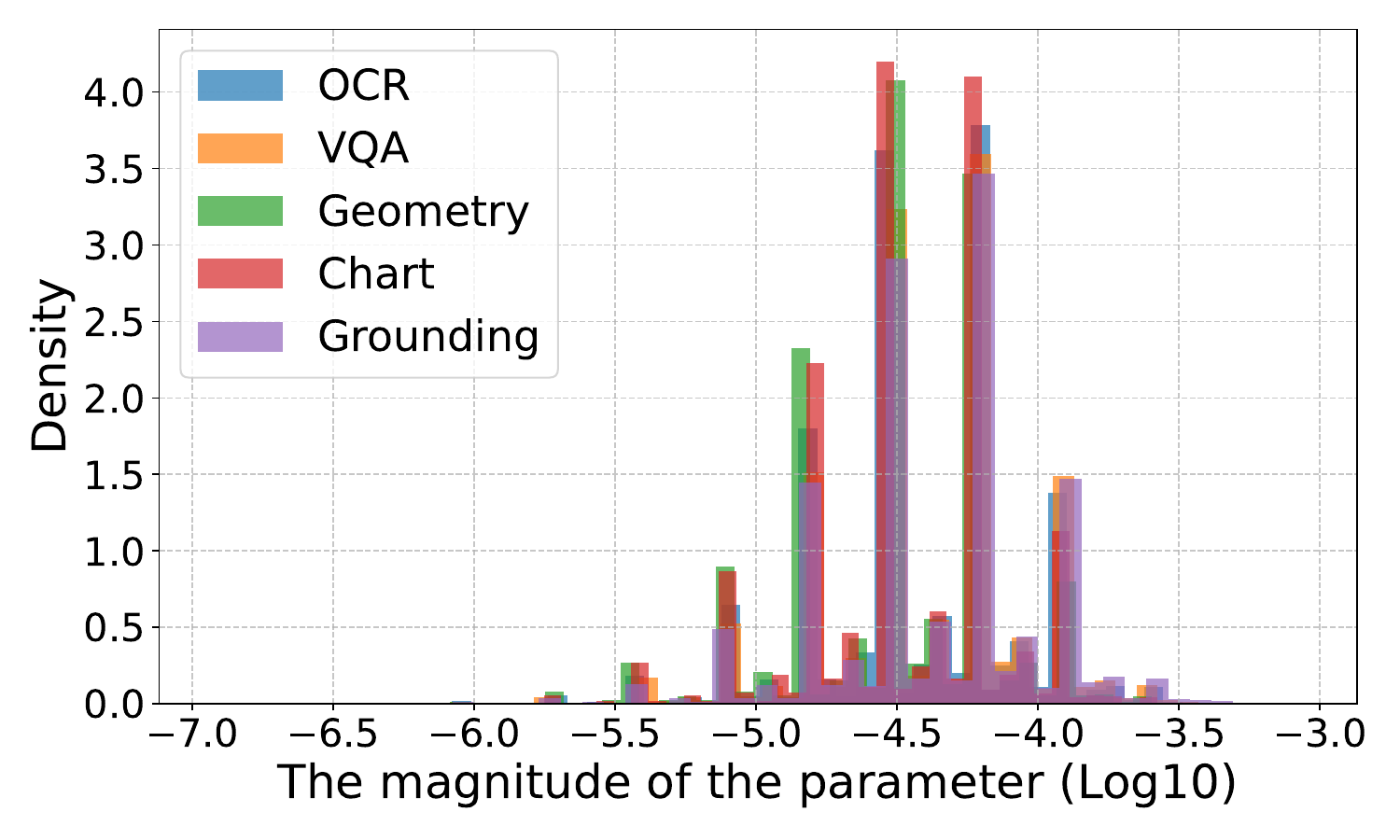}}
}
\caption{\textbf{Task-vector magnitude distribution on the MLLM benchmark.} InternVL2.5 (full fine-tuning) exhibits a right-skewed distribution typical of dense parameter updates, whereas Qwen2-VL (LoRA) displays a multi-modal distribution: the low-rank constraint and LoRA scaling factor restrict deltas to a reduced subspace, causing them to cluster along a few dominant magnitudes. Both backbones show distinct distributions across tasks, supporting layer-wise rather than global rank rules.}
\label{fig:mllm-magnitude}
\end{figure}

\begin{figure}[!t]
\centering
\resizebox{0.9\textwidth}{!}{
\subfloat[InternVL2.5-1B (full FT)]{\includegraphics[width=0.42\textwidth]{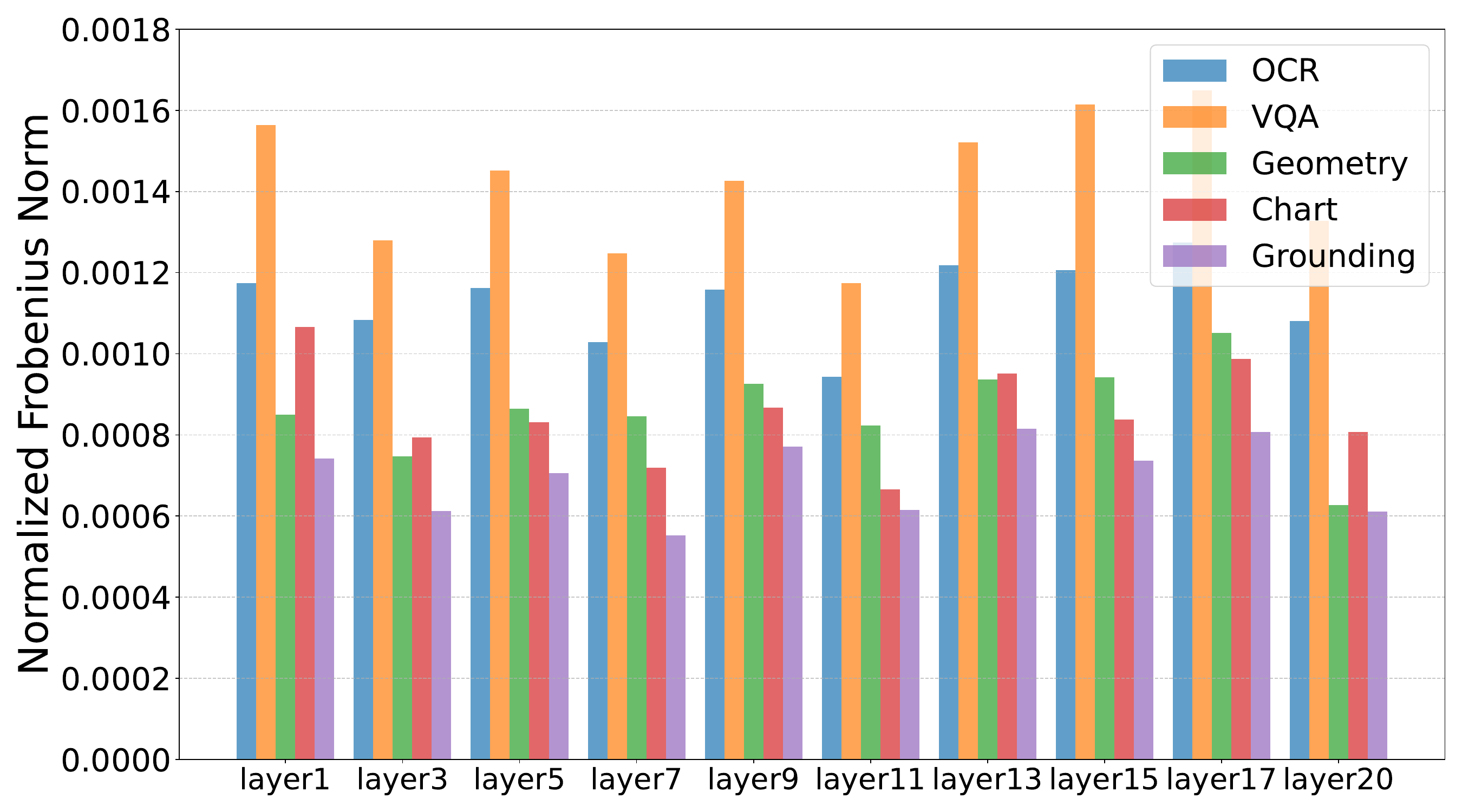}}\hfill
\subfloat[Qwen2-VL-7B (LoRA $r{=}8$)]{\includegraphics[width=0.42\textwidth]{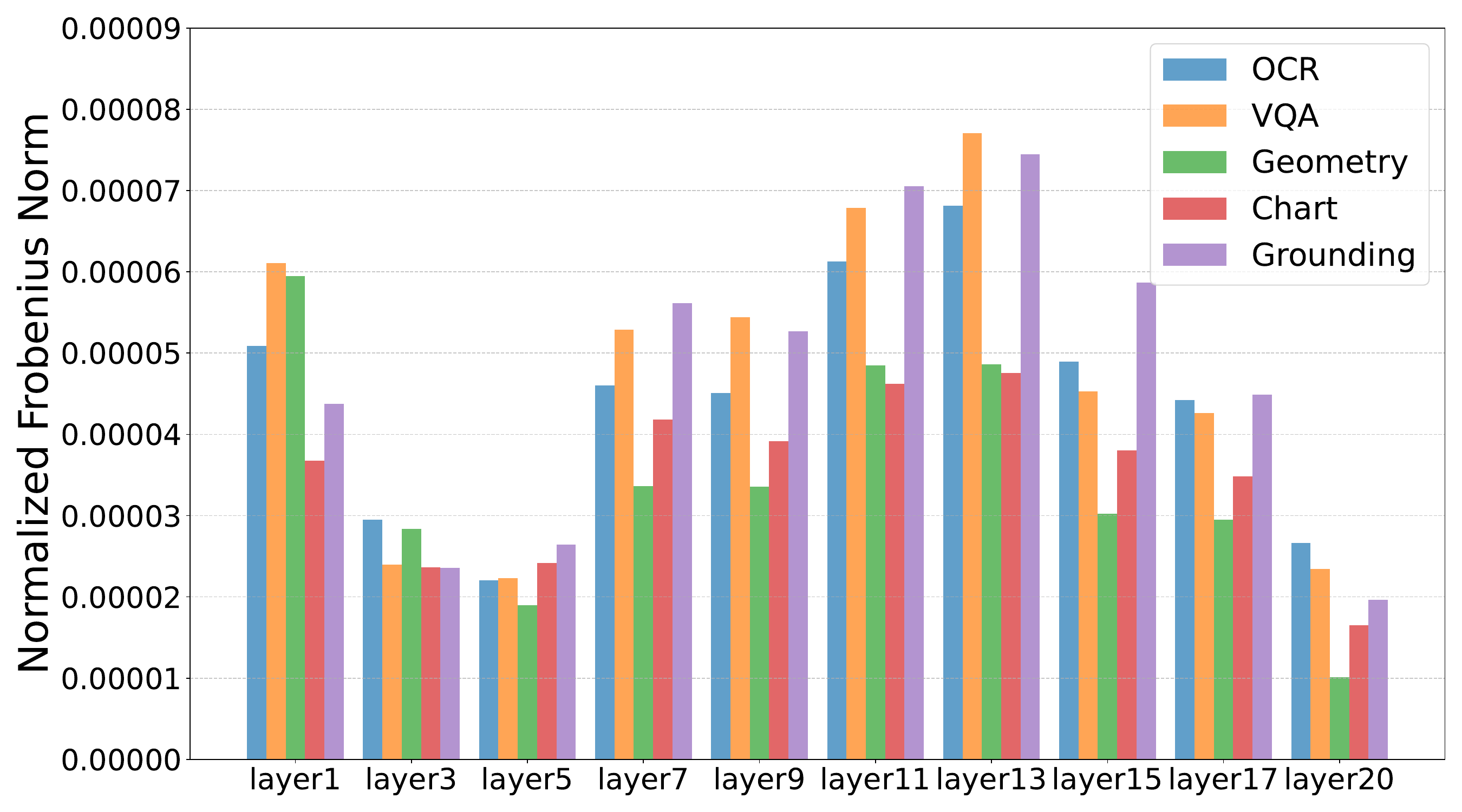}}
}
\caption{\textbf{Normalized Frobenius norm of task vectors across layers.} Norms are divided by the number of parameters of the corresponding linear layer. The Frobenius norm varies substantially across both layers and tasks, and the variation pattern differs by architecture and fine-tuning regime. Layer-wise rank adaptation in \aswudi{} addresses this heterogeneity directly. The small magnitudes in absolute terms (well below $1\%$ of the base-model weight norm) are consistent with the conference-version observation that fine-tuned MLLMs and base models occupy adjacent regions of the loss landscape with linear connectivity~\cite{wu2023pi}.}
\label{fig:mllm-norm}
\end{figure}

\begin{table*}[!t]
\centering
\caption{Per-architecture spectral diagnostics motivating adaptive rank selection.
For each benchmark, we summarize the per-layer operators
$C^{(\ell)}=\sum_i \tau_i^\top \tau_i / \|\tau_i\|_F^2$ by update magnitude, effective-rank ratio, spectral conditioning, layer-wise variability, and the ranks retained by \aswudi{}-\texttt{psqrt}.
$\lambda_{90}:=\lambda_{\lceil 0.9\,d_i\rceil}$ is the eigenvalue at the $90\%$ rank position from the top under descending eigenvalue order.}
\label{tab:spectra}
\setlength{\tabcolsep}{6pt}
\renewcommand{\arraystretch}{1.05}
\resizebox{\textwidth}{!}{%
\begin{tabular}{l|ccccc}
\toprule
\textbf{Statistic} & CLIP-B/32 & Flan-T5 LoRA & Llama-3.2-3B & Qwen2-VL-7B (LoRA) & InternVL2.5-1B \\
\midrule
\# linear layers ($\ell$)                                    & 72 & 72 & 196 & 560 & 168 \\
$\|\tau_i^{(\ell)}\|_F/\|W_0^{(\ell)}\|_F$ (mean over $i,\ell$) & 0.0173 & 0.0111 & 0.0093 & 0.0059 (LoRA $\Delta$) & 0.0181 \\
Effective-rank ratio $r_{\rm eff}/d_i$ (mean)                & 0.177 & 0.0043 & 0.266 & 0.041 & 0.408 \\
$\lambda_{\max}/\lambda_{\rm med}$ (median)                  & 85    & $1.3{\cdot}10^{8}$ & 80 & $4.2{\cdot}10^{6}$ & 73 \\
$\lambda_{90}/\lambda_{\max}$ (median)                       & $2.9{\cdot}10^{-3}$ & $4.5{\cdot}10^{-10}$ & $5.3{\cdot}10^{-3}$ & $1.7{\cdot}10^{-8}$ & $4.1{\cdot}10^{-3}$ \\
Layer-wise CV of $\lambda_{\max}$                            & 0.601 & 0.257 & 0.720 & 0.844 & 0.512 \\
\aswudi{}-\texttt{psqrt} mean $K/d_i$                        & 0.554 & 0.013 & 0.602 & 0.149 & 0.615 \\
\aswudi{}-\texttt{psqrt} min/median/max $K$                  & 26/180/512 & 4/12/64 & 16/512/3072 & 128/570/4010 & 64/472/2048 \\
\bottomrule
\end{tabular}}
\end{table*}

\subsubsection{Task-Vector Proxy and Optimizer-Filter Diagnostics}
\label{app:additional-diagnostics}
This subsection presents two per-layer diagnostics evaluated on CLIP-ViT-B/32. Fig.~\ref{fig:dropped-panels-1}(a) illustrates the capture gap between the task-vector subspaces and the input activation subspace, serving as the empirical basis for the task-vector proxy and Assumption~\ref{ass:cov}. Furthermore, Fig.~\ref{fig:dropped-panels-1}(b) demonstrates the exact equivalence between SGD on the \wudi{} quadratic objective and the Landweber spectral filter. This result confirms Proposition~\ref{prop:flow} and validates the SGD/Landweber identity deferred from Sec.~\ref{sec:flow}.

\begin{figure*}[!t]
\centering
\includegraphics[width=0.90\textwidth]{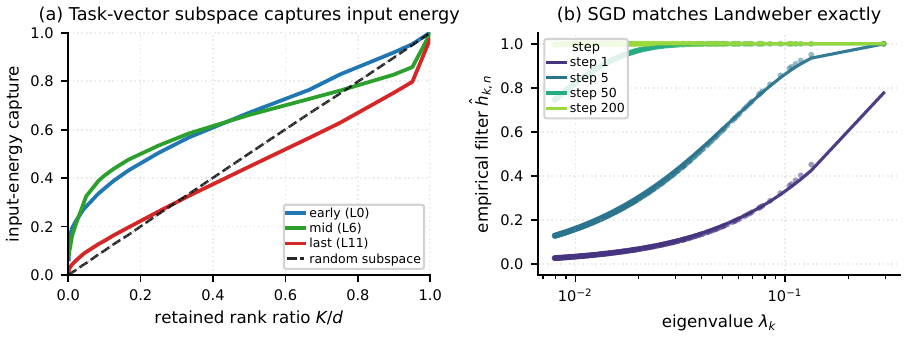}
\caption{\textbf{Task-vector proxy and optimizer-filter diagnostics.} \textbf{(a)} Task-vector subspaces capture input energy in early and middle CLIP-ViT-B/32 layers (capture gap $+0.18$--$0.43$ vs.\ random subspaces); the last MLP layer is a documented exception (gap $\approx 0$). \textbf{(b)} SGD on the \wudi{} quadratic exactly matches the Landweber spectral filter $1-(1-\eta\lambda_k)^{n_{\rm eff}}$ at all checkpoints ($R^2=1.0000$, $n_{\rm eff}\approx 2\!\cdot\!\text{step}$), confirming Proposition~\ref{prop:flow}.}
\label{fig:dropped-panels-1}
\end{figure*}

\subsection{\optmerge{} Analysis and Rank-Truncation Evidence}
This subsection revisits two pieces of evidence from the \optmerge{} analysis. The component-wise analysis explains how the iterative baseline was stabilized, while the truncation-ratio sweep motivates the hard spectral cutoff used by \swudi{}.

\subsubsection{Component-Wise Analysis}
\label{app:tab-optmerge-abl}
\optmerge{} introduces three modifications to the iterative \wudi{} objective: replacing Adam with SGD, initializing $\tau_m$ with the mean of the task vectors, and applying a low-rank approximation to $\tau_i$. We analyze the contribution of each component on Qwen2-VL (LoRA capability merging) and Vicuna-7B (modality merging), which are the two settings most affected by the narrow active subspace characteristic of LoRA. Replacing Adam with SGD in isolation is detrimental, as the optimizer struggles to escape the norm-inflation regime documented in Fig.~\ref{fig:optmerge-norm-history}. However, incorporating the mean initialization recovers and further improves accuracy, while the low-rank approximation of $\tau_i$ yields an additional marginal gain. Furthermore, these components have a neutral or mildly positive effect in the modality-merging setting. This indicates that they do not degrade performance in regimes for which \optmerge{} was not explicitly tuned.

\begin{figure}[!t]
\centering
\includegraphics[width=0.5\textwidth]{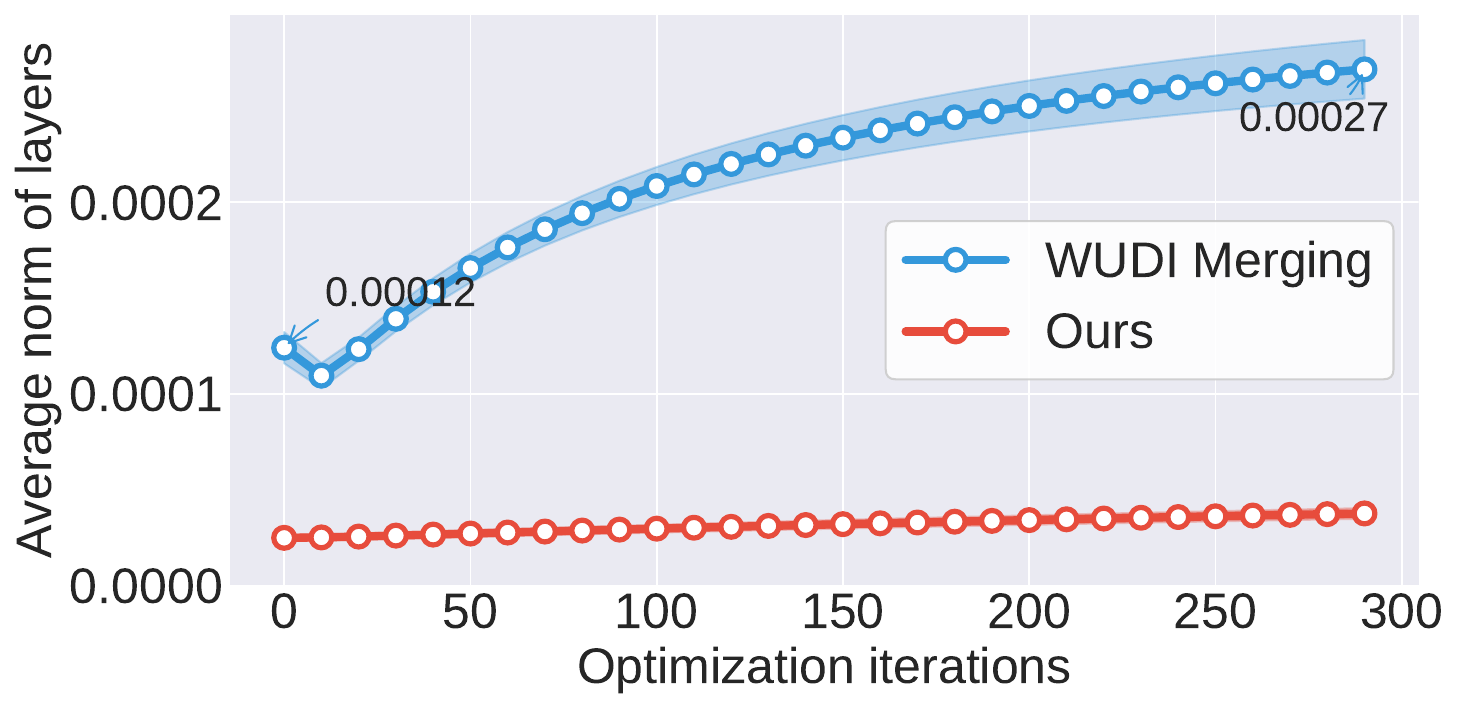}
\caption{\textbf{Frobenius-norm trajectory of $\tau_m$ under iterative \wudi{}/\optmerge{} on the Qwen2-VL LoRA setting (averaged over linear layers).} The unregularized iterative loss inflates $\|\tau_m\|_F$ throughout optimization (the norm-shortcut behavior of Sec.~\ref{sec:noise}; see also Fig.~\ref{fig:norm-shortcut}). Mean initialization plus the low-rank truncation of $\tau_i$ keep the trajectory norm-bounded while reducing the loss successfully. The closed-form spectral solvers \swudi{}/\aswudi{} avoid the inflation altogether by suppressing small-$\lambda_k$ directions in the eigenbasis.}
\label{fig:optmerge-norm-history}
\end{figure}

\subsubsection{Evidence for Head-Spectrum Truncation}
\label{app:tab-swudi-k}
Table~\ref{tab:swudi-k} presents a truncation-ratio sweep ($k/d_i\in\{0.1,0.2,0.3,0.4,0.5\}$) conducted for \optmerge{} in the InternVL2.5-1B capability merging setting. We include this analysis because the truncation index $k$ plays an identical role in both the hard-truncation factor of \swudi{} and the low-rank denoising of $\tau_i$ in \optmerge{}. Average performance remains essentially stable for $k\in[0.1,0.3]$ (ranging from $56.63\%$ to $57.43\%$), but declines for $k\ge 0.4$ as more low-eigenvalue directions are incorporated into the inversion. This observation aligns with the analysis in Sec.~\ref{sec:noise}: the head of the spectrum accounts for almost all the proxy reduction, whereas including tail directions introduces noise rather than signal. Consequently, this trend provides empirical justification for setting the default \swudi{} configuration to a small truncation ratio.

\begin{table*}[!t]
\centering
\caption{\optmerge{} truncation-ratio sweep on InternVL2.5-1B, used to motivate \swudi{} rank choices. Each row reports per-task accuracy ($\%$).}
\label{tab:swudi-k}
\setlength{\tabcolsep}{2pt}
\renewcommand{\arraystretch}{1.05}
\resizebox{0.9\textwidth}{!}{%
\begin{tabular}{l|cc|cc|c|cc|ccc|c}
\toprule
\multirow{2}{*}{\textbf{$k$ ratio}} & \multicolumn{2}{c|}{\textbf{VQA}} & \multicolumn{2}{c|}{\textbf{Geometry}} & \textbf{Chart} & \multicolumn{2}{c|}{\textbf{OCR}} & \multicolumn{3}{c|}{\textbf{Grounding}} & \multirow{2}{*}{\textbf{Avg.}}\\
\cmidrule{2-11}
 & VizWiz & GQA & MathVista & MATH-Vision & ChartQA & TextVQA & OCRVQA & RefCOCO & RefCOCO+ & RefCOCOg & \\
\midrule
$10\%$ & 30.90 & 57.26 & 51.49 & 18.42 & 68.40 & 76.10 & 46.39 & 76.36 & 69.99 & 73.96 & 56.93 \\
$20\%$ & 30.97 & 57.13 & 54.48 & 21.05 & 68.72 & 76.01 & 46.35 & 75.97 & 69.72 & 73.94 & \textbf{57.43} \\
$30\%$ & 31.55 & 57.15 & 54.50 & 21.05 & 68.72 & 76.27 & 45.67 & 73.63 & 66.84 & 70.92 & 56.63 \\
$40\%$ & 31.49 & 56.92 & 55.77 & 25.00 & 67.36 & 76.06 & 45.96 & 65.55 & 58.40 & 59.64 & 54.22 \\
$50\%$ & 31.37 & 56.68 & 56.75 & 23.68 & 68.08 & 75.81 & 45.02 & 61.45 & 54.80 & 56.19 & 52.98 \\
\bottomrule
\end{tabular}}
\end{table*}

%% file: appendices/D_reproducibility.tex
\section{Our MLLMerging Benchmark}
\label{app:reproduce}

This appendix documents MLLMerging, the benchmark introduced in Sec.~\ref{sec:bench-mllm} for merging multimodal large language models (MLLMs). It isolates the merging algorithm as the only free variable: experts share a backbone, are fine-tuned on capability-aligned data, and are merged purely in parameter space without access to the original training data. The benchmark provides the training suite, full fine-tuning and LoRA expert checkpoints, and a matched evaluation protocol, enabling fair comparison across merging methods.

\subsection{Motivation and Benchmark Scope}
\label{app:reproduce-why}

Existing model-merging benchmarks are dominated by vision-only classifiers or text-only language tasks. MLLMs introduce additional complications because a single model must preserve visual perception, language reasoning, grounding, OCR, and modality-specific alignment. MLLMerging targets three gaps that are not fully covered by earlier benchmarks.

\noindent\textbf{Training-evaluation mismatch.} Public MLLMs are often trained on mixtures of proprietary, licensed, and open-source data, while they are evaluated on standalone suites such as MMBench~\cite{liu2024mmbench}, SEED-Bench~\cite{li2024seed}, MME~\cite{fu2306mme}, and MMStar~\cite{chen2024are}. The same backbone can therefore exhibit very different capability profiles depending on the fine-tuning data. MLLMerging aligns capability-specific training data with capability-specific evaluation suites, so that the merging algorithm, rather than the upstream data composition, is the primary variable.

\noindent\textbf{Task expertise versus instruction following.} Capability datasets such as VQA, OCR, and grounding provide strong task supervision, but they may not match the broad instruction-following distribution of modern MLLMs. The benchmark therefore reports both per-capability evaluations (Tables~\ref{tab:intern} and~\ref{tab:qwen}) and integrated multimodal QA evaluations (Table~\ref{tab:integrated}).

\noindent\textbf{Capability and modality composition.} Beyond combining task-specialized experts that share a full MLLM backbone, the benchmark also studies modality merging: vision-, audio-, and video-language experts share an LLM backbone but use modality-specific encoders and connectors. This setting tests whether a parameter-space merge can preserve complementary sensory channels without online routing or joint retraining.

\subsection{Capability-Merging Tasks and Data}
\label{app:reproduce-data}

A core contribution of MLLMerging is the curated capability-merging data suite. Unlike prior studies that rely on a few vision-classification heads or a single text corpus, we construct a large, capability-aligned training pool. This ensures each expert is a true specialist, isolating the merging algorithm as the sole variable during evaluation. The suite encompasses five complementary MLLM capabilities (VQA, Geometry, Chart understanding, OCR, and Grounding), aggregating approximately $1.37$M instruction-tuning samples from over twenty public datasets (Table~\ref{tab:data-detail}). We deliberately collect at least $100$K samples per capability and prioritize source diversity (e.g., incorporating ten OCR datasets ranging from scene text to document and table understanding). This prevents experts from overfitting to a single dataset's distribution, promoting broad generalization within their respective domains.

Two additional design choices ensure the suite is readily reusable as a benchmark. First, all data sources are standardized into a ShareGPT-style instruction-tuning format with a uniform grounding-coordinate convention (Appendix~\ref{app:reproduce-hp}). This guarantees that any backbone can be fine-tuned, and any merging method evaluated, under identical supervision conditions. Second, the suite intentionally mixes English-only and bilingual (English/Chinese) sources. While InternVL2.5-1B utilizes the full multilingual dataset, Qwen2-VL-7B is restricted to the English-only subsets. This design allows us to evaluate merging algorithms on both multilingual full-parameter experts and monolingual low-rank (LoRA) experts within a unified framework. Ultimately, pairing this training suite with the capability-matched evaluation protocol (Appendix~\ref{app:reproduce-eval}) closes the train-evaluation gap often present in earlier MLLM benchmarks (Appendix~\ref{app:reproduce-why}). Consequently, any changes in downstream accuracy can be confidently attributed to the merging algorithm itself, rather than variations in upstream data composition.
\begin{table}[!tbp]
\centering
\caption{Capability training datasets used to construct the MLLMerging expert checkpoints: five capabilities, over twenty public datasets, and $\approx 1.37$M instruction-tuning samples in total ($\geq 100$K per capability).}
\label{tab:data-detail}
\setlength{\tabcolsep}{6pt}
\renewcommand{\arraystretch}{1.16}
\resizebox{\textwidth}{!}{%
\begin{tabular}{l|c|p{12cm}}
\toprule
\textbf{Capability} & \textbf{Total} & \textbf{Datasets (language)} \\
\midrule
VQA       & 588K & GQA (en)~\cite{hudson2019gqa}, VQAv2 (en)~\cite{goyal2017making}, OKVQA (en)~\cite{marino2019ok}, LLaVA-Instruct (zh)~\cite{liu2024improved}, CogVLM-Singleround (en \& zh)~\cite{wang2024cogvlm}, CogVLM-Multiround (en \& zh)~\cite{wang2024cogvlm} \\
Geometry  & 190K & GeoQA+ (zh)~\cite{cao2022augmented}, G-LLaVA (en)~\cite{gao2023g} \\
Chart     & 218K & ChartQA (en)~\cite{masry2022chartqa}, DVQA (en)~\cite{kafle2018dvqa} \\
OCR       & 238K & OCRVQA (en)~\cite{mishra2019ocr}, TextCaps (en)~\cite{sidorov2020textcaps}, SynthDoG (en)~\cite{kim2022ocr}, LLaVAR (en)~\cite{zhang2023llavar}, ST-VQA (en)~\cite{biten2019scene}, TextVQA (en)~\cite{singh2019towards}, DocVQA (en)~\cite{mathew2021docvqa}, DeepForm (en)~\cite{svetlichnaya2020deepform}, KLC (en)~\cite{stanislawek2021kleister}, TabFact (en)~\cite{chen2019tabfact} \\
Grounding & 135K & RefCOCO (en)~\cite{kazemzadeh2014referitgame,yu2016modeling,mao2016generation}, VG (en)~\cite{krishna2017visual} \\
\bottomrule
\end{tabular}}
\end{table}

\subsection{Backbones and Expert Construction}
\label{app:reproduce-hp}

\noindent\textbf{Capability merging.} We employ two representative MLLM backbones. InternVL2.5-1B-Instruct~\cite{chen2024expanding} is fully fine-tuned for one epoch with a learning rate of $4{e}{-5}$ and a warmup ratio of $3{e}{-2}$. Qwen2-VL-7B-Base~\cite{wang2024qwen2} is fine-tuned using LoRA~\cite{hu2022lora} with a rank of $r{=}8$, a learning rate of $10^{-5}$, and a warmup ratio of $10^{-1}$. These two configurations allow us to evaluate the merging methods on both dense full-parameter deltas and low-rank LoRA deltas.

\noindent\textbf{Data preprocessing.} Following standard training practices for InternVL and Qwen2-VL, we utilize only the training splits. We filter out corrupted images and samples where the combined question-answer length exceeds $8192$ tokens. The remaining data is then converted into the ShareGPT-style instruction-tuning format. Grounding coordinates are linearly mapped to the $[0,1000)$ range and enclosed within Qwen2-VL box tokens (\eg, \texttt{<|box\_start|>$\cdots$<|box\_end|>}~\cite{wang2024qwen2}).

\subsection{Modality-Merging Track}
\label{app:reproduce-modality}

For modality merging (Sec.~\ref{sec:exp_modality}, Table~\ref{tab:omni}), we follow~\cite{chen2024model} and pair Vicuna-7B-v1.5~\cite{zheng2023judging} with three modality-specific encoder/connector pairs. The vision, audio, and video experts share the same LLM backbone but are trained on different bi-modal data. Table~\ref{tab:modalities} lists the modality components.

\begin{table}[!tbp]
\centering
\caption{Modality components and training data for the three single-modality experts.}
\label{tab:modalities}
\setlength{\tabcolsep}{4pt}
\renewcommand{\arraystretch}{1.20}
\resizebox{\textwidth}{!}{%
\begin{tabular}{c|p{3.0cm}|c|p{3.0cm}|p{4.5cm}|p{3.0cm}}
\toprule
\textbf{Modality} & \textbf{Encoder} & \textbf{Connector} & \textbf{Alignment Data} & \textbf{Fine-tuning Data} & \textbf{Reference} \\
\midrule
Vision & CLIP-ViT-L-336px~\cite{radford2021learning} & MLP & LCS 558K~\cite{liu2023visual} & LLaVA-mixed 665K~\cite{liu2024improved} & LLaVA-1.5~\cite{liu2024improved} \\
\midrule
Audio  & BEATs-Iter3+~\cite{chen2023beats} & Q-Former~\cite{li2023blip} & WaveCaps 400K~\cite{mei2024wavcaps} & OpenAQA filtered 350K~\cite{gong2024listen} & X-InstructBLIP~\cite{panagopoulou2023x} \\
\midrule
Video  & LanguageBind~\cite{zhu2023languagebind} & MLP & LCS 558K~\cite{liu2023visual}, Valley 702K~\cite{luo2023valley} & Video-ChatGPT 100K~\cite{maaz2024video}, LLaVA-mixed subset 140K~\cite{liu2024improved} & Video-LLaVA~\cite{lin2024video} \\
\bottomrule
\end{tabular}}
\end{table}

The modality experts are trained in two stages. Stage~1 aligns each modality encoder to the LLM by training only the connector. Stage~2 fine-tunes the connector and the LLM, with LoRA of rank $r{=}128$ applied to all linear modules in the LLM. At merging time, the modality-specific encoders and connectors are kept intact, and only the LLM LoRA deltas are merged. The resulting Omni model can process vision, audio, and video inputs while using a single merged LLM backbone.

\subsection{Evaluation Protocol}
\label{app:reproduce-eval}

Capability evaluation is conducted using VLMEvalKit~\cite{duan2024vlmevalkit} and \texttt{lmms-eval}~\cite{zhang2024lmms} with consistent decoding, preprocessing, and answer-extraction settings. The five-capability suite includes VizWiz~\cite{gurari2018vizwiz} and GQA~\cite{hudson2019gqa} for VQA; MathVista~\cite{lu2024mathvista} and MATH-Vision~\cite{wang2024measuring} for Geometry; ChartQA~\cite{masry2022chartqa} for Chart understanding; TextVQA~\cite{singh2019towards} and OCRVQA~\cite{mishra2019ocr} for OCR; and RefCOCO/+/g~\cite{kazemzadeh2014referitgame,yu2016modeling,mao2016generation} for Grounding.

\noindent\textbf{Math geometry subsets.}\label{app:geometry-subsets} While the original conference paper~\cite{wei2026optmerge} restricted MathVista to its geometry-related subsets and MATH-Vision to four specific geometry categories, the updated protocol in Tables~\ref{tab:intern} and~\ref{tab:qwen} reports the official overall results for both MathVista and MATH-Vision.

\noindent\textbf{Integrated QA and modality evaluation.} For integrated multimodal QA, we employ MMMU~\cite{yue2024mmmu}, DocVQA~\cite{mathew2021docvqa}, ScienceQA~\cite{lu2022learn}, AI2D~\cite{kembhavi2016diagram}, and InfographicVQA~\cite{mathew2022infographicvqa}. Modality merging is evaluated on AVQA~\cite{yang2022avqa} and MUSIC-AVQA~\cite{li2022learning}, which assess spatio-temporal reasoning across audio-visual scenes.

\subsection{Answer Extraction Prompt}
\label{app:answer-extraction}

For MathVista and MATH-Vision, free-form model outputs are normalized by GPT-4o-mini using the prompt below. The template variables \texttt{\{question\}} and \texttt{\{prediction\}} denote the original question and the model's raw response.

\begingroup
\scriptsize
\begin{spverbatim}
Please read the following examples. Then extract the answer from the model response and type it at the end of the prompt.

Hint: Please answer the question requiring an integer answer and provide the final value,
e.g., 1, 2, 3, at the end.
Question: Which number is missing?
Model response: The number missing in the sequence is 14.
Extracted answer: 14

Hint: Please answer the question requiring a floating-point number with one decimal place and provide the final value,
e.g., 1.2, 1.3, 1.4, at the end.
Question: What is the fraction of females facing the camera?
Model response: The fraction of females facing the camera is 0.6,
which means that six out of ten females in the group are facing the camera.
Extracted answer: 0.6

Hint: Please answer the question requiring a floating-point number with two decimal places and provide the final value,
e.g., 1.23, 1.34, 1.45, at the end.
Question: How much money does Luca need to buy a sour apple candy and a butter-scotch candy? (Unit: $)
Model response: Luca needs $1.45 to buy a sour apple candy and a butterscotch candy.
Extracted answer: 1.45

Hint: Please answer the question requiring a Python list as an answer and provide the final list,
e.g., [1, 2, 3], [1.2, 1.3, 1.4], at the end.
Question: Between which two years does the line graph saw its maximum peak?
Model response: The line graph saw its maximum peak between 2007 and 2008.
Extracted answer: [2007, 2008]

Hint: Please answer the question and provide the correct option letter, e.g., A, B, C, D, at the end.
Question: What fraction of the shape is blue?
Choices: (A) 3/11 (B) 8/11 (C) 6/11 (D) 3/5
Model response: The correct answer is (B) 8/11.
Extracted answer: B

{question}
Model response: {prediction}
Extracted answer:
\end{spverbatim}
\endgroup